\documentclass{article}

    \PassOptionsToPackage{numbers, compress}{natbib}

    \usepackage[preprint]{neurips_2020}

\usepackage[utf8]{inputenc} 
\usepackage[T1]{fontenc}    
\usepackage{hyperref}
\hypersetup{
    colorlinks=true,
    linkcolor=black,
    citecolor=blue
}
\usepackage{url}
\usepackage{amsfonts}       
\usepackage{nicefrac}       
\usepackage{microtype}      
\usepackage{amsthm}
\usepackage{amsmath,mathtools}
\usepackage{empheq}
\usepackage{amssymb}
\usepackage{macros}
\usepackage{booktabs}
\usepackage{multirow}
\usepackage{subcaption}
\usepackage{wrapfig}
\usepackage{chngcntr}
\usepackage{xspace}
\usepackage{bm}
\usepackage[ruled,vlined]{algorithm2e}
\usepackage[colorinlistoftodos,prependcaption]{todonotes}
\usepackage{nicefrac,xfrac}
\usepackage{multirow}
\usepackage{xcolor}
\usepackage{tcolorbox}
\usepackage{makecell}

\newcommand{\eg}{\textit{e.g.}}
\newcommand{\ie}{\textit{i.e.}}

\newcommand{\surloss}{\smash{\widetilde{F}}}

\newcommand{\nworkers}{m}
\newcommand{\sg}{g}

\newcommand{\tg}{\nabla F}

\newcommand{\lr}{\eta}
\newcommand{\lip}{L}
\newcommand{\obj}{F}
\newcommand{\vbnd}{\sigma^2}
\newcommand{\x}{\bm{x}} 
\newcommand{\z}{\bm{z}}

\newcommand{\cpeff}{\tau_{\text{eff}}}
\newcommand{\cpmax}{\tau_{\text{max}}}
\newcommand{\cpavg}{\overline{\tau}}
\newcommand{\cp}{\tau}

\newcommand{\bnda}{\beta^2}
\newcommand{\bndb}{\kappa^2}
\newcommand{\csqdist}{\chi^2_{\bm{p}\|\bm{w}}}
\newcommand{\by}{\bm{y}}

\newcommand{\nsg}{\bm{d}}
\newcommand{\ntg}{\bm{h}}

\newcommand{\gradweight}{\bm{a}}
\newcommand{\samplenum}{q}

\newcommand{\matI}{\bm{I}}
\newcommand{\matH}{\bm{H}}
\newcommand{\matK}{\bm{K}}

\newcommand{\matG}{\bm{G}}

\newcommand{\fedavg}{\texttt{FedAvg}\xspace}
\newcommand{\fedprox}{\texttt{FedProx}\xspace}
\newcommand{\vrlsgd}{\texttt{VRLSGD}\xspace}
\newcommand{\scaffold}{\texttt{SCAFFOLD}\xspace}
\newcommand{\fednova}{\texttt{FedNova}\xspace}

\usepackage{scalerel}

\usepackage[colorinlistoftodos,prependcaption]{todonotes}

\usepackage{enumitem}
\usepackage{cleveref}
\crefname{equation}{}{}
\Crefname{equation}{}{}
\crefname{thm}{theorem}{theorems}
\Crefname{thm}{Theorem}{Theorems}
\crefname{clm}{claim}{claims}
\Crefname{clm}{Claim}{Claims}
\Crefname{coro}{Corollary}{Corollaries}
\Crefname{lem}{Lemma}{Lemmas}
\Crefname{sec}{Section}{Sections}
\crefname{app}{appendix}{appendices}
\Crefname{app}{Appendix}{Appendices}
\Crefname{part}{Part}{Parts}
\crefname{prop}{proposition}{propositions}
\Crefname{prop}{Proposition}{Propositions}
\Crefname{propty}{Property}{Properties}
\crefname{figure}{fig.}{figures}
\Crefname{figure}{Figure}{Figures}
\crefname{defn}{definition}{definitions}
\Crefname{defn}{Definition}{Definitions}
\crefname{fact}{fact}{facts}
\Crefname{fact}{Fact}{Facts}
\crefname{appendix}{appendix}{appendices}
\Crefname{appendix}{Appendix}{Appendices}
\crefname{algo}{algorithm}{algorithms}
\Crefname{algo}{Algorithm}{Algorithms}
\crefname{algorithm}{algorithm}{algorithms}
\Crefname{algorithm}{Algorithm}{Algorithms}
\crefname{conj}{conjecture}{conjectures}
\Crefname{conj}{Conjecture}{Conjectures}
\crefname{obs}{observation}{observations}
\Crefname{obs}{Observation}{Observations}
\crefname{assump}{assumption}{assumptions}
\Crefname{assump}{Assumption}{Assumptions}
\crefname{rem}{remark}{remarks}
\Crefname{rem}{Remark}{Remarks}

\title{Tackling the Objective Inconsistency Problem \\
in Heterogeneous Federated Optimization}

\author{%
  Jianyu Wang \\
  Carnegie Mellon University\\
  Pittsburgh, PA 15213 \\
  \And
  Qinghua Liu \\
  Princeton University \\
  Princeton, NJ 08544 \\
  \AND
  Hao Liang \\
  Carnegie Mellon University \\
  Pittsburgh, PA 15213 \\
  \And
    Gauri Joshi \\
  Carnegie Mellon University \\
  Pittsburgh, PA 15213 \\
  \And
    H. Vincent Poor \\
  Princeton University \\
  Princeton, NJ 08544 \\
}

\begin{document}

\maketitle

\begin{abstract}
In federated optimization, heterogeneity in the clients' local datasets and computation speeds results in large variations in the number of local updates performed by each client in each communication round. Naive weighted aggregation of such models causes objective inconsistency, that is, the global model converges to a stationary point of a mismatched objective function which can be arbitrarily different from the true objective. This paper provides a general framework to analyze the convergence of federated heterogeneous optimization algorithms. It subsumes previously proposed methods such as FedAvg and FedProx and provides the first principled understanding of the solution bias and the convergence slowdown due to objective inconsistency. Using insights from this analysis, we propose FedNova, a normalized averaging method that eliminates objective inconsistency while preserving fast error convergence.
\end{abstract}

\section{Introduction}

Federated learning \cite{mcmahan2016communication,kairouz2019advances,konevcny2016federated,konevcny2015federated,lim2020federated} is an emerging sub-area of distributed optimization where both data collection and model training is pushed to a large number of edge clients that have limited communication and computation capabilities. Unlike traditional distributed optimization \cite{li2014scaling,nedic2018network} where consensus (either through a central server or peer-to-peer communication) is performed after every local gradient computation, in federated learning, the subset of clients selected in each communication round perform multiple local updates before these models are aggregated in order to update a global model.

\textbf{Heterogeneity in the Number of Local Updates in Federated Learning. } The clients participating in federated learning are typically highly heterogeneous, both in the size of their local datasets as well as their computation speeds. The original paper on federated learning \cite{mcmahan2016communication} proposed that each client performs $E$ \emph{epochs} (traversals of their local dataset) of local-update stochastic gradient descent (SGD) with a mini-batch size $B$. Thus, if a client has $n_i$ local data samples, the number of local SGD iterations is $\tau_i = \lfloor E n_i/B \rfloor$, which can vary widely across clients. The heterogeneity in the number of local SGD iterations is exacerbated by relative variations in the clients' computing speeds. Within a given wall-clock time interval, faster clients can perform more local updates than slower clients. The number of local updates made by a client can also vary across communication rounds due to unpredictable straggling or slowdown caused by background processes, outages, memory limitations etc. Finally, clients may use different learning rates and local solvers (instead of vanilla SGD, they may use proximal gradient methods or adaptive learning rate schedules) which may result in heterogeneity in the model progress at each client.

\textbf{Heterogeneity in Local Updates Causes Objective Inconsistency.} Most recent works that analyze the convergence of federated optimization algorithms \cite{wang2018cooperative,stich2018local,zhou2017convergence,yu2018parallel,Li2020On,haddadpour2019convergence,haddadpour2019trading,haddadpour2019local,khaled2020tighter,stich2019error,wang2019adaptive,Wang2018Adaptive,karimireddy2019scaffold,liang2019variance,woodworth2020local,koloskova2020unified,yu2019linear,Wang2020SlowMo,huo2020faster,zhou2019distributed,zhang2020fedpd,pathak2020fedsplit,khaled2019first,woodworth2018graph,zhao2018federated,xie2019local,lin2018don,malinovsky2020local,wang2020overlap,dieuleveut2019communication} assume that number of local updates is the same across all clients (that is, $\cp_i = \cp$ for all clients $i$). These works show that periodic consensus between the locally trained client models attains a stationary point of the global objective function $F(\x) = \sum_{i=1}^{m} n_i F_i(\x)/n$, which is a sum of local objectives weighted by the dataset size $n_i$. However, none of these prior works provides insight into the convergence of local-update or federated optimization algorithms in the practical setting when the number of local updates $\tau_i$ varies across clients $1, \dots, m$. In fact, as we show in \Cref{sec:quadratic}, \emph{standard averaging of client models after heterogeneous local updates results in convergence to a stationary point -- not of the original objective function $F(\x)$, but of an inconsistent objective $\widetilde{F}(\x)$, which can be arbitrarily different from $F(\x)$ depending upon the relative values of $\cp_i$.} To gain intuition into this phenomenon, observe in \Cref{fig:illustration} that if client $1$ performs more local updates, then the updated $\x^{(t+1,0)}$ strays towards the local minimum $\x_1^*$, away from the true global minimum $\x^*$.

\begin{wrapfigure}{r}{0.3\textwidth}
    \begin{center}
        \includegraphics[width=.3\textwidth]{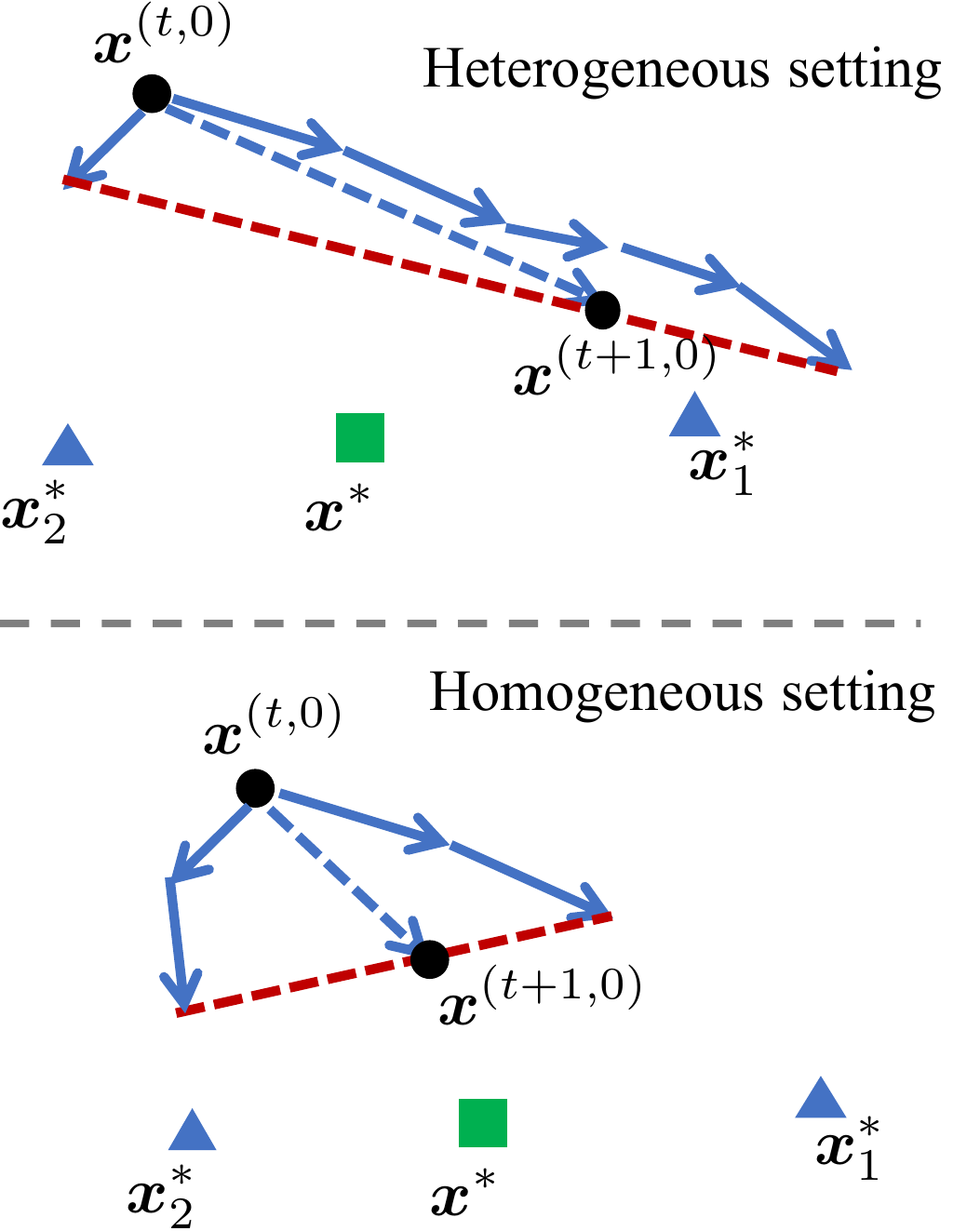}
    \end{center}
    \caption{Model updates in the parameter space. Green squares and blue triangles denote the minima of global and local objectives, respectively.}
    \label{fig:illustration}
    \vspace{-1em}
\end{wrapfigure}

\textbf{The Need for a General Analysis Framework.} A naive approach to overcome heterogeneity is to fix a target number of local updates $\tau$ that each client must finish within a communication round and keep fast nodes idle while the slow clients finish their updates. This method will ensure objective consistency (that is, the surrogate objective $\widetilde{F}(\x)$ equals to the true objective $F(\x)$), nonetheless, waiting for the slowest one can significantly increase the total training time. More sophisticated approaches such as \fedprox \cite{li2018federated}, \vrlsgd \cite{liang2019variance} and \scaffold \cite{karimireddy2019scaffold}, designed to handle non-IID local datasets, can be used to reduce (not eliminate) objective inconsistency to some extent, but these methods either result in slower convergence or require additional communication and memory. So far, there is no rigorous understanding of the objective inconsistency and the speed of convergence for this challenging setting of federated learning with heterogeneous local updates. It is also unclear how to best combine models trained with heterogeneous levels of local progress.

%
%

\textbf{Proposed Analysis Framework to Understand Bias Due to Objective Inconsistency.} To the best of our knowledge, this work provides the first fundamental understanding of the bias in the solution (caused by objective inconsistency) and how the convergence rate is influenced by heterogeneity in clients' local progress. In \Cref{sec:analysis} we propose a general theoretical framework that allows heterogeneous number of local updates, non-IID local datasets as well as different local solvers such as GD, SGD, SGD with proximal gradients, gradient tracking, adaptive learning rates, momentum, etc. It subsumes existing methods such as \fedavg and \fedprox and provides novel insights on their convergence behaviors. 

\textbf{Proposed Normalized Averaging Method FedNova.}
In \Cref{sec:algorithm} we propose \fednova, a method that correctly normalizes local model updates when averaging. The main idea of \fednova is that instead of averaging the cumulative local gradient $\x_i^{(t,\tau_i)} - \x^{(t,0)}$ returned by client $i$ (which performs $\tau_i$ local updates) in $t$-th training round, the aggregator averages the normalized local gradients $(\x_i^{(t,\tau_i)} - \x^{(t,0)})/\tau_i$. \fednova ensures objective consistency while preserving fast error convergence and outperforms existing methods as shown in \Cref{sec:exp}. It works with any local solver and server optimizer and is therefore complementary to existing approaches such as \cite{li2018federated, li2019feddane, karimireddy2019scaffold, reddi2020adaptive}. By enabling aggregation of models with heterogeneous local progress, \fednova gives the bonus benefit of overcoming the problem of stragglers, or unpredictably slow nodes by allowing fast clients to perform more local updates than slow clients within each communication round.
\section{System Model and Prior Work}
\textbf{The Federated Heterogeneous Optimization Setting.} In federated learning, a total of $m$ clients aim to jointly solve the following optimization problem:
\begin{align}
    \min_{\x\in\mathbb{R}^d} \brackets{\obj(\x) := \sum_{i=1}^{\nworkers}  p_i \obj_i(\x)} \label{eqn:global_obj}
\end{align}
where $p_i=n_i/n$ denotes the relative sample size, and $\obj_i(\x)=\frac{1}{n_i}\sum_{\xi \in \mathcal{D}_i}f_i(\x;\xi)$ is the local objective function at the $i$-th client. Here, $f_i$ is the loss function (possibly non-convex) defined by the learning model and $\xi$ represents a data sample from local dataset $\mathcal{D}_i$. In the $t$-th communication round, each client independently runs $\cp_i$ iterations of local solver (\eg, SGD) starting from the current global model $\x^{(t,0)}$ to optimize its own local objective. 

In our theoretical framework, we treat $\cp_i$ as an arbitrary scalar which can also vary across rounds. In practice, if clients run for the same local epochs $E$, then $\cp_i = \lfloor En_i/B\rfloor$, where $B$ is the mini-batch size. Alternately, if each communication round has a fixed length in terms of wall-clock time, then $\cp_i$ represents the local iterations completed by client $i$ within the time window and may change across clients (depending on their computation speeds and availability) and across communication rounds.

\textbf{The Fedavg Baseline Algorithm.} \emph{Federated Averaging} (\fedavg) \cite{mcmahan2016communication} is the first and most common algorithm used to aggregate these locally trained models at the central server at the end of each communication round. The shared global model is updated as follows:
\begin{align}
    \text{\fedavg:} \quad \x^{(t+1,0)} - \x^{(t,0)} = \sum_{i=1}^{\nworkers} p_i \Delta_i^{(t)} = - \sum_{i=1}^{\nworkers} p_i \cdot \lr \sum_{k=0}^{\cp_i-1}\sg_i(\x_i^{(t,k)}) \label{eqn:fedavg_update}
\end{align}
where $\x_i^{(t,k)}$ denotes client $i$'s model after the $k$-th local update in the $t$-th communication round and $\smash{\Delta_i^{(t)} = \x_i^{(t,\tau_i)} - \x_i^{(t,0)}}$ denotes the cumulative local progress made by client $i$ at round $t$. Also, $\lr$ is the client learning rate and $\sg_i$ represents the stochastic gradient over a mini-batch of $B$ samples. When the number of clients $\nworkers$ is large, then the central server may only randomly select a subset of clients to perform computation at each round. 

\textbf{Convergence Analysis of FedAvg.} 
\cite{wang2018cooperative,stich2018local,zhou2017convergence} first analyze \fedavg by assuming the local objectives are identical and show that \fedavg is guaranteed to converge to a stationary point of $\obj(\x)$. This analysis was further expanded to the non-IID data partition and client sampling cases by \cite{yu2018parallel,Li2020On,haddadpour2019convergence,haddadpour2019trading,haddadpour2019local,khaled2020tighter,stich2019error,wang2019adaptive,koloskova2020unified,yu2019linear}. 
However, in all these works, they assume that the number of local steps and the client optimizer are the same across all clients. Besides, asynchronous federated optimization algorithms proposed in \cite{xie2019asynchronous,stich2018local} take a different approach of allowing clients make updates to stale versions of the global model, and their analyses are limited to IID local datasets and convex local functions.

\textbf{FedProx: Improving FedAvg by Adding a Proximal Term.} 
To alleviate inconsistency due to non-IID data and heterogeneous local updates, \cite{li2018federated} proposes adding a proximal term $\frac{\mu}{2}\|\x - \x^{(t,0)}\|^2$ to each local objective, where $\mu\geq 0$ is a tunable parameter. This proximal term pulls each local model backward closer to the global model $\x^{(t,0)}$. Although \cite{li2018federated} empirically shows that \fedprox improves \fedavg, its convergence analysis is limited by assumptions that are stronger than previous \fedavg analysis and only works for sufficiently large $\mu$. Since \fedprox is a special case of our general framework, our convergence analysis provides sharp insights into the effect of $\mu$. We show that a larger $\mu$ mitigates (but does not eliminate) objective inconsistency, albeit at an expense of slower convergence. Our proposed \fednova method can improve \fedprox by guaranteeing consistency without slowing down convergence.

\textbf{Improving FedAvg via Momentum and Cross-client Variance Reduction.} The performance of \fedavg has been improved in recent literature by applying momentum on the server side \cite{Wang2020SlowMo,hsu2019measuring,reddi2020adaptive}, or using cross-client variance reduction such as \vrlsgd and \scaffold \cite{liang2019variance,karimireddy2019scaffold}. Again, these works do not consider heterogeneous local progress. Our proposed normalized averaging method \fednova is orthogonal to and can be easily combined with these acceleration or variance-reduction techniques. Moreover, \fednova is also compatible with and complementary to gradient compression/quantization \cite{reisizadeh2019fedpaq,basu2019qsparse,wang2018atomo,sattler2019robust,li2020acceleration,wu2020convergence} and fair aggregation techniques \cite{li2019fair,mohri2019agnostic}.

\section{A Case Study to Demonstrate the Objective Inconsistency Problem}
\label{sec:quadratic}

In this section, we use a simple quadratic model to illustrate the convergence problem. 
Suppose that the local objective functions are $\obj_i(\x) = \frac{1}{2}\|\x-\bm{e}_i\|^2$, where $\bm{e}_i\in \smash{\mathbb{R}^d}$ is an arbitrary vector and it is the minimum $\x_i^*$ of the local objective. Consider that the global objective function is defined as
\begin{align}
\obj(\x)= 
\frac{1}{\nworkers}\sum_{i=1}^\nworkers \obj_i(\x) = \sum_{i=1}^{\nworkers} \frac{1}{2\nworkers}\|\x-\bm{e}_i\|^2, \quad \text{which is minimized by } \x^* = \frac{1}{\nworkers}\sum_{i=1}^\nworkers \bm{e}_i. \label{eqn:quadratic_obj}
\end{align}
Below, we show that the convergence point of \fedavg can be arbitrarily away from $\x^*$.

\begin{lem}[\textbf{Objective Inconsistency in FedAvg}]
\label{lem:quadratic} 
For the objective function in \Cref{eqn:quadratic_obj}, if client $i$ performs $\cp_i$ local steps per round, then \fedavg (with sufficiently small learning rate $\lr$, deterministic gradients and full client participation) will converge to
\begin{align}
 \tilde{\x}^*_{\fedavg} = \lim_{T\rightarrow \infty} \x^{(T,0)}=  \frac{\sum_{i=1}^\nworkers\cp_i \bm{e}_i}{\sum_{i=1}^\nworkers \cp_i}, \text{which minimizes the surrogate obj.:} \surloss(\x)=\frac{\sum_{i=1}^\nworkers  \cp_i \obj_i(\x)}{\sum_{i=1}^\nworkers \cp_i}.\nonumber
\end{align}
\end{lem}
The proof (of a more general version of \Cref{lem:quadratic}) is deferred to the Appendix. 
While \fedavg aims at optimizing $\obj(\x)$, it actually converges to the optimum of a surrogate objective $\surloss(\x)$. 
As illustrated in \Cref{fig:quadratic}, there can be an arbitrarily large gap between $\tilde{\x}^*_{\fedavg}$ and $\x^*$ depending on the relative values of $\cp_i$ and $\obj_i(\x)$. This non-vanishing gap also occurs when the local steps $\tau_i$ are IID random variables across clients and communication rounds (see the right panel in \Cref{fig:quadratic}).

\textbf{Convergence Problem in Other Federated Algorithms.}  We can generalize \Cref{lem:quadratic} to the case of \fedprox to demonstrate its convergence gap, as given in \Cref{sec:proof_quadratic}. From the simulations shown in \Cref{fig:quadratic}, observe that \fedprox can slightly improve on the optimality gap of \fedavg, but it converges slower. Besides, previous cross-client variance reduction methods such as variance-reduced local SGD (\vrlsgd) \cite{liang2019variance} and \scaffold \cite{karimireddy2019scaffold} are only designed for homogeneous local steps case. In the considered heterogeneous setting, if we replace the same local steps $\cp$ in \vrlsgd by different $\cp_i$'s, then we observe that it has drastically different convergence under different settings and even diverge when clients perform random local steps (see the right panel in \Cref{fig:quadratic}). These observations emphasize the critical need for a deeper understanding of objective inconsistency and new federated heterogeneous optimization algorithms. 

\begin{figure}[t]
    \centering
    \begin{subfigure}{.33\textwidth}
    \centering
    \includegraphics[width=\textwidth]{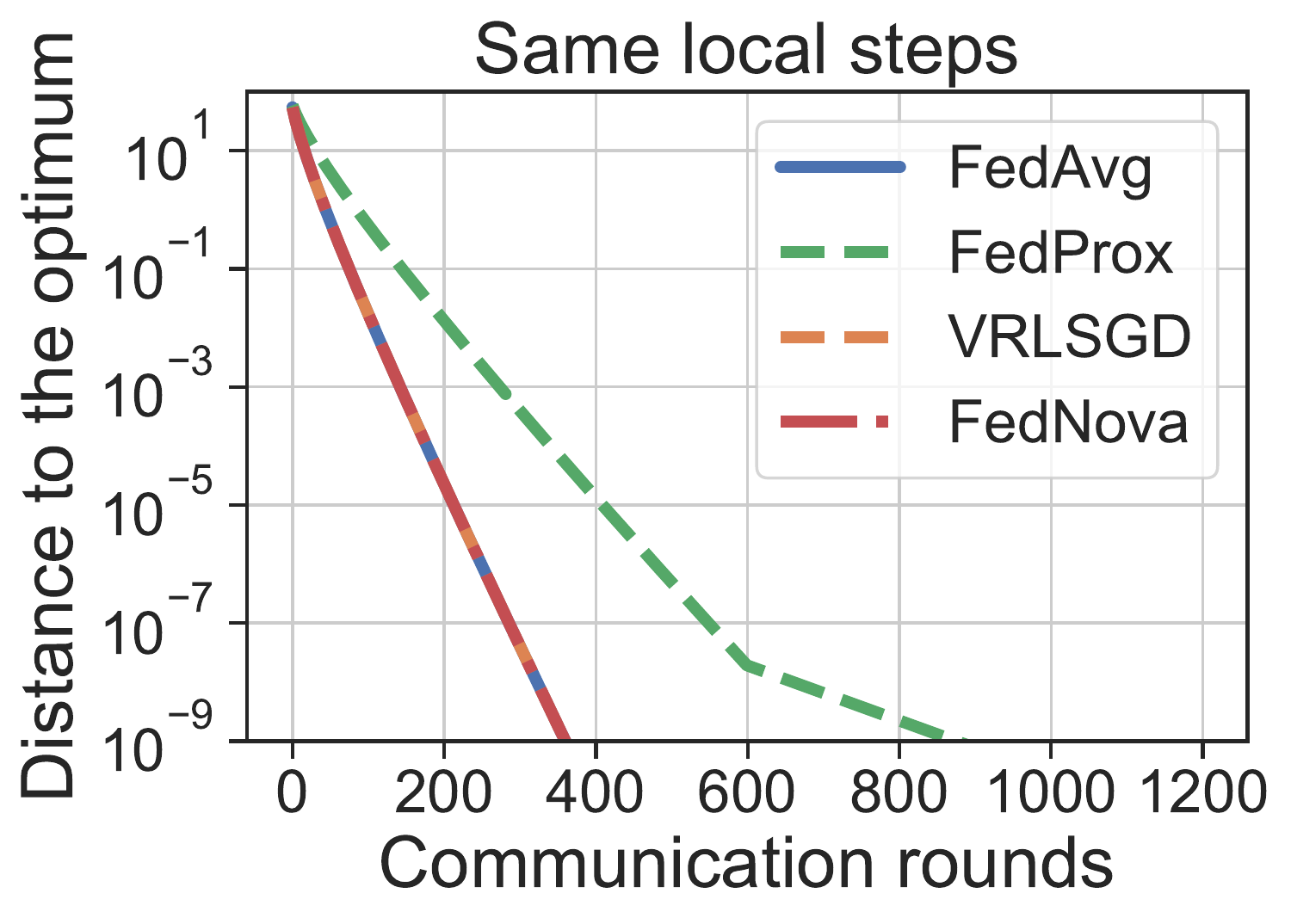}
    \end{subfigure}%
    ~
    \begin{subfigure}{.33\textwidth}
    \centering
    \includegraphics[width=\textwidth]{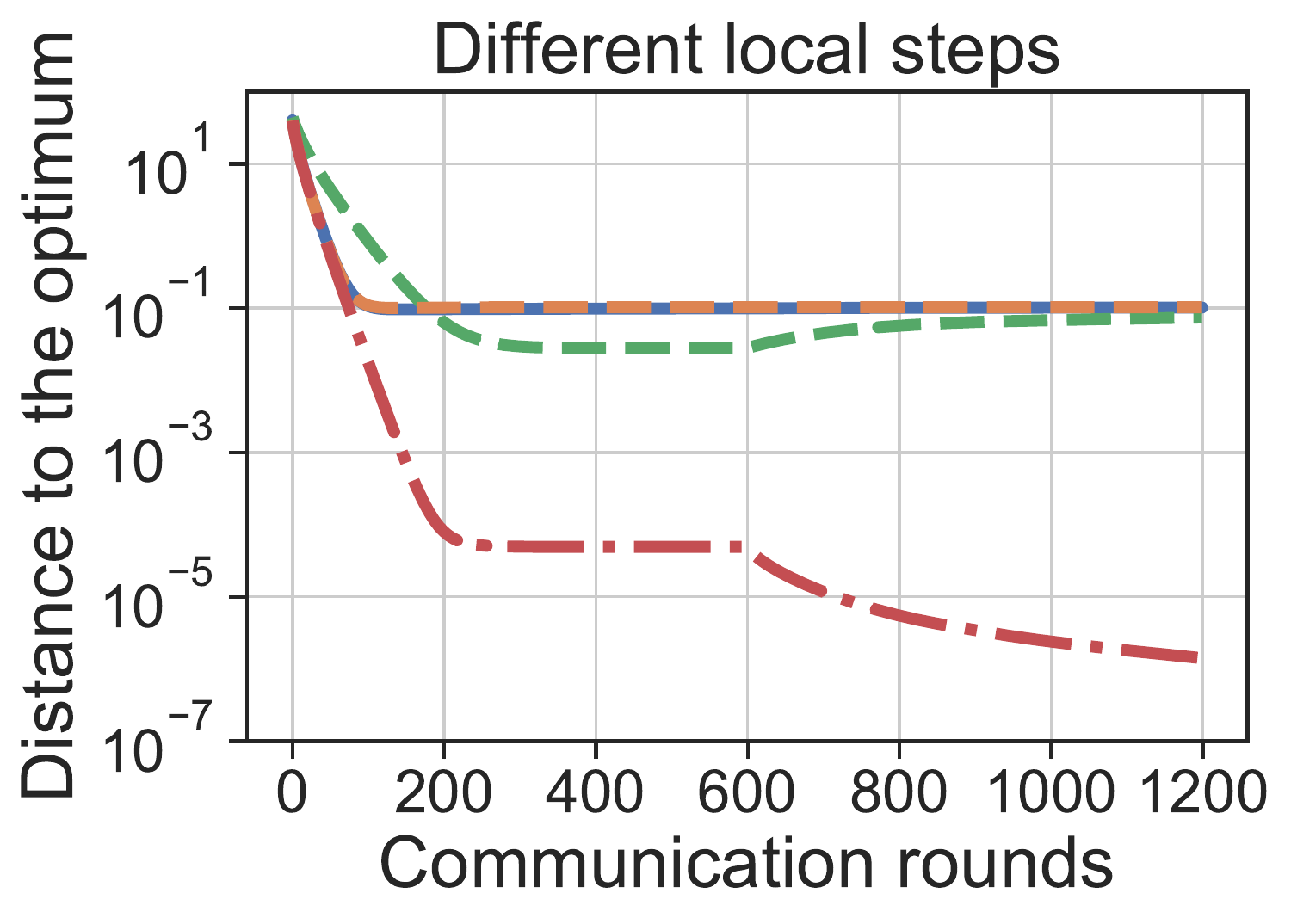}
    \end{subfigure}%
    ~
    \begin{subfigure}{.33\textwidth}
    \centering
    \includegraphics[width=\textwidth]{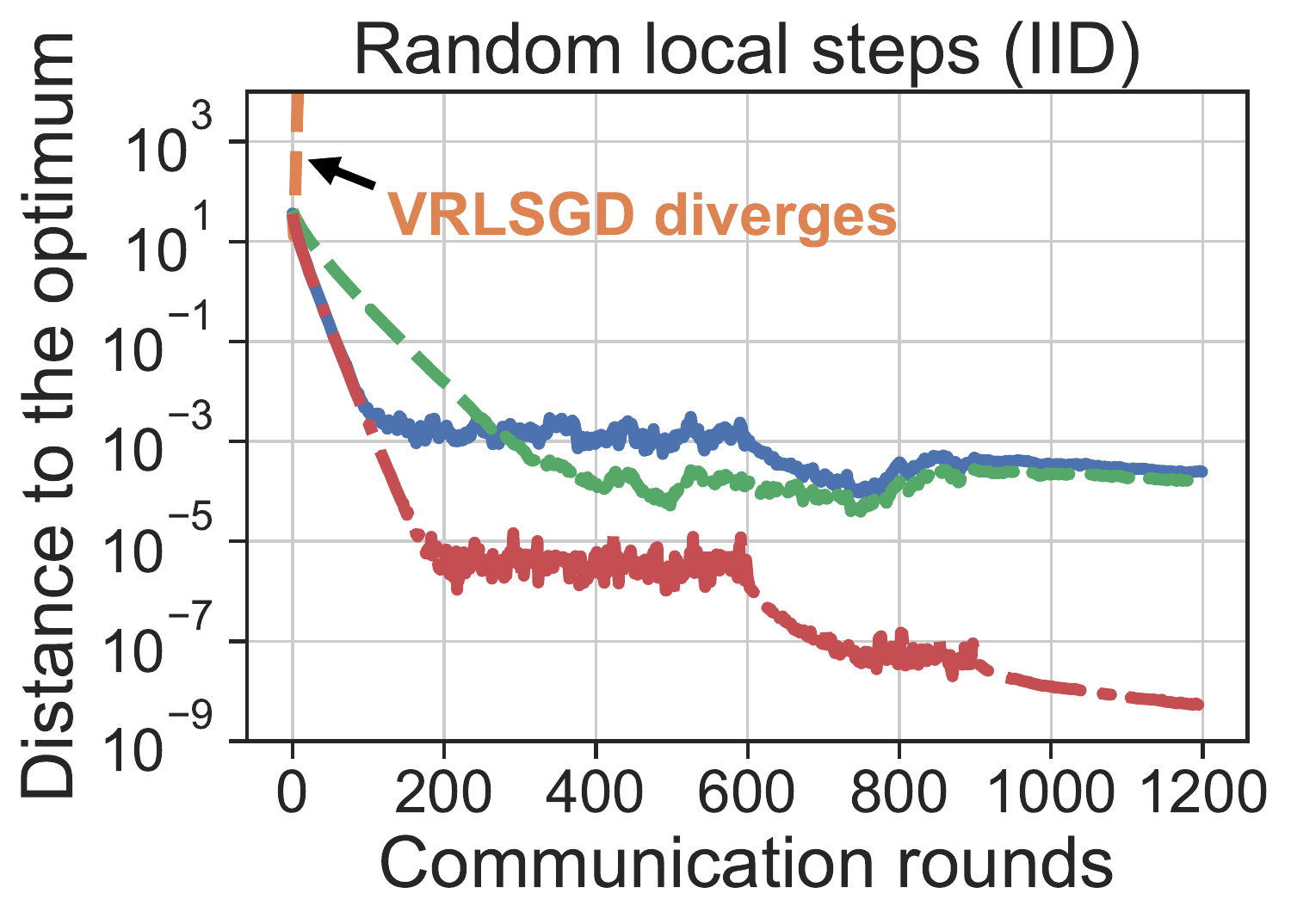}
    \end{subfigure}
    \caption{Simulations comparing the \fedavg, \fedprox ($\mu=1$), \vrlsgd and our proposed \fednova algorithms for $30$ clients with the quadratic objectives defined in \eqref{eqn:quadratic_obj}, where $\bm{e}_i \sim \mathcal{N}(0, 0.01\matI), i\in[1,30]$. Clients perform GD with $\eta = 0.05$, which is decayed by a factor of $5$ at rounds $600$ and $900$. \textbf{\emph{Left}}: Clients perform the same number of local steps $\cp_i=30$ -- \fednova is equivalent to \fedavg in this case; \textbf{\emph{Middle}}: Clients take different local steps $\cp_i\in[1,96]$ with mean $30$ but fixed across rounds; \textbf{\emph{Right}}: local steps are IID, and time-varying Gaussians with mean $30$, \ie, $\cp_i(t)\in[1,96]$. \fednova significantly outperforms others in the heterogeneous $\cp_i$ setting.}
    \label{fig:quadratic}
    \vspace{-1em}
\end{figure}
\section{New Theoretical Framework For Heterogeneous Federated Optimization}
\label{sec:analysis}

We now present a general theoretical framework that subsumes a suite of federated optimization algorithms and helps analyze the effect of objective inconsistency on their error convergence. Although the results are presented for the full client participation setting, it is fairly easy to extend them to the case where a subset of clients are randomly sampled in each round \footnote{In the case of client sampling, the update rule of \fedavg \Cref{eqn:fedavg_update} should hold in expectation in order to guarantee convergence \cite{Li2020On,haddadpour2019convergence,li2018federated,reddi2020adaptive}. One can achieve this by either (i) sampling $\samplenum$ clients with replacement with respect to probability $p_i$, and then averaging the cumulative local changes with equal weights, or (ii) sampling $\samplenum$ clients without replacement uniformly at random, and then weighted averaging local changes, where the weight of client $i$ is re-scaled to $p_i \nworkers/ \samplenum$. Our convergence analysis can be easily extended to these two cases.}. 


\subsection{A Generalized Update Rule for Heterogeneous Federated Optimization}\label{sec:formulation}

Recall from \eqref{eqn:fedavg_update} that the update rule of federated optimization algorithms can be written as $\smash{\x^{(t+1,0)} - \x^{(t,0)} = \sum_{i=1}^\nworkers p_i \Delta_i^{(t)}}$, where $\smash{\Delta_i^{(t)}:=\x^{(t,\cp_i)} - \x^{(t,0)}}$ denote the local parameter changes of client $i$ at round $t$ and $p_i = n_i/n$, the fraction of data at client $i$. We re-write this update rule in a more general form as follows:
\begin{align}
    \x^{(t+1,0)} - \x^{(t,0)} = - \cpeff \sum_{i=1}^\nworkers w_i \cdot \lr \nsg_i^{(t)}, \quad \text{which optimizes } \surloss(\x) = \sum_{i=1}^{\nworkers}  w_i F_i(\x). \label{eqn:new_update}
\end{align}
The three key elements $\cpeff, w_i$ and $\nsg_i^{(t)}$ of this update rule take different forms for different algorithms. Below, we provide detailed descriptions of these key elements.
\begin{enumerate}
    \item \textbf{Normalized gradient} $\nsg_i^{(t)}$: The normalized gradient is defined as $\smash{\nsg_i^{(t)} = \matG_i^{(t)}\gradweight_i/\|\gradweight_i\|_1}$, where $\matG_i^{(t)} = [\sg_i(\x_i^{(t,0)}), \sg_i(\x_i^{(t,1)}),\dots, \sg_i(\x_i^{(t,\cp_i)})] \in \mathbb{R}^{d \times \cp_i}$ stacks all stochastic gradients in the $t$-th round, and $\gradweight_i\in \mathbb{R}^{\cp_i}$ is a non-negative vector and defines how stochastic gradients are locally accumulated. The normalizing factor $\|\gradweight_i\|_1$ in the denominator is the $\ell_1$ norm of the vector $\gradweight_i$. By setting different $\gradweight_i$, \Cref{eqn:new_update} works for most common client optimizers such as SGD with proximal updates, local momentum, and variable learning rate, and more generally, any solver whose accumulated gradient $\smash{\Delta_i^{(t)} = -\lr \matG_i^{(t)}\gradweight_i}$, a linear combination of local gradients. 
    
    If the client optimizer is vanilla SGD (\ie, the case of \fedavg), then $\gradweight_i=[1,1,\dots,1]\in\mathbb{R}^{\cp_i}$ and $\|\gradweight_i\|_1=\cp_i$. As a result, the normalized gradient is just a simple average of all stochastic gradients within current round: $\nsg_i^{(t)} = \matG_i^{(t)}\gradweight_i/\cp_i = \sum_{k=0}^{\cp_i-1}\sg_i(\x_i^{(t,k)})/\cp_i$. Later in this section, we will present more specific examples on how to set $\gradweight_i$ in other algorithms.
    
    \item \textbf{Aggregation weights} $w_i$: Each client's normalized  gradient $\nsg_i$ is multiplied with weight $w_i$ when computing the aggregated gradient $\sum_{i=1}^{m} w_i \nsg_i$. By definition, these weights satisfy $\sum_{i=1}^{\nworkers} w_i = 1$. Observe that these weights determine the surrogate objective $\surloss(\x) = \sum_{i=1}^{m} w_i F_i(\x)$, which is optimized by the general algorithm in \eqref{eqn:new_update} instead of the true global objective $F(\x) = \sum_{i=1}^{m} p_i F_i(\x)$ -- we will prove this formally in \Cref{thm:general}. 
    
    \item \textbf{Effective number of steps} $\cpeff$: Since client $i$ makes $\tau_i$ local updates, the average number of local SGD steps per communication round is $\bar{\cp} = \sum_{i=1}^{m} \cp_i/m$. However, the server can scale up or scale down the effect of the aggregated updates by setting the parameter $\cpeff$ larger or smaller than $\bar{\cp}$ (analogous to choosing a global learning rate \cite{Wang2020SlowMo,reddi2020adaptive}). We refer to the ratio $\bar{\cp}/\cpeff$ as the slowdown, and it features prominently in the convergence analysis presented in \Cref{sec:conv_analysis}. 
\end{enumerate}
The general rule \eqref{eqn:new_update} enables us to freely choose $\cpeff$ and $w_i$ for a given local solver $\gradweight_i$, which helps design fast and consistent algorithms such as \fednova, the normalized averaging method proposed in \Cref{sec:algorithm}. In \Cref{fig:generalized_updates}, we further illustrate how the above key elements influence the algorithm and compare the novel generalized update rule and \fedavg in the model parameter space. Besides, in terms of the implementation of the generalized update rule, each client can send the normalized update $-\eta \smash{\nsg_i^{(t)}}$ to the central server, which is just a re-scaled version of  $\smash{\Delta_i^{(t)}}$, the accumulated local parameter update sent by clients in the vanilla update rule \eqref{eqn:fedavg_update}. The server is not necessary to know the specific form of local accumulation vector $\gradweight_i$.

\textbf{Previous Algorithms as Special Cases.}   
Any previous algorithm whose accumulated local changes $\smash{\Delta_i^{(t)} = -\lr \matG_i^{(t)}\gradweight_i}$, a linear combination of local gradients is subsumed by the above formulation. One can validate this as follows:
\begin{align}
    \x^{(t+1,0)} - \x^{(t,0)}
    =& \sum_{i=1}^\nworkers p_i \Delta_i^{(t)} = -\sum_{i=1}^\nworkers p_i \|\gradweight_i\|_1  \cdot \frac{\lr \matG_i^{(t)}\gradweight_i}{\|\gradweight_i\|_1}\\
    =& - \underbrace{\parenth{\sum_{i=1}^\nworkers p_i \|\gradweight_i\|_1}}_{\cpeff:\ \text{effective local steps}}\sum_{i=1}^\nworkers \lr \underbrace{\parenth{\frac{p_i\|\gradweight_i\|_1}{\sum_{i=1}^\nworkers p_i \|\gradweight_i\|_1}}}_{w_i:\ \text{weight}}\underbrace{\parenth{\frac{\matG_i^{(t)}\gradweight_i}{\|\gradweight_i\|_1}}}_{\nsg_i:\ \text{normalized gradient}}. \label{eqn:new_form1}
\end{align}
Unlike the more general form \Cref{eqn:new_update}, in \eqref{eqn:new_form1}, which subsumes the following previous methods, $\cpeff$ and $w_i$ are implicitly fixed by the choice of the local solver (\ie, the choice of $\gradweight_i$). Due to space limitations, the derivations of following examples are relegated to \Cref{sec:proof_detailed_gradweight}.
\begin{itemize}
    \item \textbf{Vanilla SGD as Local Solver (FedAvg).} In \fedavg, the local solver is SGD such that $\gradweight_i = [1,1,\dots,1]\in\mathbb{R}^{\cp_i}$ and $\|\gradweight_i\|_1 =\cp_i$. As a consequence, the normalized gradient $\nsg_i$ is a simple average over $\cp_i$ iterations, $\cpeff=\sum_{i=1}^\nworkers p_i \cp_i$, and $w_i = p_i \cp_i/\sum_{i=1}^\nworkers p_i \cp_i$. That is, the normalized gradients with more local steps will be implicitly assigned higher weights.
    \item \textbf{Proximal SGD as Local Solver (FedProx).} In \fedprox, local SGD steps are corrected by a proximal term. It can be shown that $\gradweight_i = [(1-\alpha)^{\cp_i-1},(1-\alpha)^{\cp_i-2},\dots,(1-\alpha),1]\in\mathbb{R}^{\cp_i}$, where $\alpha=\lr\mu$ and $\mu\geq 0$ is a tunable parameter. In this case, we have $\|\gradweight_i\|_1=[1-(1-\alpha)^{\cp_i}]/\alpha$ and hence,
    \begin{align}
        \cpeff = \frac{1}{\alpha}\sum_{i=1}^\nworkers p_i [1-(1-\alpha)^{\cp_i}], \quad w_i = \frac{p_i [1-(1-\alpha)^{\cp_i}]}{\sum_{i=1}^\nworkers p_i [1-(1-\alpha)^{\cp_i}]}. \label{eqn:fedprox_w}
    \end{align}
    When $\alpha = 0$, \fedprox is equivalent to \fedavg. As $\alpha = \eta \mu$ increases, the $w_i$ in \fedprox is more similar to $p_i$, thus making the surrogate objective $\widetilde{F}(\x)$ more consistent. However, a larger $\alpha$ corresponds to smaller $\cpeff$, which slows down convergence, as we discuss more in the next subsection.
    \item \textbf{SGD with Decayed Learning Rate as Local Solver.} Suppose the clients' local learning rates are exponentially decayed, then we have $\gradweight_i=[1,\gamma_i, \dots,\gamma_i^{\cp_i-1}]$ where $\gamma_i \geq 0$ can vary across clients. As a result, we have $\|\gradweight_i\|_1 = (1-\gamma_i^{\cp_i})/(1-\gamma_i)$ and $w_i \propto p_i(1-\gamma_i^{\cp_i})/(1-\gamma_i)$. Comparing with the case of \fedprox \Cref{eqn:fedprox_w}, changing the values of $\gamma_i$ has a similar effect as changing $(1-\alpha)$.
    \item \textbf{Momentum SGD as Local Solver.} If we use momentum SGD where the local momentum buffers of active clients are reset to zero at the beginning of each round \cite{Wang2020SlowMo} due to the stateless nature of FL \cite{kairouz2019advances}, then we have $\gradweight_i =[1-\rho^{\cp_i},1-\rho^{\cp_i-1},\dots,1-\rho]/(1-\rho)$, where $\rho$ is the momentum factor, and $\|\gradweight_i\|_1 = [\cp_i - \rho(1-\rho^{\cp_i})/(1-\rho)]/(1-\rho)$.
\end{itemize}
More generally, the new formulation \Cref{eqn:new_form1} suggests that $w_i \neq p_i$ whenever clients have different $\|\gradweight_i\|_1$, which may be caused by imbalanced local updates (\ie, $\gradweight_i$'s have different dimensions), or various local learning rate/momentum schedules (\ie, $\gradweight_i$'s have different scales). 

\begin{figure}[t]
    \centering
    \includegraphics[width=.9\textwidth]{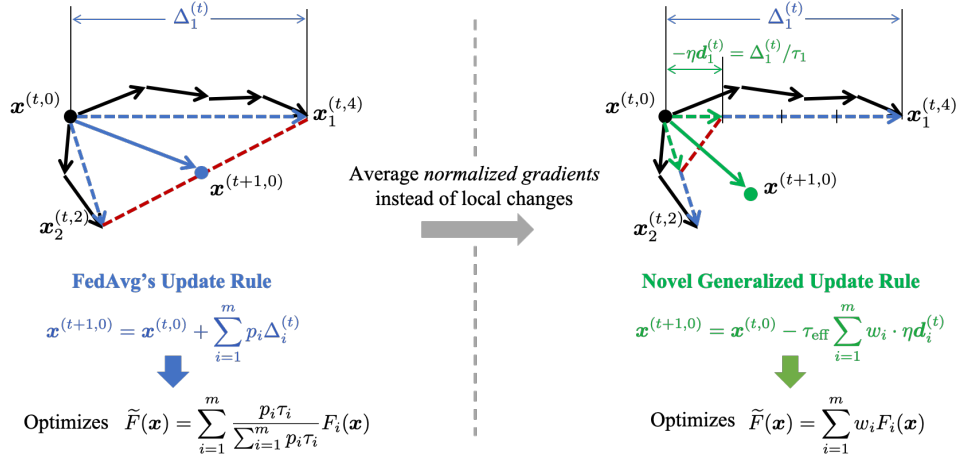}
    \caption{Comparison between the novel framework and \fedavg in the model parameter space. Solid black arrows denote local updates at clients. Solid green and blue arrows denote the global updates made by the novel generalized update rule and \fedavg respectively. While $w_i$ controls the direction of the solid green arrow, effective steps $\cpeff$ determines how far the global model moves along with this direction. In \fedavg, local changes are averaged based on the sizes of local datasets. However, this strategy implicitly assigns too higher weights for clients with more local steps, resulting in a biased global direction.} 
    \label{fig:generalized_updates}
    \vspace{-1em}
\end{figure}

\subsection{Convergence Analysis for Smooth Non-Convex Functions}
\label{sec:conv_analysis}

In \Cref{thm:general} and  \Cref{thm:bias} below we provide a convergence analysis for the general update rule \Cref{eqn:new_update} and quantify the solution bias due to objective inconsistency. The analysis relies on  \Cref{assump:smooth,assump:var} used in the standard analysis of SGD \cite{bottou2016optimization} and \Cref{assump:dissimilarity} commonly used in the federated optimization literature \cite{li2018federated,Li2020On,haddadpour2019convergence,karimireddy2019scaffold,reddi2020adaptive,wang2019matcha,kairouz2019advances} to capture the dissimilarities of local objectives.

\begin{assump}[Smoothness]\label{assump:smooth}
Each local objective function is Lipschitz smooth, that is, $\vecnorm{\tg_i(\x) - \tg_i(\by)} \leq \lip \vecnorm{\x - \by}, \forall i \in \{1, 2, \dots, \nworkers\}$.
\end{assump}
\begin{assump}[Unbiased Gradient and Bounded Variance] \label{assump:var}
The stochastic gradient at each client is an unbiased estimator of the local gradient: $\Exs_{\xi}[\sg_i(\x|\xi)] = \tg_i(\x)$, and has bounded variance $\Exs_\xi[\vecnorm{\sg_i(\x|\xi) - \tg_i(\x)}^2] \leq \vbnd, \forall i \in \{1,2,\dots,\nworkers\}, \vbnd \geq 0$. 
\end{assump}
\begin{assump}[Bounded Dissimilarity]\label{assump:dissimilarity}
For any sets of weights $\{w_i\geq 0\}_{i=1}^\nworkers, \sum_{i=1}^\nworkers w_i =1$, there exist constants $\bnda \geq 1, \bndb \geq 0$ such that $\sum_{i=1}^\nworkers w_i \vecnorm{\tg_i(\x)}^2 \leq \bnda\vecnorm{\sum_{i=1}^\nworkers w_i\tg_i(\x)}^2 + \bndb$. If local functions are identical to each other, then we have $\bnda = 1, \bndb = 0$. 
\end{assump} 
Our main theorem is stated as follows.
\begin{thm}[\textbf{Convergence to the Surrogate Objective $\widetilde{F}(\x)$'s Stationary Point}]\label{thm:general}
Under \Cref{assump:smooth,assump:var,assump:dissimilarity}, any federated optimization algorithm that follows the update rule \Cref{eqn:new_update}, will converge to a stationary point of a surrogate objective $\surloss(\x)=\sum_{i=1}^\nworkers w_i \obj_i(\x)$. More specifically, if the total communication rounds $T$ is pre-determined and the learning rate $\eta$ is small enough $\lr = \sqrt{\nworkers/\cpavg T}$ where $\cpavg=\frac{1}{\nworkers}\sum_{i=1}^\nworkers \cp_i$, then the optimization error will be bounded as follows:
\begin{align}
    \min_{t\in[T]} \Exs\| \nabla \surloss(\x^{(t,0)}) \|^2
    \leq& \underbrace{\mathcal{O}\parenth{\frac{\cpavg/\cpeff}{\sqrt{\nworkers \cpavg T}}} + \mathcal{O}\parenth{\frac{A\vbnd}{\sqrt{\nworkers \cpavg T}} } + \mathcal{O}\parenth{\frac{\nworkers B\vbnd}{\cpavg T}} + \mathcal{O}\parenth{\frac{\nworkers C \bndb}{\cpavg T}}}_{\text{denoted by }\epsilon_{\text{opt}} \text{ in } \eqref{eqn:error_decomp}} \label{eqn:general_error}
\end{align}
where $\mathcal{O}$ swallows all constants (including $\lip$), and quantities $A, B, C$ are defined as follows: 
\begin{align}
    A = \nworkers\cpeff\sum_{i=1}^\nworkers \frac{w_i^2 \vecnorm{\bm{a}_i}_2^2}{\vecnorm{\bm{a}_i}_1^2}, \
    B = \sum_{i=1}^\nworkers w_i (\vecnorm{\bm{a}_i}_2^2 - a_{i,-1}^2), \
    C = \max_i\{\vecnorm{\bm{a}_i}_1^2-\vecnorm{\bm{a}_i}_1 a_{i,-1} \}
\end{align}
where $a_{i,-1}$ is the last element in the vector $\gradweight_i$.
\end{thm}
In \Cref{sec:proof_thm1}, we also provide another version of this theorem that explicitly contains the local learning rate $\lr$. Moreover, since the surrogate objective $\surloss(\x)$ and the original objective $\obj(\x)$ are just different linear combinations of the local functions, once the algorithm converges to a stationary point of $\surloss(\x)$, one can also obtain some guarantees in terms of $\obj(\x)$, as given by \Cref{thm:bias} below. 
\begin{thm}[\textbf{Convergence in Terms of the True Objective $F(\x)$}]\label{thm:bias}
Under the same conditions as \Cref{thm:general}, the minimal gradient norm of the true global objective function $\obj(\x)=\sum_{i=1}^\nworkers p_i \obj_i(\x)$ will be bounded as follows:
\begin{align}
    \min_{t\in[T]} \|\nabla \obj(\x^{(t,0)})\|^2
    \leq \underbrace{2\brackets{\csqdist(\bnda-1)+1}\epsilon_{\text{opt}}}_{\text{vanishing error term } } + \underbrace{2\csqdist \bndb}_{\text{non-vanishing error due to obj. inconsistency}} \label{eqn:error_decomp}
\end{align}
where $\epsilon_{\text{opt}}$ denotes the vanishing optimization error given by \Cref{eqn:general_error} and $\smash{\csqdist}=\sum_{i=1}^\nworkers (p_i-w_i)^2/w_i$ represents the chi-square divergence between vectors $\bm{p}=[p_1,\dots,p_\nworkers]$ and $\bm{w}=[w_1,\dots,w_\nworkers]$.
\end{thm}
\textbf{Discussion:}
\Cref{thm:general,thm:bias} describe the convergence behavior of a broad class of federated heterogeneous optimization algorithms. Observe that when all clients take the same number of local steps using the same local solver, we have $\bm{p} = \bm{w}$ such that $\chi^2 =0$. Also, when all local functions are identical to each other, we have $\bnda=1, \bndb=0$. Only in these two special cases, is there no objective inconsistency. For most other algorithms subsumed by the general update rule in \eqref{eqn:new_update}, both $w_i$ and $\cpeff$ are influenced by the choice of $\gradweight_i$. When clients have different local progress (\ie, different $\gradweight_i$ vectors), previous algorithms will end up with a non-zero error floor $\chi^2\bndb$, which does not vanish to $0$ even with sufficiently small learning rate. In \Cref{sec:lower_bound}, we further construct a lower bound and show that $\lim_{T\rightarrow\infty}\min_{t\in[T]} \|\nabla \obj(\x^{(t,0)})\|^2 = \Omega(\chi^2\bndb)$, suggesting \Cref{eqn:error_decomp} is tight.

\textbf{Consistency with Previous Results.}
When the client optimizer is fixed as vanilla SGD (\ie, $\gradweight_i=[1,1,\dots,1]\in\mathbb{R}^{\cp_i}$), then \cref{thm:general,thm:bias} gives the convergence guarantee for \fedavg. In this case, $\|\gradweight_i\|_1 = \cp_i$. As a consequence, the quantities $A,B,C$ can be written as follows:
\begin{align}
    A_{\fedavg}= \frac{\nworkers\sum_{i=1}^\nworkers p_i^2 \cp_i}{\Exs_{\bm{p}}[\bm{\cp}]}, \
    B_{\fedavg}= \Exs_{\bm{p}}[\bm{\cp}] - 1 + \frac{\var_{\bm{p}}[\bm{\cp}]}{\Exs_{\bm{p}}[\bm{\cp}]}, \
    C_{\fedavg}=\cpmax(\cpmax-1)
\end{align}
where $\Exs_{\bm{p}}[\bm{\cp}]=\sum_{i=1}^\nworkers p_i \cp_i$ and $\var_{\bm{p}}[\bm{\cp}]=\sum_{i=1}^\nworkers p_i \cp_i^2 - [\sum_{i=1}^\nworkers p_i \cp_i]^2$. If all clients perform the same local steps $\cp_i=\cp$ and have the same amount of data $p_i=1/\nworkers$, then $w_i=p_i=1/\nworkers, \cpeff = \cp, A_\fedavg=1, B_\fedavg = \cp-1, C_\fedavg=\cp(\cp-1)$. Substituting these values back into \Cref{thm:general}, it recovers previous results of \fedavg \cite{yu2019linear,wang2018cooperative,karimireddy2019scaffold}. If one further sets $\cp=1$, then $B_\fedavg=C_\fedavg=0$, \Cref{thm:general} recovers the result of synchronous mini-batch SGD \cite{bottou2016optimization}.


\textbf{Novel Insights Into the Convergence of FedProx and the Effect of $\mu$.}
Recall that in \fedprox $\gradweight_i=[(1-\alpha)^{\cp_i-1},\dots,(1-\alpha),1]$, where $\alpha =\lr\mu$. Accordingly, substituting the effective steps and aggregated weight, given by \Cref{eqn:fedprox_w}, into \Cref{eqn:general_error,eqn:error_decomp}, we get the convergence guarantee for \fedprox. Again, it has objective inconsistency because $w_i\neq p_i$. As we increase $\alpha$, the weights $w_i$ come closer to $p_i$ and thus, the non-vanishing error $\chi^2\bndb$ in \Cref{eqn:error_decomp} decreases (see blue curve in \Cref{fig:tradeoff}). However increasing $\alpha$ worsens the slowdown $\cpavg/\cpeff$, which appears in the first error term in \Cref{eqn:general_error} (see the red curve in \Cref{fig:tradeoff}). In the extreme case when $\alpha=1$, although \fedprox achieves objective consistency, it has a significantly slower convergence because $\cpeff = 1$ and the first term in \Cref{eqn:general_error} is $\cpavg$ times larger than that with \fedavg (eq. to $\alpha=0$).
\begin{figure}[!ht]
    \centering
    \includegraphics[width=.45\textwidth]{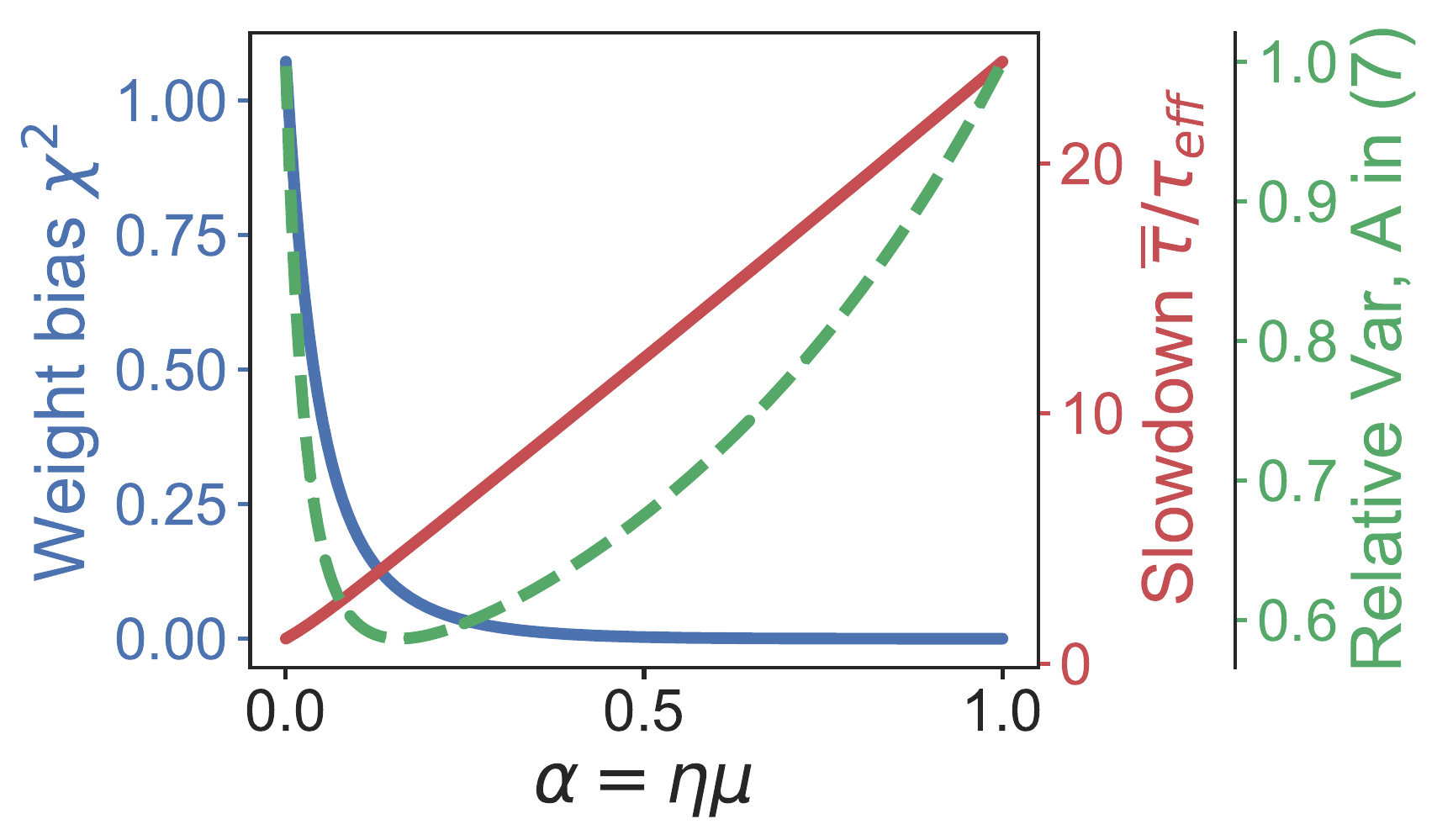}
    \caption{Illustration on how the parameter $\alpha=\lr\mu$ influences the convergence of \fedprox. We set $\nworkers=30, p_i=1/\nworkers, \cp_i\sim\mathcal{N}(20,20)$. `\emph{Weight bias}' denotes the chi-square distance between $\bm{p}$ and $\bm{w}$. `\emph{Slowdown}' and '\emph{Relative Variance}' quantify how the first and the second terms in \Cref{eqn:general_error} change.}
    \label{fig:tradeoff}
\end{figure}

\Cref{thm:general} also reveals that, in \fedprox, there should exist a best value of $\alpha$ that balances all terms in \Cref{eqn:general_error}. In Appendix, we provide a corollary showing that $\alpha=\mathcal{O}(\nicefrac{\nworkers^{\frac{1}{2}}}{\cpavg^{\frac{1}{2}}T^{\frac{1}{6}}})$ optimizes the error bound \Cref{eqn:general_error} of \fedprox and yields a convergence rate of $\mathcal{O}(\nicefrac{1}{\sqrt{\nworkers\cpavg T}} + \nicefrac{1}{T^{\frac{2}{3}}})$ on the surrogate objective. This can serve as a guideline on setting $\alpha$ in practice.

\textbf{Linear Speedup Analysis.} Another implication of \Cref{thm:general} is that when the communication rounds $T$ is sufficiently large, then the convergence of the surrogate objective will be dominated by the first two terms in \Cref{eqn:general_error}, which is $\nicefrac{1}{\sqrt{\nworkers \cpavg T}}$. This suggests that the algorithm only uses $T/\gamma$ total rounds when using $\gamma$ times more clients (\ie, achieving linear speedup) to reach the same error level. 

\section{FedNova: Proposed Federated Normalized Averaging Algorithm}\label{sec:algorithm}
\Cref{thm:general,thm:bias} suggest an extremely simple solution to overcome the problem of objective inconsistency. When we set $w_i=p_i$ in \Cref{eqn:new_update}, then the second non-vanishing term $\smash{\csqdist}\bndb$ in \Cref{eqn:error_decomp} will just become zero. This simple intuition yields the following new algorithm:
\begin{align}
    \fednova \quad 
    \x^{(t+1,0)} - \x^{(t,0)}
    = -\cpeff^{(t)} \sum_{i=1}^\nworkers p_i \cdot \lr\nsg_i^{(t)} \quad \text{where} \ \nsg_i^{(t)} = \frac{\matG_i^{(t)}\gradweight_i^{(t)}}{\|\gradweight_i^{(t)}\|_1} \label{eqn:nova_updates}
\end{align}
The proposed algorithm is named \emph{federated normalized averaging} (\fednova), because the normalized stochastic gradients $\nsg_i$ are averaged/aggregated instead of the local changes $\Delta_i=-\lr\matG_i\gradweight_i$. When the local solver is vanilla SGD, then $\gradweight_i = [1,1,\dots,1]\in\mathbb{R}^{\cp_i}$ and $\smash{\nsg_i^{(t)}}$ is a simple average over current round's gradients. In order to be consistent with \fedavg whose update rule is \Cref{eqn:new_form1}, one can simply set $\cpeff^{(t)}=\sum_{i=1}^\nworkers p_i \cp_i^{(t)}$. Then, in this case, the update rule of \fednova is equivalent to $\x^{(t+1,0)} - \x^{(t,0)} = (\sum_{i=1}^\nworkers p_i \cp_i^{(t)}) \sum_{i=1}^\nworkers p_i\Delta_i^{(t)}/\cp_i^{(t)}$. Comparing to previous algorithm $\x^{(t+1,0)} - \x^{(t,0)} = \sum_{i=1}^\nworkers p_i \Delta_i^{(t)}$, each accumulative local change $\Delta_i$ in \fednova is re-scaled by $(\sum_{i=1}^\nworkers p_i \cp_i^{(t)})/\cp_i^{(t)}$. This simple tweak in the aggregation weights eliminates inconsistency in the solution and gives better convergence than previous methods.

\textbf{Flexibility in Choosing Hyper-parameters and Local Solvers.} Besides vanilla SGD, the new formulation of \fednova naturally allows clients to choose various local solvers (\ie, client-side optimizer). As discussed in \Cref{sec:formulation}, the local solver can also be GD/SGD with decayed local learning rate, GD/SGD with proximal updates, GD/SGD with local momentum, etc. Furthermore, the value of $\cpeff$ is not necessarily to be controlled by the local solver as previous algorithms. For example, when using SGD with proximal updates, one can simply set $\cpeff = \sum_{i=1}^\nworkers p_i \cp_i$ instead of its default value $\sum_{i=1}^\nworkers p_i [1-(1-\alpha)^{\cp_i}]/\alpha$. This can help alleviate the slowdown problem discussed in \Cref{sec:conv_analysis}.

\textbf{Combination with Acceleration Techniques.} If clients have additional communication bandwidth, they can use cross-client variance reduction techniques to further accelerate the training \cite{liang2019variance,karimireddy2019scaffold,li2019feddane}. In this case, each local gradient step at the $t$-round will be corrected by $\sum_{i=1}^\nworkers p_i \nsg_i^{(t-1)} - \nsg_i^{(t-1)}$. That is, the local gradient at the $k$-th local step becomes $\sg_i(\x^{(t,k)}) + \sum_{i=1}^\nworkers p_i \nsg_i^{(t-1)} - \nsg_i^{(t-1)}$. Besides, on the server side, one can also implement server momentum or adaptive server optimizers \cite{Wang2020SlowMo,hsu2019measuring,reddi2020adaptive}, in which the aggregated normalized gradient $-\cpeff\sum_{i=1}^\nworkers \lr p_i\nsg_i$ is used to update the server momentum buffer instead of directly updating the server model.

\textbf{Convergence Analysis.} In \fednova, the local solvers at clients do not necessarily need to be the same or fixed across rounds. In the following theorem, we obtain strong convergence guarantee for \fednova, even with \emph{arbitrarily time-varying} local updates and client optimizers.
\begin{thm}[\textbf{Convergence of \fednova to a Consistent Solution}]\label{thm:nova}
Suppose that each client performs arbitrary number of local updates $\cp_i(t)$ using arbitrary gradient accumulation method $\gradweight_i(t), t\in[T]$ per round. Under \Cref{assump:smooth,assump:var,assump:dissimilarity}, and local learning rate as $\lr=\sqrt{\nworkers/(\widetilde{\cp}T)}$, where $\widetilde{\cp}= \sum_{t=0}^{T-1}\cpavg(t)/T$ denotes the average local steps over all rounds at clients, then \fednova converges to a stationary point of $\obj(\x)$ in a rate of $\mathcal{O}(1/\sqrt{\nworkers \widetilde{\cp} T})$. The detailed bound is the same as the right hand side of \Cref{eqn:general_error}, except that $\cpavg,A,B,C$ are replaced by their average values over all rounds.
\end{thm}
Using the techniques developed in \cite{Li2020On,karimireddy2019scaffold,haddadpour2019convergence}, \Cref{thm:nova}  can be further generalized to incorporate client sampling schemes. We provide a corresponding corollary in \Cref{sec:proof_sampling}. When $\samplenum$ clients are selected per round, then the convergence rate of \fednova is $\mathcal{O}(1/\sqrt{\samplenum \widetilde{\cp} T})$, where $\widetilde{\cp} = \sum_{t=0}^{T-1}\sum_{i \in \mathcal{S}^{(t)}} \cp_i^{(t)}/(q T)$ and $\mathcal{S}^{(t)}$ is the set of selected client indices at the $t$-th round.

\section{Experimental Results}\label{sec:exp}
\textbf{Experimental Setup.} We evaluate all algorithms on two setups with non-IID data partitioning: (1) \emph{Logistic Regression on a Synthetic Federated Dataset}: The dataset \texttt{Synthetic}$(1,1)$ is originally constructed in \cite{li2018federated}. The local dataset sizes $n_i, i\in[1,30]$ follows a power law. (2) \emph{DNN trained on a Non-IID partitioned CIFAR-10 dataset}: We train a VGG-11 \cite{simonyan2014very} network on the CIFAR-10 dataset \cite{krizhevsky2009learning}, which is partitioned across $16$ clients using a Dirichlet distribution $\text{Dir}_{16}(0.1)$, as done in \cite{wang2020federated}. The original CIFAR-10 test set (without partitioning) is used to evaluate the generalization performance of the trained global model. The local learning rate $\lr$ is decayed by a constant factor after finishing $50\%$ and $75\%$ of the communication rounds. The initial value of $\lr$ is tuned separately for \fedavg with different local solvers. When using the same solver, \fednova uses the same $\lr$ as \fedavg to guarantee a fair comparison. On CIFAR-10, we run each experiment with $3$ random seeds and report the average and standard deviation. More details are provided in \Cref{sec:more_exps}.
\begin{figure}[!ht]
    \centering
    \begin{subfigure}{.33\textwidth}
    \centering
    \includegraphics[width=\textwidth]{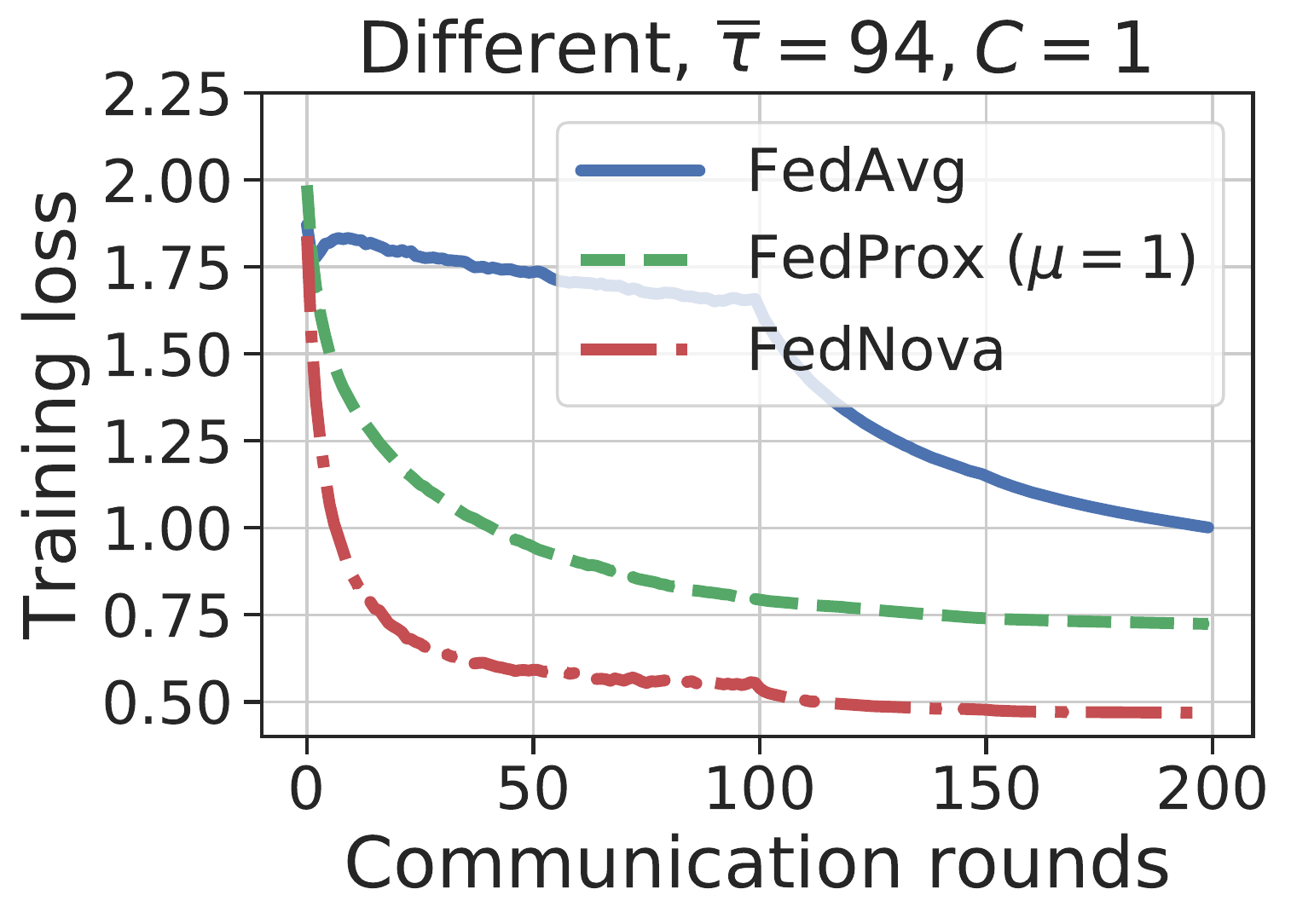}
    \end{subfigure}%
    ~
    \begin{subfigure}{.33\textwidth}
    \centering
    \includegraphics[width=\textwidth]{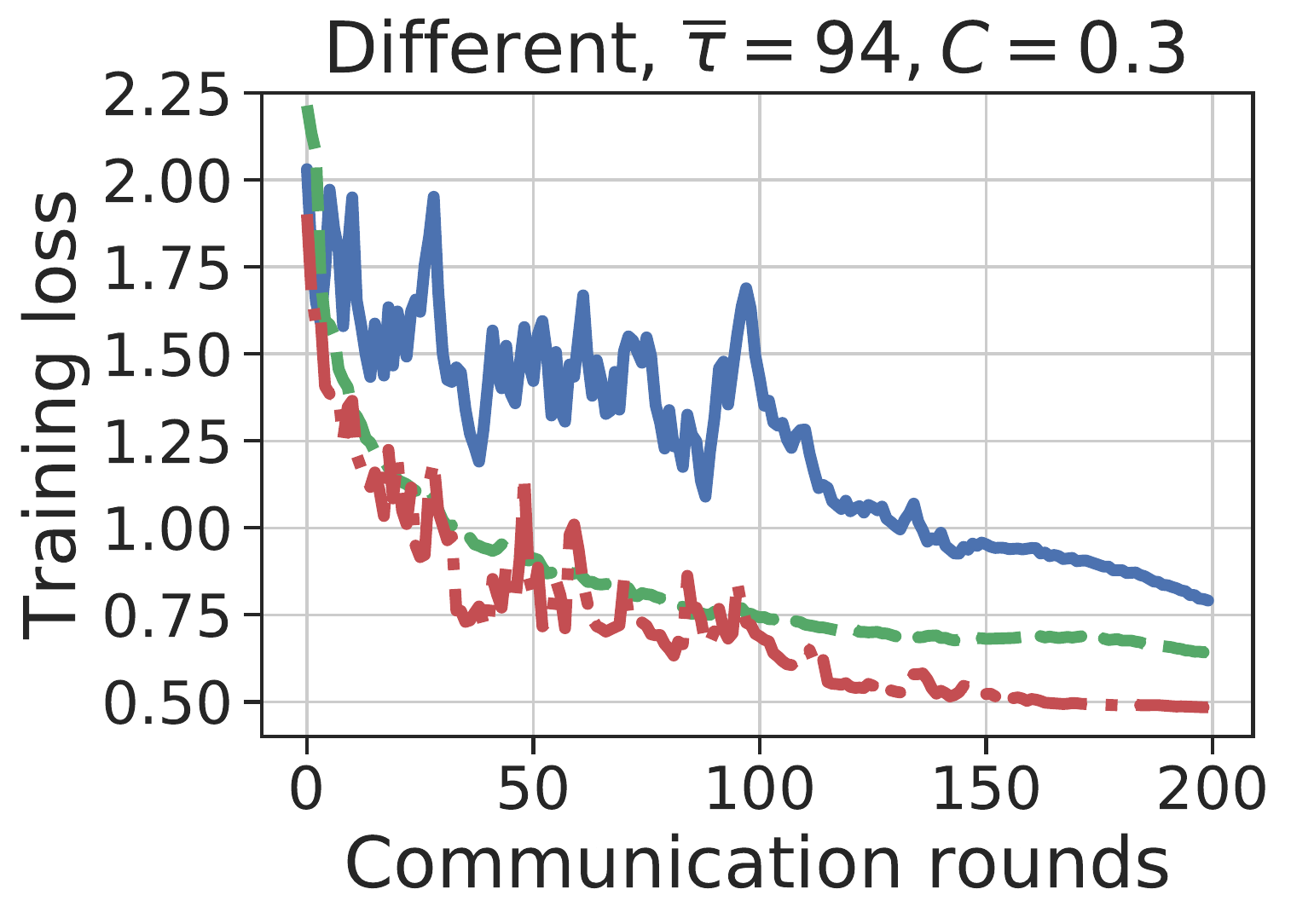}
    \end{subfigure}%
    ~
    \begin{subfigure}{.33\textwidth}
    \centering
    \includegraphics[width=\textwidth]{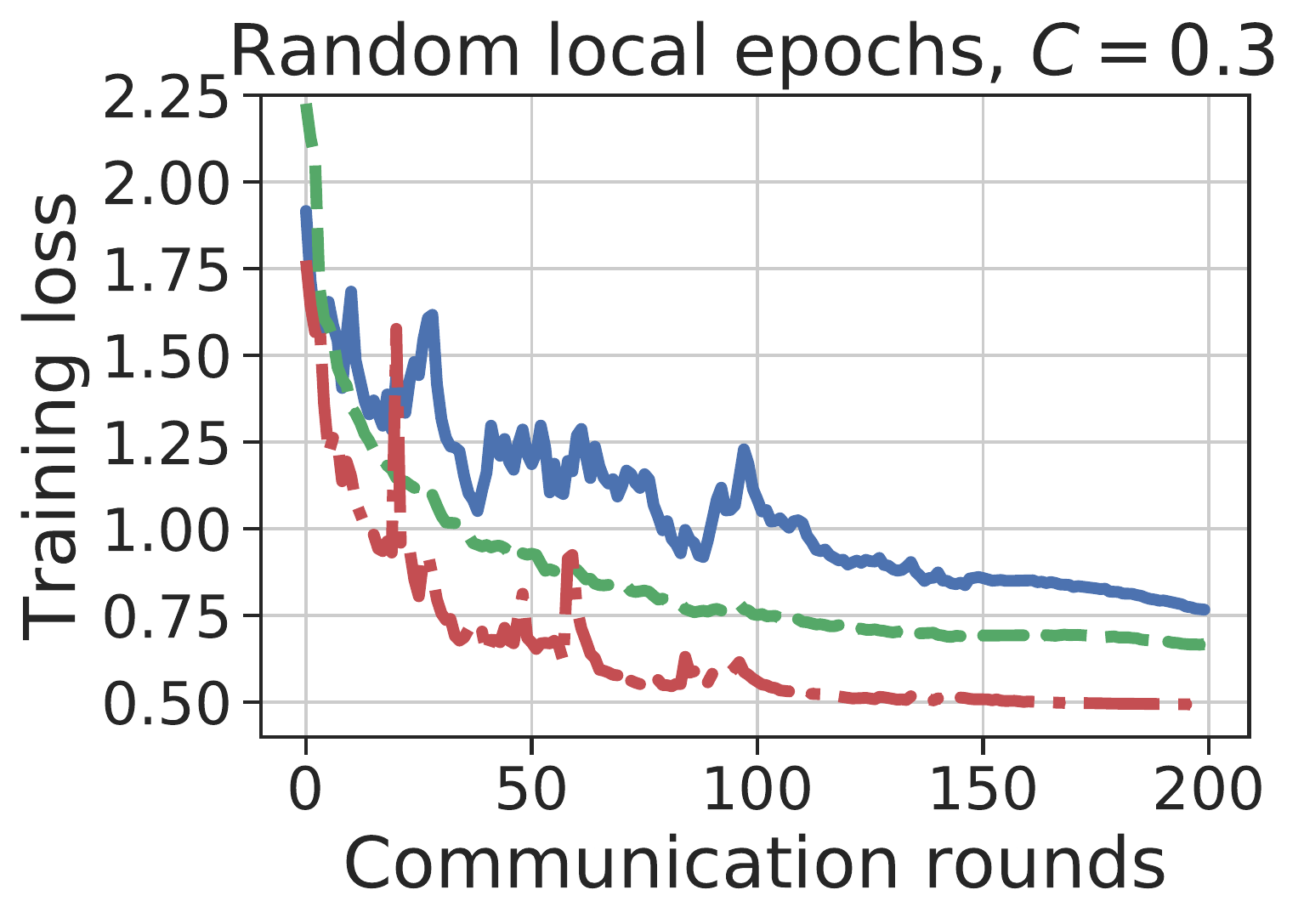}
    \end{subfigure}
    \caption{Results on the synthetic dataset constructed in \cite{li2018federated} under three different settings. \textbf{\emph{Left}}: All clients perform $E_i=5$ local epochs; \textbf{\emph{Middle}}: Only $C=0.3$ fraction of clients are randomly selected per round to perform $E_i=5$ local epochs; \textbf{\emph{Right}}: Only $C=0.3$ fraction of clients are randomly selected per round to perform random and time-varying local epochs $E_i(t) \sim \mathcal{U}(1,5)$.}
    \label{fig:synthetic}
    \vspace{-1em}
\end{figure}

\textbf{Synthetic Dataset Simulations.} In \Cref{fig:synthetic}, we observe that by simply changing $w_i$ to $p_i$, \fednova not only converges significantly faster than \fedavg but also achieves consistently the best performance under three different settings. Note that the only difference between \fednova and \fedavg is the aggregated weights when averaging the normalized gradients.

\begin{table}
    \caption{Results comparing \fedavg and \fednova with various client optimizers (\ie, local solvers) trained on non-IID CIFAR-10 dataset. \fedprox and \scaffold correspond to \fedavg with proximal SGD updates and cross-client variance-reduction (VR), respectively.}
    \label{tab:cifar}
    
    \centering
    \begin{tabular}{*4c}\toprule
    \multirow{2}{*}{Local Epochs} & \multirow{2}{*}{Client Opt.} & \multicolumn{2}{c}{Test Accuracy $\%$} \\\cmidrule(lr){3-4}
        & & \fedavg  & \fednova \\ \midrule
    \multirow{3}{*}{\makecell{$E_i=2$ \\ {\tiny$(16\leq\cp_i \leq 408)$}}} & Vanilla & $60.68${\color{gray}$\pm 1.05$} & $\textbf{66.31}${\color{gray}$\pm 0.86$}\\
    & Momentum &$65.26${\color{gray}$\pm 2.42$} & $\textbf{73.32}${\color{gray}$\pm 0.29$} \\
    & Proximal \cite{li2018federated} & $60.44${\color{gray}$\pm 1.21$} & $\textbf{69.92}${\color{gray}$\pm 0.34$} \\\midrule
    \multirow{5}{*}{\makecell{$E_i^{(t)}\sim \mathcal{U}(2,5)$ \\ {\tiny$(16\leq\cp_i^{(t)}\leq 1020)$}}} & Vanilla & $64.22${\color{gray}$\pm 1.06$}  & $\textbf{73.22}${\color{gray}$\pm 0.32$} \\
    & Momentum & $70.44${\color{gray}$\pm 2.99$}  & $\textbf{77.07}${\color{gray}$\pm 0.12$} \\
    & Proximal \cite{li2018federated} & $63.74${\color{gray}$\pm 1.44$}  & $\textbf{73.41}${\color{gray}$\pm 0.45$} \\
    & VR \cite{karimireddy2019scaffold} & $74.72${\color{gray}$\pm 0.34$}  & $\textbf{74.72}${\color{gray}$\pm 0.19$} \\
    & Momen.+VR & Not Defined  & $\textbf{79.19}${\color{gray}$\pm 0.17$} \\
    \bottomrule
    \end{tabular}
\end{table}

\textbf{Non-IID CIFAR-10 Experiments.} In \Cref{tab:cifar} we compare the performance of \fednova and \fedavg on non-IID CIFAR-10 with various client optimizers run for $100$ communication rounds. When the client optimizer is SGD or SGD with momentum, simply changing the weights yields a $6$-$9\%$ improvement on the test accuracy; When the client optimizer is proximal SGD, \fedavg is equivalent to \fedprox. By setting $\cpeff=\sum_{i=1}^\nworkers p_i \cp_i$ and correcting the weights $w_i=p_i$ while keeping $\gradweight_i$ same as \fedprox, \texttt{FedNova-Prox} achieves about $10\%$ higher test accuracy than \fedprox. In \Cref{fig:cifar_test_curves}, we further compare the training curves. It turns out that \fednova consistently converges faster than \fedavg. When using variance-reduction methods such as \scaffold (that requires doubled communication), \fednova-based method preserves the same test accuracy. Furthermore, combining local momentum and variance-reduction can be easily achieved in \fednova. It yields the highest test accuracy among all other local solvers. This kind of combination is non-trivial and has not appeared yet in the literature. We provide its pseudocode in \Cref{sec:pseudocode}.

\begin{figure}[t]
    \centering
    \begin{subfigure}{.33\textwidth}
    \centering
    \includegraphics[width=\textwidth]{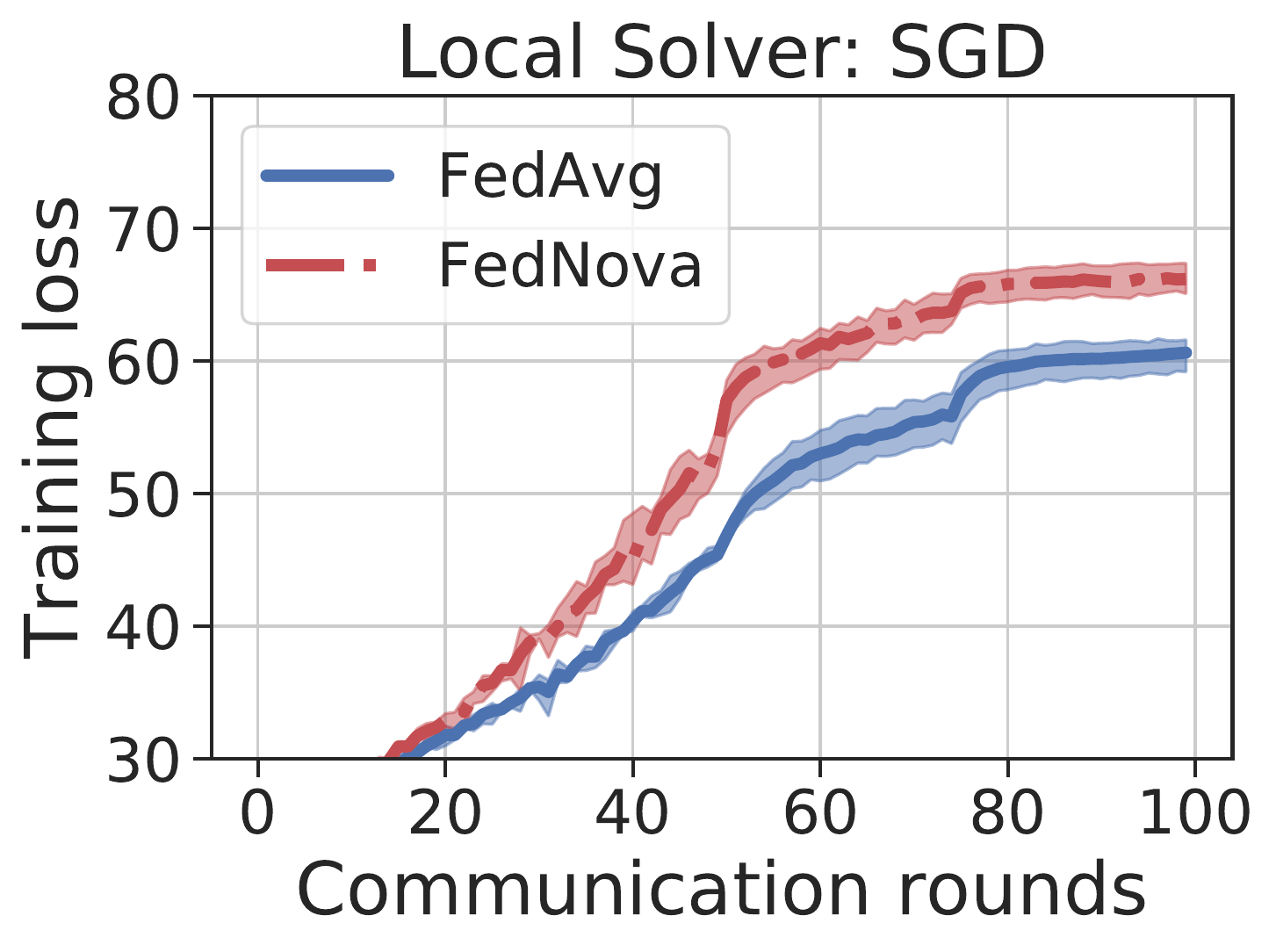}
    \end{subfigure}%
    ~
    \begin{subfigure}{.33\textwidth}
    \centering
    \includegraphics[width=\textwidth]{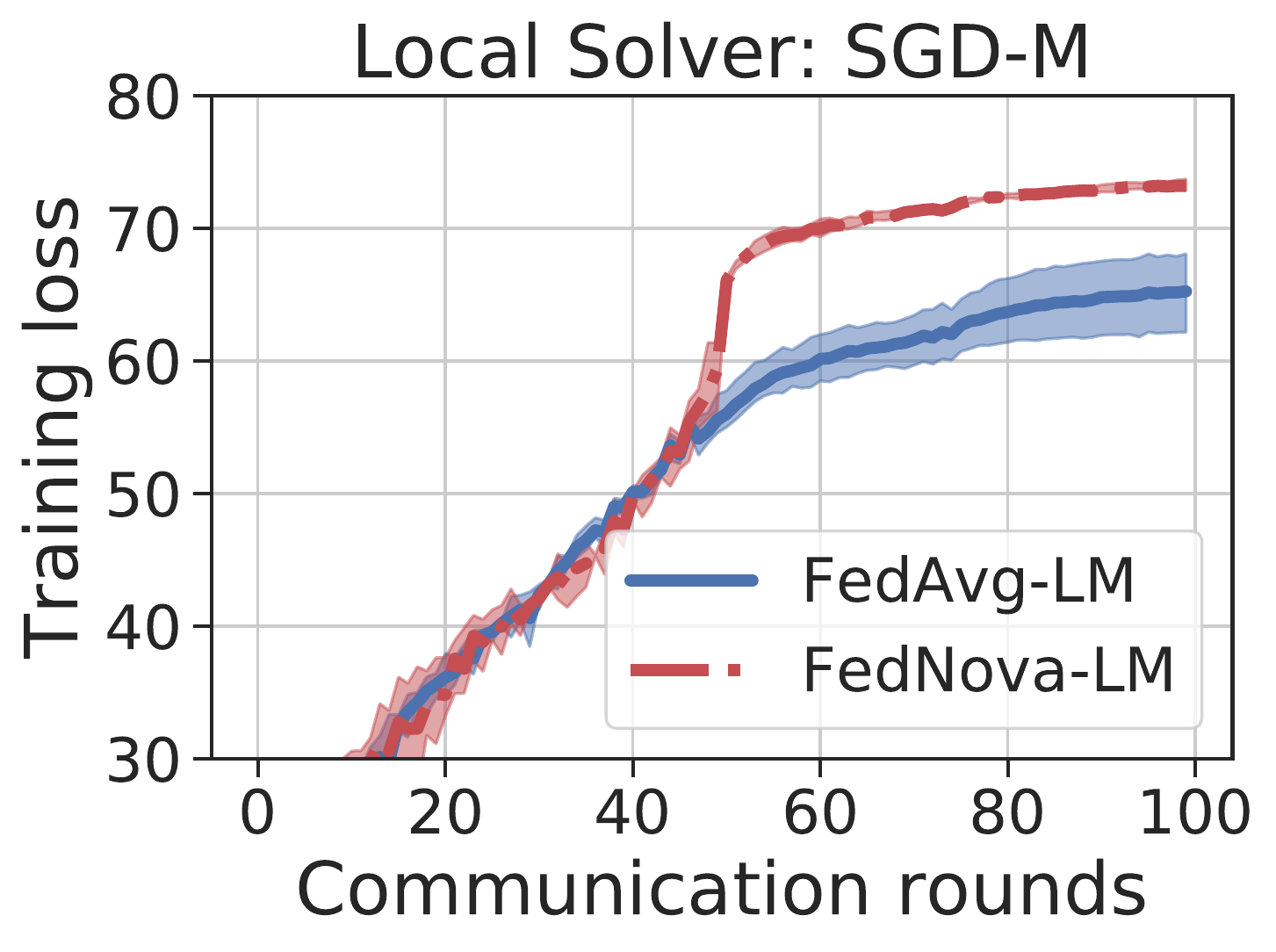}
    \end{subfigure}%
    ~
    \begin{subfigure}{.33\textwidth}
    \centering
    \includegraphics[width=\textwidth]{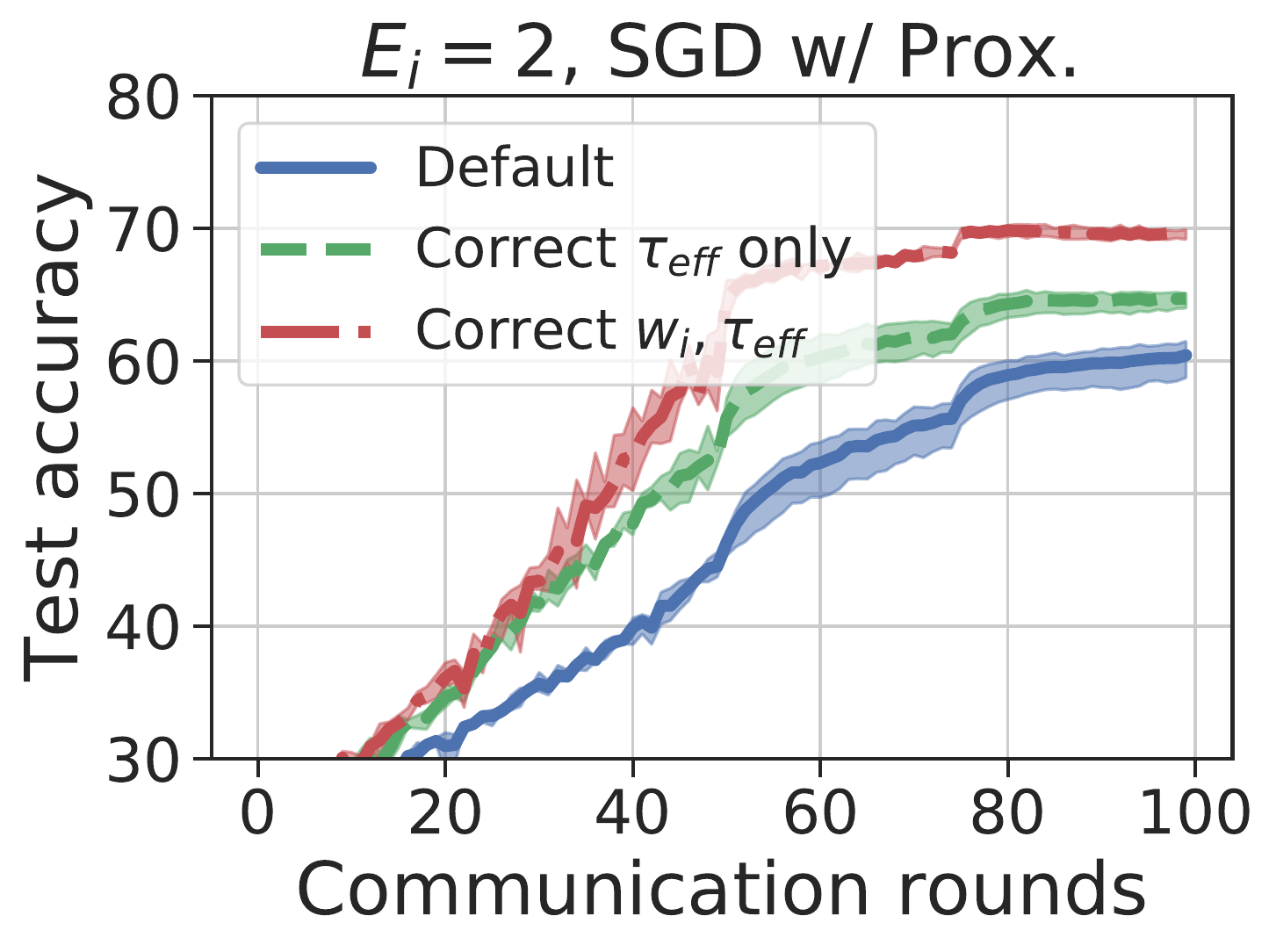}
    \end{subfigure}
    \caption{Training curves on non-IID partitioned CIFAR10 dataset. All clients perform $2$ local epochs of training and the number of local steps varies from $16$ to $408$. \textbf{\emph{Left}}: Client optimizer is vanilla SGD; \textbf{\emph{Middle}}: Client optimizer is SGD with momentum. `LM' represents for local momentum (\ie, using momentum locally); \textbf{\emph{Right}}: Client Optimizer is SGD with proximal updates. `Default' (blue curve) corresponds to \fedprox algorithm. In the green curve, we set $\cpeff$ to be $\sum_{i=1}^\nworkers p_i \cp_i$ instead of its default value $\sum_{i=1}^\nworkers p_i [1-(1-\alpha)^{\cp_i}]/\alpha$. In the red curve, we use \fednova with proximal updates and it gives both higher accuracy and faster convergence than the original \fedprox.}
    \label{fig:cifar_test_curves}
\end{figure}

\textbf{Effectiveness of Local Momentum.} From \Cref{tab:cifar}, it is worth noting that using momentum SGD as the local solver is an effective way to improve the performance. It generally achieves $3$-$7\%$ higher test accuracy than vanilla SGD. This local momentum scheme can be further combined with server momentum \cite{Wang2020SlowMo,hsu2019measuring,reddi2020adaptive}. When $E_i(t)\sim\mathcal{U}(2,5)$, the hybrid momentum scheme achieves test accuracy $81.15\pm 0.38\%$  As a reference, using server momentum alone achieves $77.49\pm 0.25\%$. 

\section{Concluding Remarks}
In federated learning, the participated clients (\eg, IoT sensors, mobile devices) are typically highly heterogeneous, both in the size of their local datasets as well as their computation speeds. Clients can also join and leave the training at any time according to their availabilities. Therefore, it is common that clients perform different amount of works within one round of local computation. However, previous analyses on federated optimization algorithms are limited to the homogeneous case where all clients have the same local steps, hyper-parameters, and client optimizers. In this paper, we develop a novel theoretical framework to analyze the challenging heterogeneous setting. We show that original \fedavg algorithm will converge to stationary points of a mismatched objective function which can be arbitrarily different from the true objective. To the best of our knowledge, we provide the first fundamental understanding of how the convergence rate and bias in the final solution of federated optimization algorithms are influenced by the heterogeneity in clients' local progress. The new framework naturally allows clients to have different local steps and local solvers, such as GD, SGD, SGD with momentum, proximal updates, etc. Inspired by the theoretical analysis, we propose \fednova, which can automatically adjust the aggregated weight and effective local steps according to the local progress. We validate the effectiveness of \fednova both theoretically and empirically. On a non-IID version of CIFAR-10 dataset, \fednova generally achieves $6$-$9\%$ higher test accuracy than \fedavg. Future directions include extending the theoretical framework to adaptive optimization methods or gossip-based training methods.

\section{Acknowledgements}
This research was generously supported in part by NSF grants CCF-1850029, the 2018 IBM Faculty Research Award, and the Qualcomm Innovation fellowship (Jianyu Wang). We thank Anit Kumar Sahu, Tian Li, Zachary Charles, Zachary Garrett, and Virginia Smith for helpful discussions.


\bibliographystyle{unsrt}
{\bibliography{sample.bib}}

\newpage
\appendix
\section{Proof of \Cref{lem:quadratic}: Objective Inconsistency in Quadratic Model}\label{sec:proof_quadratic}
\paragraph{Formulation.}
Consider a simple setting where each local objective function is strongly convex and defined as follows:
\begin{align}
F_i(\x) = \frac{1}{2}\x\tp\matH_i\x - \bm{e}_i\tp \x + \frac{1}{2}\bm{e}_i\tp\matH_i^{-1}\bm{e}_i \geq 0
\end{align}
where $\matH_i \in \mathbb{R}^{d \times d}$ is an invertible matrix and $\bm{e}_i\in \mathbb{R}^d$ is an arbitrary vector. It is easy to show that the optimum of the $i$-th local function is $\x_i^* = \matH_i^{-1}\bm{e}_i$. Without loss of generality, we assume the global objective function to be a weighted average across all local functions, that is:
\begin{align}
\obj(\x) = \sum_{i=1}^\nworkers p_i\obj_i(\x)
= \frac{1}{2}\x\tp\overline{\matH}\x - \overline{\bm{e}}\tp \x + \frac{1}{2}\sum_{i=1}^\nworkers p_i\bm{e}_i\tp\ntg_i^{-1}\bm{e}_i
\end{align}
where $\overline{\matH} = \sum_{i=1}^\nworkers p_i\matH_i$ and $\overline{\bm{e}} = \sum_{i=1}^\nworkers p_i\bm{e}_i$. As a result, the global minimum is $\x^* = \overline{\matH}^{-1}\overline{\bm{e}}$. Now, let us study whether previous federated optimization algorithms can converge to this global minimum.

\paragraph{Local Update Rule.}
The local update rule of FedProx for the $i$-th device can be written as follows:
\begin{align}
\x_i^{(t,k+1)} 
&= \x_i^{(t,k)} - \lr \brackets{\matH_i\x_i^{(t,k)} - \bm{e}_i + \mu (\x_i^{(t,k)} - \x^{(t,0)})} \\
&= (\matI - \lr\mu\matI - \lr\matH_i)\x_i^{(t,k)} + \lr\bm{e}_i + \lr\mu\x^{(t,0)} \label{eqn:local_updt1}
\end{align}
where $\x_i^{(t,k)}$ denotes the local model parameters at the $k$-th local iteration after $t$ communication rounds, $\lr$ denotes the local learning rate and $\mu$ is a tunable hyper-parameter in FedProx. When $\mu=0$, the algorithm will reduce to FedAvg. We omit the device index in $\x^{(t,0)}$, since it is synchronized and the same across all devices.

After minor arranging \eqref{eqn:local_updt1}, we obtain
\begin{align}
\x^{(t,k+1)}_{i} - \bm{c}_i^{(t)}
&= \parenth{\matI-\lr\mu\matI - \lr\matH_i}\parenth{\x^{(t,k)}_{i} - \bm{c}_i^{(t)}}.
\end{align}
where $\bm{c}_i^{(t)} = \parenth{\matH_i + \mu\matI}^{-1}\parenth{\bm{e}_i+\mu\x^{(t,0)}}$. Then, after performing $\tau_i$ steps of local updates, the local model becomes
\begin{align}
\x^{(t,\tau_i)}_{i}
&= \parenth{\matI-\lr\mu\matI - \lr\matH_i}^{\tau_i}\parenth{\x^{(t,0)} - \bm{c}_i^{(t)}} + \bm{c}_i^{(t)}, \\
\x^{(t,\tau_i)}_{i} - \x^{(t,0)}
&= \parenth{\matI-\lr\mu\matI - \lr\matH_i}^{\tau_i}\parenth{\x^{(t,0)} - \bm{c}_i^{(t)}} + \bm{c}_i^{(t)} - \x^{(t,0)} \\
&= \brackets{\parenth{\matI-\lr\mu\matI - \lr\matH_i}^{\tau_i} - \matI}\parenth{\x^{(t,0)} - \bm{c}_i^{(t)}} \\
&= \brackets{\matI - \parenth{\matI-\lr\mu\matI - \lr\matH_i}^{\tau_i}}\parenth{\matH_i + \mu\matI}^{-1}\parenth{\bm{e}_i - \matH_i \x^{(t,0)}}.
\end{align}
For the ease of writing, we define $\matK_i(\lr,\mu) = \brackets{\matI - \parenth{\matI - \lr\mu\matI - \lr\matH_i}^{\tau_i}}\parenth{\matH_i+\mu\matI}^{-1}$.

\paragraph{Server Aggregation.}
For simplicity, we only consider the case when all devices participate in the each round. In FedProx, the server averages all local models according to the sample size:
\begin{align}
\x^{(t+1,0)} - \x^{(t,0)}
&= \sum_{i=1}^\nworkers p_i \parenth{\x^{(t,\tau_i)}_i - \x^{(t,0)}} \\
&= \sum_{i=1}^\nworkers p_i \matK_i(\lr,\mu) \parenth{\bm{e}_i - \matH_i \x^{(t,0)}}.
\end{align}
Accordingly, we get the following update rule for the central model:
\begin{align}
\x^{(t+1,0)}
&= \brackets{\matI - \sum_{i=1}^\nworkers p_i\matK_i(\lr,\mu) \matH_i}\x^{(t,0)} + \sum_{i=1}^\nworkers p_i\matK_i(\lr,\mu)\bm{e}_i.
\end{align}
It is equivalent to
\begin{align}
\x^{(t+1,0)} - \widetilde{\x}
&= \brackets{\matI - \sum_{i=1}^\nworkers p_i \matK_i(\lr,\mu) \matH_i}\brackets{\x^{(t,0)} - \widetilde{\x}}.
\end{align}
where 
\begin{align}
\widetilde{\x} 
&= \parenth{\sum_{i=1}^\nworkers p_i\matK_i(\lr,\mu) \matH_i}^{-1}\parenth{\sum_{i=1}^\nworkers p_i \matK_i(\lr,\mu)\bm{e}_i}.
\end{align}
After $T$ communication rounds, one can get
\begin{align}
\x^{(T,0)} 
&= \brackets{\matI - \sum_{i=1}^\nworkers p_i \matK_i(\lr,\mu) \matH_i}^T \brackets{\x^{(t,0)} - \widetilde{\x}} + \widetilde{\x}.
\end{align}
Accordingly, when $\opnorm{\matI - \sum_{i=1}^\nworkers p_i\matK_i(\lr,\mu) \matH_i}<1$, the iterates will converge to
\begin{align}
\lim_{T\rightarrow \infty} \x^{(T,0)}
= \widetilde{\x} 
= \parenth{\sum_{i=1}^\nworkers p_i\matK_i(\lr,\mu) \matH_i}^{-1}\parenth{\sum_{i=1}^\nworkers p_i \matK_i(\lr,\mu)\bm{e}_i}. \label{eqn:T_limit}
\end{align}
Recall that $\matK_i(\lr,\mu) = \brackets{\matI - \parenth{\matI - \lr\mu\matI - \lr\matH_i}^{\tau_i}}\parenth{\matH_i+\mu\matI}^{-1}$.

\paragraph{Concrete Example in \Cref{lem:quadratic}.}
Now let us focus on a concrete example where $p_1 = p_2 = \cdots = p_m = 1/m, \matH_1 = \matH_2 = \cdots = \matH_m = \matI$ and $\mu = 0$. Then, in this case, $\matK_i = 1 - (1-\lr)^{\tau_i}$. As a result, we have
\begin{align}
\lim_{T\rightarrow \infty} \x^{(T,0)} = \frac{\sum_{i=1}^m \brackets{1-(1-\lr)^{\tau_i}}\bm{e}_i}{\sum_{i=1}^m \brackets{1-(1-\lr)^{\tau_i}}}.
\end{align}
Furthermore, when the learning rate is sufficiently small (\eg, can be achieved by gradually decaying the learning rate), according to L'Hospital's rule, we obtain
\begin{align}
\lim_{\lr\rightarrow 0}\lim_{T\rightarrow \infty} \x^{(T,0)} = \frac{\sum_{i=1}^m \tau_i\bm{e}_i}{\sum_{i=1}^m \tau_i}.
\end{align}
Here, we complete the proof of \Cref{lem:quadratic}.

\section{Detailed Derivations for Various Local Solvers}\label{sec:proof_detailed_gradweight}
In this section, we will derive the specific expression of the vector $\gradweight_i$ when using different local solvers. Recall that the local change at client $i$ is $\Delta_i^{(t)} = -\lr\matG_i^{(t)}\gradweight_i$ where $\matG_i^{(t)}$ stacks all stochastic gradients in the current round and $\gradweight$ is a non-negative vector.

\subsection{SGD with Proximal Updates}
In this case, we can write the update rule of local models as follows:
\begin{align}
\x_i^{(t,\cp_i)} = \x_i^{(t,\cp_i-1)} - \lr \brackets{\sg_i(\x_i^{(t,\cp_i-1)}) + \mu\parenth{\x_i^{(t,\cp_i-1)} - \x^{(t,0)}}}.
\end{align}
Subtracting $\x_i^{(t,0)}$ on both sides, we obtain
\begin{align}
\x_i^{(t,\cp_i)} - \x^{(t,0)}
=& \x_i^{(t,\cp_i-1)} - \x^{(t,0)} - \lr \brackets{\sg_i(\x_i^{(t,\cp_i-1)}) + \mu\parenth{\x_i^{(t,\cp_i-1)} - \x^{(t,0)}}} \\
=& (1-\lr\mu)\parenth{\x_i^{(t,\cp_i-1)} - \x^{(t,0)}} - \lr\sg_i(\x_i^{(t,\cp_i-1)}).
\end{align}
Repeating the above procedure, it follows that
\begin{align}
\Delta_i^{(t)} = \x_i^{(t,\cp_i)} - \x^{(t,0)} = -\lr\sum_{k=0}^{\cp_i-1} (1-\lr\mu)^{\cp_i-1-k}\sg_i(\x_i^{(t,k)}).
\end{align}
According to the definition, we have $\gradweight_i = [(1-\alpha)^{\cp_i-1}, (1-\alpha)^{\cp_i-2},\dots,(1-\alpha),1]$ where $\alpha = \lr\mu$.

\subsection{SGD with Local Momentum}
Let us firstly write down the update rule of the local models. Suppose that $\rho$ denotes the local momentum factor and $\bm{u}_i$ is the local momentum buffer at client $i$. Then, the update rule of local momentum SGD is:
\begin{align}
\bm{u}_i^{(t,\cp_i)} 
=& \rho \bm{u}_i^{(t,\cp_i-1)} + \sg_i(\x_i^{(t,\cp_i-1)}), \\
\x_i^{(t,\cp_i)} 
=& \x_i^{(t,\cp_i-1)} - \lr\bm{u}_i^{(t,\cp_i)}. \label{eqn:momen_update}
\end{align}
One can expand the expression of local momentum buffer as follows:
\begin{align}
\bm{u}_i^{(t,\cp_i)} 
=& \rho \bm{u}_i^{(t,\cp_i-1)} + \sg_i(\x_i^{(t,\cp_i-1)}) \\
=& \rho^2 \bm{u}_i^{(t,\cp_i-2)} + \rho \sg_i(\x_i^{(t,\cp_i-2)}) + \sg_i(\x_i^{(t,\cp_i-1)}) \\
=& \sum_{k=0}^{\cp_i-1} \rho^{\cp_i-1-k} \sg_i(\x_i^{(t,k)}) \label{eqn:momen_buffer}
\end{align}
where the last equation comes from the fact $\bm{u}_i^{(t,0)}=0$. Substituting \Cref{eqn:momen_buffer} into \Cref{eqn:momen_update}, we have
\begin{align}
\x_i^{(t,\cp_i)} 
=& \x_i^{(t,\cp_i-1)} - \lr\sum_{k=0}^{\cp_i-1} \rho^{\cp_i-1-k} \sg_i(\x_i^{(t,k)}) \\
=& \x_i^{(t,\cp_i-2)} - \lr\sum_{k=0}^{\cp_i-2} \rho^{\cp_i-2-k} \sg_i(\x_i^{(t,k)}) - \lr\sum_{k=0}^{\cp_i-1} \rho^{\cp_i-1-k} \sg_i(\x_i^{(t,k)}).
\end{align}
Repeating the above procedure, it follows that
\begin{align}
\x_i^{(t,\cp_i)} - \x^{(t,0)}
= -\lr \sum_{s=0}^{\cp_i-1} \sum_{k=0}^{s} \rho^{s-k}\sg_i(\x_i^{(t,k)})
\end{align}
Then, the coefficient of $\sg_i(\x_i^{(t,k)})$ is
\begin{align}
\sum_{s \geq k}^{\cp_i-1} \rho^{s-k} = 1 + \rho + \rho^2 + \cdots + \rho^{\cp_i-1-k} = \frac{1-\rho^{\cp_i-k}}{1-\rho}.
\end{align}
That is, $\gradweight_i = [1-\rho^{\cp_i}, 1-\rho^{\cp_i-1}, \dots, 1-\rho]/(1-\rho)$. In this case, the $\ell_1$ norm of $\gradweight_i$ is
\begin{align}
\vecnorm{\gradweight_i}_1 = \frac{1}{1-\rho}\sum_{k=0}^{\cp_i-1}\parenth{1 - \rho^{\cp_i-k}} 
=& \frac{1}{1-\rho} \parenth{\cp_i - \sum_{k=0}^{\cp_i-1}\rho^{\cp_i-k}} \\
=& \frac{1}{1-\rho} \brackets{\cp_i - \frac{\rho(1-\rho^{\cp_i})}{1-\rho}}.
\end{align}

\newpage
\section{Proof of \Cref{thm:general}: Convergence of Surrogate Objective}\label{sec:proof_thm1}

\subsection{Preliminaries}
For the ease of writing, let us define a surrogate objective function $\surloss(\x) = \sum_{i=1}^\nworkers w_i F_i(\x)$, where $\sum_{i=1}^\nworkers w_i = 1$, and define the following auxiliary variables
\begin{align}
\text{Normalized Stochastic Gradient:} \quad \nsg_i^{(t)} &= \frac{1}{a_i}\sum_{k=0}^{\tau_i-1} a_{i,k} \sg_i(\x_i^{(t,k)}), \\
\text{Normalized Gradient:} \quad \ntg_i^{(t)} &= \frac{1}{a_i}\sum_{k=0}^{\tau_i-1} a_{i,k} \tg_i(\x_i^{(t,k)})
\end{align}
where $a_{i,k} \geq 0$ is an arbitrary scalar, $\bm{a}_i = [a_{i,0},\dots,a_{i,\cp_i-1}]\tp$, and $a_i = \vecnorm{\bm{a}_i}_1$. Besides, one can show that $\Exs[\nsg_i^{(t)} - \ntg_i^{(t)}]=0$. In addition, since workers are independent to each other, we have $\Exs\inprod{\nsg_i^{(t)}-\ntg_i^{(t)}}{\nsg_j^{(t)} - \ntg_j^{(t)}} = 0, \forall i \neq j$. Recall that the update rule of the global model can be written as follows:
\begin{align}
\x^{(t+1,0)} - \x^{(t,0)}
&= -\cpeff \lr\sum_{i=1}^\nworkers w_i \nsg_i^{(t)}.
\end{align}
According to the Lipschitz-smooth assumption, it follows that
\begin{align}
&\Exs\brackets{\surloss(\x^{(t+1,0)})} - \surloss(\x^{(t,0)}) \nonumber \\ 
\leq& -\cpeff \lr \underbrace{\Exs\brackets{\inprod{\nabla \surloss(\x^{(t,0)})}{\sum_{i=1}^{\nworkers}w_i\nsg_i^{(t)}}}}_{T_1} + \frac{\cpeff^2\lr^2 \lip}{2}\underbrace{\Exs\brackets{\vecnorm{\sum_{i=1}^{\nworkers}w_i\nsg_i^{(t)}}^2}}_{T_2} \label{eqn:basis}
\end{align}
where the expectation is taken over mini-batches $\xi_i^{(t,k)}, \forall i\in \{1,2,\dots,\nworkers\}, k \in \{0, 1,\dots,\cp_i-1\}$. Before diving into the detailed bounds for $T_1$ and $T_2$, we would like to firstly introduce several useful lemmas.
\begin{lem}
	Suppose $\{A_k\}_{k=1}^T$ is a sequence of random matrices and $\Exs[A_k|A_{k-1},A_{k-2},\dots,A_1] = \bm{0},\forall k$. Then,
	\begin{align}
	\Exs\brackets{\fronorm{\sum_{k=1}^T A_k}^2} = \sum_{k=1}^T \Exs\brackets{\fronorm{A_k}^2}.
	\end{align}
	\label{lem:sum_norm}
\end{lem}
\begin{proof}
	\begin{align}
	\Exs\brackets{\fronorm{\sum_{k=1}^T A_k}^2}
	=& \sum_{k=1}^T \Exs\brackets{\fronorm{A_k}^2} + \sum_{i=1}^T\sum_{j=1,j\neq i}^T\Exs\brackets{\trace\{A_i\tp A_j\}} \\
	=& \sum_{k=1}^T \Exs\brackets{\fronorm{A_k}^2} + \sum_{i=1}^T\sum_{j=1,j\neq i}^T\trace\{\Exs\brackets{A_i\tp A_j}\}
	\end{align}
	Assume $i < j$. Then, using the law of total expectation,
	\begin{align}
	\Exs\brackets{A_i\tp A_j}
	= \Exs\brackets{ A_i \tp\Exs\brackets{A_j | A_{i},\dots,A_1}}
	= \bm{0}.
	\end{align}
\end{proof}

\subsection{Bounding First term in \Cref{eqn:basis}}
For the first term on the right hand side (RHS) in \Cref{eqn:basis}, we have
\begin{align}
T_1
=& \Exs\brackets{\inprod{\nabla \surloss(\x^{(t,0)})}{\sum_{i=1}^{\nworkers}w_i\parenth{\nsg_i^{(t)}-\ntg_i^{(t)}}}} + \Exs\brackets{\inprod{\nabla \surloss(\x^{(t,0)})}{\sum_{i=1}^\nworkers w_i\ntg_i^{(t)}}} \\
=& \Exs\brackets{\inprod{\nabla \surloss(\x^{(t,0)})}{\sum_{i=1}^\nworkers w_i\ntg_i^{(t)}}} \\
=& \frac{1}{2}\vecnorm{\nabla \surloss(\x^{(t)})}^2 + \frac{1}{2}\Exs\brackets{\vecnorm{\sum_{i=1}^\nworkers w_i\ntg_i^{(t)}}^2} - \frac{1}{2}\Exs\brackets{\vecnorm{\nabla \surloss(\x^{(t,0)}) - \sum_{i=1}^\nworkers w_i\ntg_i^{(t)}}^2} \label{eqn:T1}
\end{align}
where the last equation uses the fact: $2\inprod{a}{b} = \vecnorm{a}^2 + \vecnorm{b}^2 - \vecnorm{a-b}^2$.

\subsection{Bounding Second term in \Cref{eqn:basis}}
For the second term on the right hand side (RHS) in \Cref{eqn:basis}, we have
\begin{align}
T_2 
=& \Exs\brackets{\vecnorm{\sum_{i=1}^{\nworkers}w_i\parenth{\nsg_i^{(t)}-\ntg_i^{(t)}} + \sum_{i=1}^\nworkers w_i \ntg_i^{(t)}}^2} \\
\leq& 2\Exs\brackets{\vecnorm{\sum_{i=1}^{\nworkers}w_i\parenth{\nsg_i^{(t)}-\ntg_i^{(t)}}}^2} +2\Exs\brackets{\vecnorm{ \sum_{i=1}^\nworkers w_i\ntg_i^{(t)}}^2} \label{eqn:T2_step1}\\
=& 2\sum_{i=1}^{\nworkers}w_i^2\Exs\brackets{\vecnorm{\nsg_i^{(t)}-\ntg_i^{(t)}}^2}+2\Exs\brackets{\vecnorm{ \sum_{i=1}^\nworkers w_i\ntg_i^{(t)}}^2} \label{eqn:T2_step2}
\end{align}
where \Cref{eqn:T2_step1} follows the fact: $\vecnorm{a+b}^2 \leq 2\vecnorm{a}^2 + 2\vecnorm{b}^2$ and \Cref{eqn:T2_step2} uses the special property of $\nsg_i^{(t)},\ntg_i^{(t)}$, that is, $\Exs\inprod{\nsg_i^{(t)}-\ntg_i^{(t)}}{\nsg_j^{(t)} - \ntg_j^{(t)}} = 0, \forall i \neq j$. Then, let us expand the expression of $\nsg_i^{(t)}$ and $\ntg_i^{(t)}$, we obtain that
\begin{align}
T_2
\leq& \sum_{i=1}^{\nworkers}\frac{2w_i^2}{a_i^2}\sum_{k=0}^{\tau_i-1} [a_{i,k}]^2\Exs\brackets{\vecnorm{\sg_i(\x_i^{(t,k)})-\tg_i(\x_i^{(t,k)})}^2}+2\Exs\brackets{\vecnorm{ \sum_{i=1}^\nworkers w_i\ntg_i^{(t)}}^2} \label{eqn:T2_step3}\\
\leq& 2\vbnd\sum_{i=1}^\nworkers \frac{w_i^2\vecnorm{\gradweight_i}^2}{\vecnorm{\gradweight_i}_1^2} +2\Exs\brackets{\vecnorm{ \sum_{i=1}^\nworkers w_i\ntg_i^{(t)}}^2} \label{eqn:T2_step4}
\end{align}
where \Cref{eqn:T2_step3} is derived using \Cref{lem:sum_norm}, \Cref{eqn:T2_step4} follows \Cref{assump:var}.

\subsection{Intermediate Result}
Plugging \Cref{eqn:T1,eqn:T2_step4} back into \Cref{eqn:basis}, we have
\begin{align}
\Exs\brackets{\surloss(\x^{(t+1,0)})} - \surloss(\x^{(t,0)})
\leq& -\frac{\cpeff \lr}{2} \vecnorm{\nabla \surloss(\x^{(t,0)})}^2 - \frac{\cpeff \lr}{2}\parenth{1 - 2\cpeff \lr\lip}\Exs\brackets{\vecnorm{ \sum_{i=1}^\nworkers w_i\ntg_i^{(t)}}^2} \nonumber \\
+\cpeff^2\lr^2\lip\vbnd&\sum_{i=1}^\nworkers \frac{w_i^2 \vecnorm{\bm{a}_i}_2^2}{\vecnorm{\bm{a}_i}_1^2} + \frac{\cpeff\lr}{2}\Exs\brackets{\vecnorm{\nabla \surloss(\x^{(t,0)}) - \sum_{i=1}^\nworkers w_i \ntg_i^{(t)}}^2}
\end{align}
When $\cpeff \lr\lip \leq 1/2$, it follows that
\begin{align}
\frac{\Exs\brackets{\surloss(\x^{(t+1,0)})} - \surloss(\x^{(t,0)})}{\lr\cpeff}
\leq& -\frac{1}{2} \vecnorm{\nabla \surloss(\x^{(t,0)})}^2 + \cpeff\lr\lip\vbnd\sum_{i=1}^\nworkers \frac{w_i^2 \vecnorm{\bm{a}_i}_2^2}{\vecnorm{\bm{a}_i}_1^2} \nonumber \\
& + \frac{1}{2}\Exs\brackets{\vecnorm{\nabla \surloss(\x^{(t,0)}) - \sum_{i=1}^\nworkers w_i \ntg_i^{(t)}}^2} \\
\leq& -\frac{1}{2} \vecnorm{\nabla \surloss(\x^{(t,0)})}^2 + \cpeff\lr\lip\vbnd\sum_{i=1}^\nworkers \frac{w_i^2 \vecnorm{\bm{a}_i}_2^2}{\vecnorm{\bm{a}_i}_1^2} \nonumber \\
& + \frac{1}{2}\sum_{i=1}^\nworkers w_i  \Exs\brackets{\vecnorm{\tg_i(\x^{(t,0)}) - \ntg_i^{(t)}}^2} \label{eqn:mid_result}
\end{align}
where the last inequality uses the fact $\surloss(\x) = \sum_{i=1}^m w_i F_i(\x)$ and Jensen's Inequality: $\vecnorm{\sum_{i=1}^\nworkers w_i z_i}^2 \leq \sum_{i=1}^\nworkers w_i\vecnorm{z_i}^2$. Next, we will focus on bounding the last term in \Cref{eqn:mid_result}.

\subsection{Bounding the Difference Between Server Gradient and Normalized Gradient}
Recall the definition of $\ntg_i^{(t)}$, one can derive that
\begin{align}
\Exs\brackets{\vecnorm{\tg_i(\x^{(t,0)}) - \ntg_i^{(t)}}^2}
=& \Exs\brackets{\vecnorm{\tg_i(\x^{(t,0)}) - \frac{1}{a_i}\sum_{k=0}^{\cp_i-1} a_{i,k} \tg_i(\x_i^{(t,k)})}^2} \\
=& \Exs\brackets{\vecnorm{\frac{1}{a_i}\sum_{k=0}^{\cp_i-1}a_{i,k}\parenth{\tg_i(\x^{(t,0)}) - \tg_i(\x_i^{(t,k)})}}^2} \\
\leq& \frac{1}{a_i}\sum_{k=0}^{\cp_i-1} \braces{a_{i,k} \Exs\brackets{\vecnorm{\tg_i(\x^{(t,0)}) - \tg_i(\x_i^{(t,k)})}^2}} \label{eqn:T3_step1} \\
\leq& \frac{\lip^2}{a_i}\sum_{k=0}^{\cp_i-1} \braces{a_{i,k}\Exs\brackets{\vecnorm{\x^{(t,0)} - \x_i^{(t,k)}}^2}} \label{eqn:T3_step2}
\end{align}
where \Cref{eqn:T3_step1} uses Jensen's Inequality again: $\vecnorm{\sum_{i=1}^\nworkers w_i z_i}^2 \leq \sum_{i=1}^\nworkers w_i\vecnorm{z_i}^2$, and \Cref{eqn:T3_step2} follows \Cref{assump:smooth}. Now, we turn to bounding the difference between the server model $\x^{(t,0)}$ and the local model $\x_i^{(t,k)}$. Plugging into the local update rule and using the fact $\vecnorm{a+b}^2 \leq 2\vecnorm{a}^2 + 2\vecnorm{b}^2$,
\begin{align}
\Exs\brackets{\vecnorm{\x^{(t,0)} - \x_i^{(t,k)}}^2}
=& \lr^2 \cdot \Exs\brackets{\vecnorm{\sum_{s=0}^{k-1}a_{i,s}\sg_i(\x_i^{(t,s)})}^2} \\
\leq& 2\lr^2\Exs\brackets{\vecnorm{\sum_{s=0}^{k-1}a_{i,s}\parenth{\sg_i(\x_i^{(t,s)})-\tg_i(\x_i^{(t,s)})}}^2}\\
&+ 2\lr^2\Exs\brackets{\vecnorm{\sum_{s=0}^{k-1}a_{i,s}\tg_i(\x_i^{(t,s)})}^2}
\end{align}
Applying \Cref{lem:sum_norm} to the first term,
\begin{align}
\Exs\brackets{\vecnorm{\x^{(t,0)} - \x_i^{(t,k)}}^2}
=& 2\lr^2\sum_{s=0}^{k-1} [a_{i,s}]^2 \Exs\brackets{\vecnorm{\sg_i(\x_i^{(t,s)}) - \tg_i(\x_i^{(t,s)})}^2} \nonumber \\
& + 2\lr^2\Exs\brackets{\vecnorm{\sum_{s=0}^{k-1} a_{i,s}\tg_i(\x_i^{(t,s)})}^2} \\
\leq& 2\lr^2 \vbnd \sum_{s=0}^{k-1} [a_{i,s}]^2 + 2\lr^2\Exs\brackets{\vecnorm{\sum_{s=0}^{k-1} a_{i,s} \tg_i(\x_i^{(t,s)})}^2} \\
\leq& 2\lr^2 \vbnd \sum_{s=0}^{k-1} [a_{i,s}]^2 + 2\lr^2 \brackets{\sum_{s=0}^{k-1} a_{i,s}} \sum_{s=0}^{k-1} a_{i,s}\Exs\brackets{\vecnorm{\tg_i(\x_i^{(t,s)})}^2}  \label{eqn:T3_step2_1} \\
\leq& 2\lr^2 \vbnd \sum_{s=0}^{k-1} [a_{i,s}]^2 + 2\lr^2 \brackets{\sum_{s=0}^{k-1} a_{i,s}} \sum_{s=0}^{\cp_i-1} a_{i,s}\Exs\brackets{\vecnorm{\tg_i(\x_i^{(t,s)})}^2}  \label{eqn:T3_step3}
\end{align}
where \Cref{eqn:T3_step2_1} follows from Jensen's Inequality. Furthermore, note that
\begin{align}
\frac{1}{\vecnorm{\bm{a}_i}_1}\sum_{k=0}^{\cp_i-1} a_{i,k} \brackets{\sum_{s=0}^{k-1}[a_{i,s}]^2}
\leq& \frac{1}{a_i}\sum_{k=0}^{\cp_i-1} a_{i,k} \brackets{\sum_{s=0}^{\cp_i-2}[a_{i,s}]^2} \\
=& \sum_{s=0}^{\cp_i-2}[a_{i,s}]^2 = \vecnorm{\bm{a}_i}_2^2 - [a_{i,-1}]^2,
\end{align}
\begin{align}
\frac{1}{\vecnorm{\bm{a}_i}_1}\sum_{k=0}^{\cp_i-1} a_{i,k} \brackets{\sum_{s=0}^{k-1}[a_{i,s}]}
\leq& \frac{1}{a_i}\sum_{k=0}^{\cp_i-1} a_{i,k} \brackets{\sum_{s=0}^{\cp_i-2}[a_{i,s}]} \\
=& \sum_{s=0}^{\cp_i-2}[a_{i,s}] = \vecnorm{\bm{a}_i}_1 - a_{i,-1}
\end{align}
where $a_{i,-1}$ is the last element in the vector $\gradweight_i$. As a result, we have
\begin{align}
\frac{1}{\vecnorm{\bm{a}_i}_1}\sum_{k=0}^{\cp_i-1} a_{i,k}\Exs\brackets{\vecnorm{\x^{(t,0)} - \x_i^{(t,k)}}^2}
\leq& 2\lr^2\vbnd \parenth{\vecnorm{\bm{a}_i}_2^2 - [a_{i,-1}]^2} \nonumber \\
+ 2\lr^2 \parenth{\vecnorm{\bm{a}_i}_1 - a_{i,-1}} & \sum_{k=0}^{\cp_i-1} a_{i,s}\Exs\brackets{\vecnorm{\tg_i(\x_i^{(t,k)})}^2}
\end{align}
In addition, we can bound the second term using the following inequality:
\begin{align}
\Exs\brackets{\vecnorm{\tg_i(\x_i^{(t,k)})}^2}
\leq& 2\Exs\brackets{\vecnorm{\tg_i(\x_i^{(t,k)}) - \tg_i(\x^{(t,0)})}^2} + 2\Exs\brackets{\vecnorm{\tg_i(\x^{(t,0)})}^2} \\
\leq& 2\lip^2\Exs\brackets{\vecnorm{\x^{(t,0)} - \x_i^{(t,k)}}^2} + 2\Exs\brackets{\vecnorm{\tg_i(\x^{(t,0)})}^2}. \label{eqn:T3_step4}
\end{align}
Substituting \Cref{eqn:T3_step4} into \Cref{eqn:T3_step3}, we get
\begin{align}
&\frac{1}{\vecnorm{\bm{a}_i}_1}\sum_{k=0}^{\cp_i-1} a_{i,k}\Exs\brackets{\vecnorm{\x^{(t,0)} - \x_i^{(t,k)}}^2} \nonumber \\ 
\leq& 2\lr^2\vbnd \parenth{\vecnorm{\bm{a}_i}_2^2 - [a_{i,-1}]^2} 
+ 4\lr^2 \lip^2 \parenth{\vecnorm{\bm{a}_i}_1 - a_{i,-1}} \sum_{k=0}^{\cp_i-1} a_{i,k}\Exs\brackets{\vecnorm{\x^{(t,0)} - \x_i^{(t,k)}}^2} \nonumber \\
& + 4\lr^2 \parenth{\vecnorm{\bm{a}_i}_1 - a_{i,-1}}\sum_{k=0}^{\cp_i-1} a_{i,k}\Exs\brackets{\vecnorm{\tg_i(\x_i^{(t,0)})}^2}
\end{align}
After minor rearranging, it follows that
\begin{align}
\frac{1}{\vecnorm{\bm{a}_i}_1}\sum_{k=0}^{\cp_i-1}a_{i,k}\Exs\brackets{\vecnorm{\x^{(t,0)} - \x_i^{(t,k)}}^2}
\leq& \frac{2\lr^2\vbnd}{1-4\lr^2\lip^2 \vecnorm{\bm{a}_i}_1(\vecnorm{\bm{a}_i}_1-a_{i,-1})} \parenth{\vecnorm{\bm{a}_i}_2^2 - [a_{i,-1}]^2} \nonumber \\
+ &\frac{4\lr^2\vecnorm{\bm{a}_i}_1 (\vecnorm{\bm{a}_i}_1-a_{i,-1})}{1-4\lr^2\lip^2 \vecnorm{\bm{a}_i}_1(\vecnorm{\bm{a}_i}_1-a_{i,-1})}\Exs\brackets{\vecnorm{\tg_i(\x^{(t,0)})}^2} \label{eqn:T3_step5}
\end{align}
Define $D = 4\lr^2\lip^2 \max_i\{\vecnorm{\bm{a}_i}_1(\vecnorm{\bm{a}_i}_1-a_{i,-1}) \} < 1$. We can simplify \Cref{eqn:T3_step5} as follows
\begin{align}
\frac{\lip^2}{a_i}\sum_{k=0}^{\cp_i-1}a_{i,k}\Exs\brackets{\vecnorm{\x^{(t,0)} - \x_i^{(t,k)}}^2}
\leq& \frac{2\lr^2\lip^2\vbnd}{1-D}\parenth{\vecnorm{\bm{a}_i}_2^2 - [a_{i,-1}]^2} + \frac{D}{1-D}\Exs\brackets{\vecnorm{\tg_i(\x^{(t,0)})}^2}. \label{eqn:T3_step6}
\end{align}
Taking the average across all workers and applying \Cref{assump:dissimilarity}, one can obtain
\begin{align}
\frac{1}{2}\sum_{i=1}^\nworkers w_i \Exs\brackets{\vecnorm{\tg_i(\x^{(t,0)}) - \ntg_i^{(t)}}^2}
\leq& \frac{\lr^2\lip^2\vbnd}{1-D}\sum_{i=1}^\nworkers w_i \parenth{\vecnorm{\bm{a}_i}_2^2 - [a_{i,-1}]^2} \nonumber \\
& + \frac{D}{2(1-D)}\sum_{i=1}^\nworkers w_i\Exs\brackets{\vecnorm{\tg_i(\x^{(t,0)})}^2} \\
\leq& \frac{\lr^2\lip^2\vbnd}{1-D}\sum_{i=1}^\nworkers w_i \parenth{\vecnorm{\bm{a}_i}_2^2 - [a_{i,-1}]^2} \nonumber \\
& + \frac{D\bnda}{2(1-D)}\Exs\brackets{\vecnorm{\nabla \surloss(\x^{(t,0)})}^2} + \frac{D\bndb}{2(1-D)}. \label{eqn:T3_step7}
\end{align}
Now, we are ready to derive the final result.

\subsection{Final Results}
Plugging \Cref{eqn:T3_step7} back into \Cref{eqn:mid_result}, we have
\begin{align}
\frac{\Exs\brackets{\surloss(\x^{(t+1,0)})} - \surloss(\x^{(t,0)})}{\lr\cpeff}
\leq& -\frac{1}{2} \vecnorm{\nabla \surloss(\x^{(t,0)})}^2 + \cpeff\lr\lip\vbnd\sum_{i=1}^\nworkers \frac{w_i^2 \vecnorm{\bm{a}_i}_2^2}{\vecnorm{\bm{a}_i}_1^2} \nonumber \\
& + \frac{\lr^2\lip^2\vbnd}{1-D}\sum_{i=1}^\nworkers w_i \parenth{\vecnorm{\bm{a}_i}_2^2 - [a_{i,-1}]^2} \nonumber \\
& + \frac{D\kappa^2}{2(1-D)} + \frac{D\beta^2}{2(1-D)} \Exs\brackets{\vecnorm{\nabla \surloss(\x^{(t,0)})}^2} \\
=& -\frac{1}{2}\parenth{\frac{1-D(1+\beta^2)}{1-D}}\vecnorm{\nabla \surloss(\x^{(t,0)})}^2 + \cpeff\lr\lip\vbnd\sum_{i=1}^\nworkers \frac{w_i^2 \vecnorm{\bm{a}_i}_2^2}{\vecnorm{\bm{a}_i}_1^2} + \nonumber \\
& + \frac{\lr^2\lip^2\vbnd}{1-D} \sum_{i=1}^\nworkers w_i \parenth{\vecnorm{\bm{a}_i}_2^2 - [a_{i,-1}]^2} + \frac{D\kappa^2}{2(1-D)}. \label{eqn:final_step1}
\end{align}
If $D \leq \frac{1}{2\beta^2+1}$, then it follows that $\frac{1}{1-D}\leq 1+\frac{1}{2\beta^2}$ and $\frac{D\beta^2}{1-D} \leq \frac{1}{2}$. These facts can help us further simplify inequality \Cref{eqn:final_step1}.
\begin{align}
\frac{\Exs\brackets{\surloss(\x^{(t+1,0)})} - \surloss(\x^{(t,0)})}{\lr\cpeff}
\leq& -\frac{1}{4} \vecnorm{\nabla \surloss(\x^{(t,0)})}^2 + \cpeff\lr\lip\vbnd\sum_{i=1}^\nworkers  \frac{w_i^2 \vecnorm{\bm{a}_i}_2^2}{\vecnorm{\bm{a}_i}_1^2} \nonumber \\
& + \lr^2\lip^2\vbnd \parenth{1+\frac{1}{2\beta^2}} \sum_{i=1}^\nworkers w_i \parenth{\vecnorm{\bm{a}_i}_2^2 - [a_{i,-1}]^2}\nonumber \\
& +  2\lr^2\lip^2 \max_i\{\vecnorm{\bm{a}_i}_1(\vecnorm{\bm{a}_i}_1-a_{i,-1}) \} \kappa^2 \parenth{1+\frac{1}{2\beta^2}} \\
\leq& -\frac{1}{4} \vecnorm{\nabla \surloss(\x^{(t,0)})}^2 + \cpeff\lr\lip\vbnd\sum_{i=1}^\nworkers \frac{w_i^2 \vecnorm{\bm{a}_i}_2^2}{\vecnorm{\bm{a}_i}_1^2} \nonumber \\
& + \frac{3}{2}\lr^2\lip^2\vbnd \sum_{i=1}^\nworkers w_i \parenth{\vecnorm{\bm{a}_i}_2^2 - [a_{i,-1}]^2} \nonumber \\
& +  3\lr^2\lip^2\bndb \max_i\{\vecnorm{\bm{a}_i}_1(\vecnorm{\bm{a}_i}_1-a_{i,-1}) \} \label{eqn:final_step2}
\end{align}
Taking the average across all rounds, we get
\begin{align}
\frac{1}{T}\sum_{t=0}^{T-1}\Exs\brackets{\vecnorm{\nabla \surloss(\x^{(t,0)})}^2}
\leq& \frac{4\brackets{\surloss(\x^{(0,0)}) - \surloss_{\text{inf}}}}{\lr\cpeff T} + 4\cpeff\lr\lip\vbnd\sum_{i=1}^\nworkers \frac{w_i^2 \vecnorm{\bm{a}_i}_2^2}{\vecnorm{\bm{a}_i}_1^2} \nonumber \\
& + 6\lr^2\lip^2\vbnd\sum_{i=1}^\nworkers w_i \parenth{\vecnorm{\bm{a}_i}_2^2 - [a_{i,-1}]^2} \nonumber \\
& + 12\lr^2\lip^2\bndb \max_i\{\vecnorm{\bm{a}_i}_1(\vecnorm{\bm{a}_i}_1-a_{i,-1}) \}.
\end{align}
For the ease of writing, we define the following auxiliary variables:
\begin{align}
A &= \nworkers \cpeff\sum_{i=1}^\nworkers \frac{w_i^2 \vecnorm{\bm{a}_i}_2^2}{\vecnorm{\bm{a}_i}_1^2}, \\
B &= \sum_{i=1}^\nworkers w_i \parenth{\vecnorm{\bm{a}_i}_2^2 - [a_{i,-1}]^2}, \\
C &= \max_i\{\vecnorm{\bm{a}_i}_1(\vecnorm{\bm{a}_i}_1-a_{i,-1}) \}.
\end{align}
It follows that
\begin{align}
\frac{1}{T}\sum_{t=0}^{T-1}\Exs\brackets{\vecnorm{\nabla \surloss(\x^{(t,0)})}^2}
\leq& \frac{4\brackets{\surloss(\x^{(0,0)}) - \surloss_{\text{inf}}}}{\lr\cpeff T} + \frac{4\lr\lip\vbnd A}{\nworkers} + 6\lr^2\lip^2\vbnd B + 12\lr^2\lip^2\bndb C
\end{align}
Since $\min \Exs\brackets{\vecnorm{\nabla \surloss(\x^{(t,0)})}^2} \leq \frac{1}{T}\sum_{t=0}^{T-1}\Exs\brackets{\vecnorm{\nabla \surloss(\x^{(t,0)})}^2}$, we have
\begin{align}
\min_{t\in[T]} \Exs\brackets{\vecnorm{\nabla \surloss(\x^{(t,0)})}^2}
\leq& \frac{4\brackets{\surloss(\x^{(0,0)}) - \surloss_{\text{inf}}}}{\lr\cpeff T} + \frac{4\lr\lip\vbnd A}{\nworkers} + 6\lr^2\lip^2\vbnd B + 12\lr^2\lip^2\bndb C. \label{eqn:final_fixed_lr}
\end{align}

\subsection{Constraint on Local Learning Rate}
Here, let us summarize the constraints on local learning rate:
\begin{align}
\lr\lip &\leq \frac{1}{2\cpeff}, \\
4\lr^2\lip^2 \max_i\{\vecnorm{\bm{a}_i}_1(\vecnorm{\bm{a}_i}_1-a_{i,-1}) \} &\leq \frac{1}{2\beta^2+1}.
\end{align}
For the second constraint, we can further tighten it as follows:
\begin{align}
4\lr^2\lip^2\max_i\{\vecnorm{\bm{a}_i}_1(\vecnorm{\bm{a}_i}_1-a_{i,-1}) \} \leq 4\lr^2\lip^2 \max_i \vecnorm{\gradweight_i}_1^2
&\leq \frac{1}{2\beta^2+1}
\end{align}
That is,
\begin{align}
\lr\lip \leq \frac{1}{2}\min\braces{\frac{1}{\max_i \vecnorm{\gradweight_i}_1 \sqrt{2\beta^2+1}}, \frac{1}{\cpeff}}.
\end{align}

\subsection{Further Optimizing the Bound}
By setting $\lr = \sqrt{\frac{\nworkers}{\cpavg T}}$ where $\cpavg = \frac{1}{\nworkers}\sum_{i=1}^\nworkers \cp_i$, we have
\begin{align}
\min_{t\in[T]} \Exs\vecnorm{\nabla \surloss(\x^{(t,0)})}^2
\leq& \mathcal{O}\parenth{\frac{\cpavg/\cpeff}{\sqrt{\nworkers \cpavg T}}} + \mathcal{O}\parenth{\frac{A \vbnd}{\sqrt{\nworkers \cpavg T}}} + \mathcal{O}\parenth{\frac{\nworkers B \vbnd}{\cpavg T}} + \mathcal{O}\parenth{\frac{\nworkers C\bndb}{\cpavg T}}.
\end{align}
Here, we complete the proof of \Cref{thm:general}.


\section{Proof of \Cref{thm:bias}: Including Bias in the Error Bound}\label{sec:proof_thm2}
\begin{lem}
	For any model parameter $\x$, the difference between the gradients of $\obj(\x)$ and $\surloss(\x)$ can be bounded as follows:
	\begin{align}
	\vecnorm{\nabla \obj(\x) - \nabla \surloss(\x)}^2 \leq \csqdist\brackets{(\bnda-1)\vecnorm{\nabla \surloss(\x)}^2 + \bndb}
	\end{align}
	where $\csqdist$ denotes the chi-square distance between $\bm{p}$ and $\bm{w}$, \ie, $\csqdist=\sum_{i=1}^\nworkers (p_i-w_i)^2/w_i$.
\end{lem}
\begin{proof}
	According to the definition of $\obj(x)$ and $\surloss(\x)$, we have
	\begin{align}
	\nabla \obj(x) - \nabla \surloss(\x)
	=& \sum_{i=1}^\nworkers (p_i - w_i) \nabla \obj_i(\x) \\
	=& \sum_{i=1}^\nworkers (p_i - w_i) \parenth{\nabla \obj_i(\x)-\nabla\surloss(\x)} \\
	=& \sum_{i=1}^\nworkers \frac{p_i-w_i}{\sqrt{w}_i} \cdot \sqrt{w}_i \parenth{\nabla \obj_i(\x)-\nabla\surloss(\x)}.
	\end{align}
	Applying Cauchy–Schwarz inequality, it follows that
	\begin{align}
	\vecnorm{\nabla \obj(x) - \nabla \surloss(\x)}^2
	\leq& \brackets{\sum_{i=1}^\nworkers \frac{(p_i-w_i)^2}{w_i}} \brackets{\sum_{i=1}^\nworkers w_i \vecnorm{\nabla \obj_i(x) - \nabla \surloss(\x)}^2} \\
	\leq& \csqdist \brackets{(\bnda-1)\vecnorm{\nabla \surloss(\x)}^2 + \bndb}.
	\end{align}
	where the last inequality uses \Cref{assump:dissimilarity}.
\end{proof}

Note that
\begin{align}
\vecnorm{\nabla \obj(\x)}^2
\leq& 2\vecnorm{\nabla \obj(\x) - \nabla \surloss(\x)}^2 + 2\vecnorm{\nabla \surloss(\x)}^2 \\
\leq& 2\brackets{\csqdist(\bnda-1)+1}\vecnorm{\nabla \surloss(\x)}^2 + 2\csqdist \bndb.
\end{align}
As a result, we obtain
\begin{align}
\min_{t\in[T]} \vecnorm{\nabla \obj(\x^{(t,0)})}^2
\leq& \frac{1}{T}\sum_{t=0}^{T-1}\vecnorm{\nabla \obj(\x^{(t,0)})}^2 \\
\leq& 2\brackets{\csqdist(\bnda-1)+1}\frac{1}{T}\sum_{t=0}^{T-1}\vecnorm{\nabla \surloss(\x^{(t,0)})}^2 + 2\csqdist \bndb \\
\leq& 2\brackets{\csqdist(\bnda-1)+1}\epsilon_{\text{opt}} + 2\csqdist \bndb
\end{align}
where $\epsilon_{\text{opt}}$ denotes the optimization error.

\subsection{Constructing a Lower Bound}\label{sec:lower_bound}
In this subsection, we are going to construct a lower bound of $\Exs\vecnorm{\tg(\x^{(t,0)})}^2$, showing that \Cref{eqn:error_decomp} is tight and the non-vanishing error term in \Cref{thm:bias} is not an artifact of our analysis.
\begin{lem}
	One can manually construct a strongly convex objective function such that \fedavg with heterogeneous local updates cannot converge to its global optimum. In particular, the gradient norm of the objective function does not vanish as learning rate approaches to zero. We have the following lower bound:
	\begin{align}
	\lim_{T \rightarrow \infty}\Exs\vecnorm{\tg(\x^{(T,0)})}^2 = \Omega(\csqdist\bndb)
	\end{align}
	where $\csqdist$ denotes the chi-square divergence between weight vectors and $\bndb$ quantifies the dissimilarities among local objective functions and is defined in \Cref{assump:dissimilarity}.
\end{lem}
\begin{proof}
	Suppose that there are only two clients with local objectives $\obj_1(x) = \frac{1}{2}(x-a)^2$ and $\obj_2(x) = \frac{1}{2}(w+a)^2$. The global objective is defined as $\obj(x) = \frac{1}{2}\obj_1(x) + \frac{1}{2}\obj_2(x)$. For any set of weights $w_1,w_2,w_1+w_2=1$, we define the surrogate objective function as $\widetilde{F}(\x)=w_1 F_1(\x) + w_2 F_2(\x)$. As a consequence, we have
	\begin{align}
	&\sum_{i=1}^m w_i \vecnorm{\tg_i(x) - \nabla\surloss(x)}^2  \nonumber \\
	=& w_1 [(x-a)-[x-(w_1-w_2)a]]^2 + w_2 [(x+a)-[x-(w_1-w_2)a]]^2 \\
	=& w_1[2w_2 a]^2 + w_2 [2w_1 a]^2 = 2(w1+w_2)(w_1 w_2 a^2) = 2w_1 w_2 a^2
	\end{align}
	Comparing with \Cref{assump:dissimilarity}, we can define $\bndb=2 w_1 w_2 a^2$ and $\bnda=1$ in this case. Furthermore, according to the derivations in \Cref{sec:proof_quadratic}, the iterate of \fedavg can be written as follows:
	\begin{align}
	\lim_{T\rightarrow \infty} x^{(T,0)} = \frac{\cp_1 a - \cp_2 a}{\cp_1+\cp_2} \quad \text{and} \quad w_1=\cp_1/(\cp_1+\cp_2), w_2 = \cp_2/(\cp_1+\cp_2).
	\end{align}
	As a results, we have
	\begin{align}
	\lim_{T\rightarrow \infty}\vecnorm{\tg(x^{(T,0)})}^2
	=& \lim_{T\rightarrow \infty} \brackets{\frac{1}{2}(x^{(T,0)}-a) + \frac{1}{2}(x^{(T,0)}+a)}^2 \\
	=& \lim_{T\rightarrow \infty} \brackets{x^{(T,0)}}^2 \\
	=& \parenth{\frac{\cp_1-\cp_2}{\cp_1+\cp_2}}^2 a^2 = \frac{(\cp_2-\cp_1)^2}{2\cp_1\cp_2}\bndb = \Omega(\csqdist\kappa^2).
	\end{align}
	where $\csqdist=\sum_{i=1}^\nworkers (p_i-w_i)^2/w_i = (w_1-1/2)^2/w_1+(w_2-1/2)^2/w_2=(\cp_2-\cp_1)^2/(2\cp_1\cp_2)$.
\end{proof}

\section{Special Cases of \Cref{thm:general}}\label{sec:proof_special_cases}
Here, we provide several instantiations of \Cref{thm:general} and check its consistency with previous results.
\subsection{FedAvg}
In \fedavg, $\bm{a}_i = [1, 1, \dots, 1]\tp \in \mathbb{R}^{\cp_i}$, $\vecnorm{\bm{a}_i}_2^2 = \cp_i$, and $\vecnorm{\bm{a}_i}_1 = \cp_i$. In addition, we have $w_i = p_i \cp_i/(\sum_{i=1}^\nworkers p_i \cp_i)$. Accordingly, we get the closed-form expressions of the following quantities:
\begin{align}
\cpeff &= \sum_{i=1}^\nworkers p_i \cp_i = \Exs_{\bm{p}}[\bm{\cp}],\\
A_{\fedavg} &= \nworkers\cpeff\sum_{i=1}^\nworkers \frac{w_i^2 \vecnorm{\bm{a}_i}_2^2}{\vecnorm{\bm{a}_i}_1^2} = \frac{\nworkers\sum_{i=1}^\nworkers p_i^2\cp_i}{\sum_{i=1}^\nworkers p_i \cp_i}, \\
B_{\fedavg} &= \sum_{i=1}^\nworkers w_i \parenth{\vecnorm{\bm{a}_i}_2^2 - [a_{i,-1}]^2} = \frac{\sum_{i=1}^\nworkers p_i \cp_i (\cp_i -1)}{\sum_{i=1}^\nworkers p_i \cp_i} = \Exs_{\bm{p}}[\bm{\cp}] -1 + \frac{\var_{\bm{p}}[\bm{\cp}]}{\Exs_{\bm{p}}[\bm{\cp}]}, \\
C_{\fedavg} &= \max_i\{\vecnorm{\bm{a}_i}_1(\vecnorm{\bm{a}_i}_1-a_{i,-1}) \} = \cpmax(\cpmax - 1).
\end{align}
In the case where all clients have the same local dataset size, \ie, $p_i=1/\nworkers, \forall i$. It follows that
\begin{align}
\cpeff = \cpavg, \
A_{\fedavg} = 1, \ 
B_{\fedavg} = \cpavg - 1 + \frac{\var[\bm{\cp}]}{\cpavg}, \ 
C_{\fedavg} = \cpmax(\cpmax - 1). \label{eqn:fedavg_special_case_abc}
\end{align}
Substituting \Cref{eqn:fedavg_special_case_abc} into \Cref{thm:general}, we get the convergence guarantee for \fedavg. We formally state it in the following corollary.
\begin{corollary}[\textbf{Convergence of FedAvg}]\label{corollary:fedavg_special}
	Under the same conditions as \Cref{thm:general}, if $p_i=1/\nworkers$, then \fedavg algorithm (vanilla SGD with fixed local learning rate as local solver) will converge to the stationary point of a surrogate objective $\widetilde{F}(\x)=\sum_{i=1}^\nworkers \cp_i \obj_i(\x)/\sum_{i=1}^\nworkers \cp_i$. The optimization error will be bounded as follows:
	\begin{align}
	\min_{t\in[T]}\Exs\vecnorm{\nabla \surloss(\x)}^2
	\leq& \mathcal{O}\parenth{\frac{1+\vbnd}{\sqrt{\nworkers \cpavg T}}} + \mathcal{O}\parenth{\frac{\nworkers\vbnd(\cpavg-1+\var[\bm{\cp}]/\cpavg)}{\cpavg T}} + \mathcal{O}\parenth{\frac{\nworkers\bndb\cpmax(\cpmax-1)}{\cpavg T}} \label{eqn:fedavg_special}
	\end{align}
	where $\mathcal{O}$ swallows all constants (including $\lip$), and $\var[\bm{\cp}]=\sum_{i=1}^\nworkers \cp_i^2/\nworkers - \cpavg^2$ denotes the variance of local steps.
\end{corollary}
\textbf{Consistent with Previous Results.} When all clients perform the same local steps, $\ie, \cp_i = \cp$, then $\var[\bm{\cp}]=0$ and the above error bound \Cref{eqn:fedavg_special} recovers previous results \cite{wang2018cooperative,yu2019linear,karimireddy2019scaffold}. When $\cp_i=1$, then \fedavg reduces to fully synchronous SGD and the error bound \Cref{eqn:fedavg_special} becomes $1/\sqrt{\nworkers T}$, which is the same as standard SGD convergence rate \cite{bottou2016optimization}.

\subsection{FedProx}
In \fedprox, we have $\bm{a}_i = [(1-\alpha)^{\cp_i-1}, \dots, (1-\alpha), 1]\tp \in \mathbb{R}^{\cp_i}$. Accordingly, the norms of $\gradweight_i$ can be written as:
\begin{align}
\vecnorm{\bm{a}_i}_2^2 
= \frac{1-(1-\alpha)^{2\cp_i}}{1-(1-\alpha)^2}, \
\vecnorm{\bm{a}_i}_1
= \frac{1-(1-\alpha)^{\cp_i}}{\alpha}, \
w_i = \frac{p_i [1-(1-\alpha)^{\cp_i}]}{\sum_{i=1}^\nworkers p_i [1-(1-\alpha)^{\cp_i}]}.
\end{align}
As a consequence, we can derive the closed-form expression of $\cpeff,A,B,C$ as follows:
\begin{align}
\cpeff 
=& \frac{1}{\alpha}\sum_{i=1}^\nworkers p_i [1-(1-\alpha)^{\cp_i}], \\
A_{\fedprox}
=& \frac{\nworkers\alpha}{\sum_{i=1}^\nworkers p_i (1-(1-\alpha)^{\cp_i})}\sum_{i=1}^\nworkers p_i^2 \frac{1- (1-\alpha)^{2\cp_i}}{1-(1-\alpha)^2}, \\
B_{\fedprox}
=& \sum_{i=1}^\nworkers \frac{p_i[1-(1-\alpha)^{\cp_i}]}{\sum_{i=1}^\nworkers p_i[1-(1-\alpha)^{\cp_i}] }\brackets{\frac{1- (1-\alpha)^{2\cp_i}}{1- (1-\alpha)^2} - 1},\\
C_{\fedprox}
=& \frac{1-(1-\alpha)^{\cpmax}}{\alpha}\parenth{\frac{1-(1-\alpha)^{\cpmax}}{\alpha}-1}.
\end{align}
Substituting $A_\fedprox,B_\fedprox,C_\fedprox$ back into \Cref{thm:general}, one can obtain the convergence guarantee for \fedprox. Again, it will converge to the stationary points of a surrogate objective due to $w_i \neq p_i$.

\textbf{Consistency with FedAvg.} From the update rule of \fedprox, we know that when $\mu=0$ (or $\alpha=0$), \fedprox is equivalent to \fedprox. This can also be validated from the expressions of $A_\fedprox, B_\fedprox, C_\fedprox$. Using L'Hospital law, it is easy to show that
\begin{align}
\lim_{\alpha\rightarrow 0} A_\fedprox = A_\fedavg, \ 
\lim_{\alpha\rightarrow 0} B_\fedprox = B_\fedavg, \ 
\lim_{\alpha\rightarrow 0} C_\fedprox = C_\fedavg.
\end{align}

\textbf{Best value of $\alpha$ in FedProx.} Given the expressions of $\cpeff$ and $A,B,C$, we can further select a best value of $\alpha$ that optimizes the error bound of \fedprox, as stated in the following corollary.
\begin{corollary}\label{corrollary:fedprox_special}
	Under the same conditions as \Cref{thm:general} and suppose $p_i = 1/\nworkers$ and $\cp_i \gg 1$, then $\alpha = \mathcal{O}(\nicefrac{\nworkers^{\frac{1}{2}}}{\cpavg^{\frac{1}{2}}T^{\frac{1}{6}}}))$ minimizes the optimization error bound of \fedprox in terms of converging to the stationary points of the surrogate objective. In particular, we have
	\begin{align}
	\min_{t\in[T]}\Exs\vecnorm{\nabla \surloss(\x)}^2
	\leq& \mathcal{O}\parenth{\frac{1}{\sqrt{\nworkers \cpavg T}}} + \mathcal{O}\parenth{\frac{1}{T^{\frac{2}{3}}}}
	\end{align}
	where $\mathcal{O}$ swallows all other constants. Furthermore, if we define $K=\cpavg T$ the average gradient evaluations at clients and let $\cpavg \leq \mathcal{O}(K^{\frac{1}{4}}\nworkers^{-\frac{3}{4}})$ (which is equivalent to $T\geq \mathcal{O}(K^{\frac{3}{4}}\nworkers^{\frac{3}{4}})$), then it follows that $\min_{t\in[T]}\Exs\vecnorm{\nabla \surloss(\x)}^2 \leq \mathcal{O}(1/\sqrt{\nworkers K})$.
\end{corollary}
\textbf{Discussion:} \Cref{corrollary:fedprox_special} shows that there exists a non-zero value of $\alpha$ that optimizes the error upper bound of \fedprox. That is to say, \fedprox ($\alpha>0$) is better than \fedavg ($\alpha=0$) by a constant in terms of error upper bound. However, on the other hand, it is worth noting that the minimal communication rounds of \fedprox to achieve $1/\sqrt{\nworkers K}$ rate, given by \Cref{corrollary:fedprox_special}, is exactly the same as \fedavg \cite{yu2019linear}. In this sense, \fedprox has the same convergence rate as \fedavg and cannot further reduce the communication overhead.
\begin{proof}
	First of all, let us relax the error terms of \fedprox. Under the assumption of $\cp_i\gg 1$, the quantities $A,B,C$ can be bounded or approximated as follows:
	\begin{align}
	\cpeff 
	\simeq& \frac{1}{\alpha}, \\
	A_\fedprox
	\simeq& \nworkers \alpha \sum_{i=1}^\nworkers \frac{p_i^2}{(2-\alpha)\alpha} = \frac{\nworkers \sum_{i=1}^\nworkers p_i^2}{2-\alpha}\leq m\sum_{i=1}^\nworkers p_i^2 = 1, \\
	B_\fedprox
	\leq& \frac{1- (1-\alpha)^{2\cp_i}}{1- (1-\alpha)^2} - 1 \leq \frac{1}{\alpha(2-\alpha)} \leq \frac{1}{\alpha} \leq \frac{1}{\alpha^2}, \\
	C_\fedprox
	\leq& \frac{1}{\alpha^2}.
	\end{align}
	Accordingly, the error upper bound of \fedprox can be rewritten as follows:
	\begin{align}
	\min_{t\in[T]}\Exs\vecnorm{\nabla \surloss(\x)}^2
	\leq& \mathcal{O}\parenth{\frac{\alpha\cpavg}{\sqrt{\nworkers \cpavg T}}} + \mathcal{O}\parenth{\frac{1}{\sqrt{\nworkers \cpavg T}}} + \mathcal{O}\parenth{\frac{\nworkers}{\alpha^2 \cpavg T}}. \label{eqn:final_fedprox}
	\end{align}
	In order to optimize the above bound, we can simply take the derivative with respect to $\alpha$. When the derivative equals to zero, we get
	\begin{align}
	\frac{\cpavg}{\sqrt{\nworkers \cpavg T}} = \frac{\nworkers}{\alpha^3 \cpavg T} \ \Longrightarrow \
	\alpha = \mathcal{O}\parenth{\frac{m^\frac{1}{2}}{\cpavg^{\frac{1}{2}}T^\frac{1}{6}}}.
	\end{align}
	Plugging the expression of best $\alpha$ into \Cref{eqn:final_fedprox}, we have
	\begin{align}
	\min_{t\in[T]}\Exs\vecnorm{\nabla \surloss(\x)}^2
	\leq \mathcal{O}\parenth{\frac{1}{\sqrt{\nworkers \cpavg T}}} + \mathcal{O}\parenth{\frac{1}{T^{\frac{2}{3}}}}
	= \mathcal{O}\parenth{\frac{1}{\sqrt{\nworkers K}}} + \mathcal{O}\parenth{\frac{\cp^{\frac{2}{3}}}{K^{\frac{2}{3}}}}
	\end{align}
	where $K = \cpavg T$ denotes the average total gradient steps at clients. In order to let the first term dominates the convergence rate, it requires that
	\begin{align}
	\frac{1}{\sqrt{\nworkers K}} \geq \frac{\cpavg^{\frac{2}{3}}}{K^{\frac{2}{3}}} \ \Longrightarrow \
	\cpavg \leq \mathcal{O}\parenth{K^{\frac{1}{4}}\nworkers^{-\frac{3}{4}}}.
	\end{align}
	As a results, the total communication rounds $T=K/\cpavg$ should be greater than $\mathcal{O}(K^{\frac{3}{4}}\nworkers^{\frac{3}{4}})$.
\end{proof}

\newpage
\section{Proof of \Cref{thm:nova}}\label{sec:proof_sec3}
In the case of \fednova, the aggregated weights $w_i$ equals to $p_i$. Therefore, the surrogate objective $\surloss(\x)=\sum_{i=1}^\nworkers w_i \obj_i(\x)$ is the same as the original objective function $\obj(\x)=\sum_{i=1}^\nworkers p_i \obj_i(\x)$. We can directly reuse the intermediate results in the proof of \Cref{thm:general}. According to \Cref{eqn:final_step2}, we have
\begin{align}
\frac{\Exs[\obj(\x^{(t+1,0)})] - \obj(\x^{(t,0)})}{\lr\cpeff}
\leq& -\frac{1}{4}\vecnorm{\tg(\x^{(t,0)})}^2 + \frac{\lr\lip\vbnd A^{(t)}}{\nworkers}  + \frac{3}{2}\lr^2\lip^2\vbnd B^{(t)} + 3\lr^2\lip^2\bndb C^{(t)} \label{eqn:nova_final_step1}
\end{align}
where quantities $A^{(t)}, B^{(t)}, C^{(t)}$ are defined as follows:
\begin{align}
A^{(t)} &= \nworkers \cpeff\sum_{i=1}^\nworkers \frac{w_i^2 \vecnorm{\bm{a}_i^{(t)}}_2^2}{\vecnorm{\bm{a}_i^{(t)}}_1^2}, \\
B^{(t)} &= \sum_{i=1}^\nworkers p_i \parenth{\vecnorm{\bm{a}_i^{(t)}}_2^2 - [a_{i,-1}^{(t)}]^2}, \\
C^{(t)} &= \max_i\braces{\vecnorm{\bm{a}_i^{(t)}}_1\parenth{\vecnorm{\bm{a}_i^{(t)}}_1-a_{i,-1}^{(t)}}}.
\end{align}
Taking the total expectation and averaging over all rounds, it follows that
\begin{align}
\frac{\Exs[\obj(\x^{(T,0)})] - \obj(\x^{(0,0)})}{\lr\cpeff T}
\leq& -\frac{1}{4T}\sum_{t=0}^{T-1}\Exs\vecnorm{\tg(\x^{(t,0)})}^2 + \frac{\lr\lip \vbnd \widetilde{A}}{\nworkers} \nonumber \\
&+ \frac{3}{2}\lr^2\lip^2 \vbnd \widetilde{B} + 3\lr^2\lip^2 \bndb \widetilde{C}
\end{align}
where $\widetilde{A} = \sum_{t=0}^{T-1}A^{(t)}/T, \widetilde{B} = \sum_{t=0}^{T-1}B^{(t)}/T$, and $\widetilde{C} = \sum_{t=0}^{T-1}C^{(t)}/T$. After minor rearranging, we have
\begin{align}
\frac{1}{T}\sum_{t=0}^{T-1}\Exs\brackets{\vecnorm{\nabla \obj(\x^{(t,0)})}^2}
\leq& \frac{4\brackets{\obj(\x^{(0,0)}) - \obj_{\text{inf}}}}{\lr\cpeff T} + \frac{4\lr\lip\vbnd \widetilde{A}}{\nworkers} + 6\lr^2\lip^2\vbnd \widetilde{B} + 12\lr^2\lip^2\bndb \widetilde{C}.
\end{align}
Bt setting $\lr = \sqrt{\frac{\nworkers}{\widetilde{\cp} T}}$ where $\widetilde{\cp} = \sum_{t=0}^{T-1}\cpavg^{(t)}/T$, the above upper bound can be further optimized as follows:
\begin{align}
\min_{t\in[T]}\Exs\brackets{\vecnorm{\nabla \obj(\x^{(t,0)})}^2}
\leq& \frac{1}{T}\sum_{t=0}^{T-1}\Exs\brackets{\vecnorm{\nabla \obj(\x^{(t,0)})}^2} \\
\leq& \frac{4\widetilde{\cp}/\cpeff \cdot \brackets{\obj(\x^{(0,0)}) - \obj_{\text{inf}}}}{\sqrt{\nworkers \widetilde{\cp} T}} + \frac{4\lip\vbnd \widetilde{A}}{\sqrt{\nworkers\widetilde{\cp} T}} + \frac{6\nworkers \lip^2\vbnd \widetilde{B}}{\widetilde{\cp} T} + \frac{12\nworkers\lip^2\bndb\widetilde{C}}{\widetilde{\cp}T} \\
=& \mathcal{O}\parenth{\frac{\widetilde{\cp}/\cpeff}{\sqrt{\nworkers \widetilde{\cp} T}}} + \mathcal{O}\parenth{\frac{\widetilde{A}\vbnd}{\sqrt{\nworkers \widetilde{\cp} T}}}  +\mathcal{O}\parenth{\frac{\nworkers \widetilde{B}\vbnd}{\widetilde{\cp} T}} + \mathcal{O}\parenth{\frac{\nworkers \widetilde{C} \bndb}{\widetilde{\cp} T}}.
\end{align}
Here, we complete the proof of \Cref{thm:nova}.

Moreover, it is worth mentioning the constraints on the local learning rate. Recall that, at the $t$-th round, we have the following constraint:
\begin{align}
\lr\lip \leq \frac{1}{2}\min\braces{\frac{1}{\max_i \vecnorm{\gradweight_i^{(t)}}_1 \sqrt{2\beta^2+1}}, \frac{1}{\cpeff}}.
\end{align}
In order to guarantee the convergence, the above inequality should hold in every round. That is to say,
\begin{align}
\lr\lip \leq \frac{1}{2}\min\braces{\frac{1}{\max_{i\in[m],t\in[T]} \vecnorm{\gradweight_i^{(t)}}_1 \sqrt{2\beta^2+1}}, \frac{1}{\cpeff}}.
\end{align}


\section{Extension: Incorporating Client Sampling}\label{sec:proof_sampling}
In this section, we extend the convergence guarantee of \fednova to the case of client sampling. Following previous works \cite{li2018federated,Li2020On,karimireddy2019scaffold,haddadpour2019local}, we assume the sampling scheme guarantees that the update rule \Cref{eqn:nova_updates} hold in expectation. This can be achieved by sampling with replacement from $\{1,2,\dots,\nworkers\}$ with probabilities $\{p_i\}$, and averaging local updates from selected clients with equal weights. Specifically, we have
\begin{align}
\x^{(t+1,0)} - \x^{(t,0)}
= -\cpeff \sum_{j=1}^\samplenum \frac{1}{\samplenum} \cdot \lr\nsg_{l_j}^{(t)} \quad \text{where} \ \nsg_{l_j}^{(t)} = \matG_{l_j}^{(t)}\gradweight_{l_j}/\|\gradweight_{l_j}\|_1
\end{align}
where $q$ is the number of selected clients per round, and $l_j$ is a random index sampled from $\{1,2,\cdots,\nworkers\}$ satisfying $\mathbb{P}(l_j = i) = p_i$. Recall that $p_i=n_i/n$ is the relative sample size at client $i$. For the ease of presentation, let $\gradweight_i$ to be fixed across rounds. One can directly validate that
\begin{align}
\Exs_S\brackets{\frac{1}{\samplenum}\sum_{j=1}^q \nsg_{l_j}^{(t)}} = \frac{1}{\samplenum}\sum_{j=1}^q\Exs_S\brackets{\nsg_{l_j}^{(t)}}=\Exs_S\brackets{\nsg_{l_j}^{(t)}}=\sum_{i=1}^\nworkers p_i \nsg_i^{(t)}
\end{align}
where $\Exs_S$ represents the expectation over random indices at current round.

\begin{corollary}
	Under the same condition as \Cref{thm:general}, suppose at each round, the server randomly selects $\samplenum (\leq \nworkers)$ clients with replacement to perform local computation. The probability of choosing the $i$-th client is $p_i=n_i/n$. In this case, \fednova will converge to the stationary points of the global objective $\obj(\x)$. If we set $\lr=\sqrt{\samplenum/\widetilde{\cp}T}$ where $\widetilde{\cp}$ is the average local updates across all rounds, then the expected gradient norm is bounded as follows:
	\begin{align}
	\min_{t\in[T]}\Exs\vecnorm{\tg(\x^{(t,0)})}^2
	\leq& \mathcal{O}\parenth{\frac{\widetilde{\cp}/\cpeff}{\sqrt{\samplenum \widetilde{\cp} T}}} + \mathcal{O}\parenth{\frac{\cpeff/\widetilde{\cp}}{\sqrt{\samplenum \widetilde{\cp} T}}} + \mathcal{O}\parenth{\frac{\samplenum (B+C)}{\widetilde{\cp} T}}
	\end{align}
	where $\mathcal{O}$ swallows all other constants (including $\lip, \vbnd, \bndb$).
\end{corollary}

\begin{proof}
	According to the Lipschitz-smooth assumption, it follows that
	\begin{align}
	\Exs\brackets{\obj(\x^{(t+1,0)})} - \obj(\x^{(t,0)})
	\leq& -\cpeff \lr \underbrace{\Exs\brackets{\inprod{\tg(\x^{(t,0)})}{\sum_{j=1}^{\samplenum}\frac{\nsg_{l_j}^{(t)}}{\samplenum}}}}_{T_3} + \frac{\cpeff^2\lr^2 \lip}{2}\underbrace{\Exs\brackets{\vecnorm{\sum_{j=1}^{\samplenum}\frac{\nsg_{l_j}^{(t)}}{\samplenum}}^2}}_{T_4} \label{eqn:sampling_basis}
	\end{align}
	where the expectation is taken over randomly selected indices $\{l_j\}$ as well as mini-batches $\xi_i^{(t,k)}, \forall i\in \{1,2,\dots,\nworkers\}, k \in \{0, 1,\dots,\cp_i-1\}$.
	
	For the first term in \Cref{eqn:sampling_basis}, we can first take the expectation over indices and obtain
	\begin{align}
	T_3
	=& \Exs\brackets{\inprod{\tg(\x^{(t,0)})}{\Exs_S\brackets{\sum_{j=1}^{\samplenum}\frac{\nsg_{l_j}^{(t)}}{\samplenum}}}} \\
	=& \Exs\brackets{\inprod{\tg(\x^{(t,0)})}{\sum_{i=1}^{\nworkers}p_i\nsg_i^{(t)}}}.
	\end{align}
	This term is exactly the same as the first term in \Cref{eqn:basis}. We can directly reuse previous results in the proof of \Cref{thm:general}. Comparing with \Cref{eqn:T1}, we have
	\begin{align}
	T_3
	=& \frac{1}{2}\vecnorm{\tg(\x^{(t)})}^2 + \frac{1}{2}\Exs\brackets{\vecnorm{\sum_{i=1}^\nworkers p_i\ntg_i^{(t)}}^2} - \frac{1}{2}\Exs\brackets{\vecnorm{\tg(\x^{(t,0)}) - \sum_{i=1}^\nworkers p_i\ntg_i^{(t)}}^2} \\
	\geq& \frac{1}{2}\vecnorm{\tg(\x^{(t)})}^2 + \frac{1}{2}\Exs\brackets{\vecnorm{\sum_{i=1}^\nworkers p_i\ntg_i^{(t)}}^2} - \frac{1}{2}\sum_{i=1}^\nworkers p_i\Exs\brackets{\vecnorm{\tg_i(\x^{(t,0)}) -\ntg_i^{(t)}}^2}. \label{eqn:sampling_T3}
	\end{align}
	For the second term in \Cref{eqn:sampling_basis},
	\begin{align}
	T_4
	\leq& 2\Exs\brackets{\vecnorm{\frac{1}{\samplenum}\sum_{j=1}^q (\nsg_{l_j}^{(t)} - \ntg_{l_J}^{(t)})}^2} + 2\Exs\brackets{\vecnorm{\frac{1}{\samplenum}\sum_{j=1}^q \ntg_{l_j}^{(t)}}^2} \\
	=& \frac{1}{\samplenum}\sum_{i=1}^\nworkers p_i \Exs\brackets{\vecnorm{\nsg_i^{(t)} - \ntg_i^{(t)}}^2} + 2\Exs\brackets{\vecnorm{\frac{1}{\samplenum}\sum_{j=1}^q \ntg_{l_j}^{(t)}}^2} \\
	\leq& \frac{2\vbnd}{\samplenum}\sum_{i=1}^\nworkers p_i \frac{\vecnorm{\gradweight_i}_2^2}{\vecnorm{\gradweight_i}_1^2} + 2\Exs\brackets{\vecnorm{\frac{1}{\samplenum}\sum_{j=1}^q \ntg_{l_j}^{(t)}}^2} \\
	\leq& \frac{2\vbnd}{\samplenum}\sum_{i=1}^\nworkers p_i \frac{\vecnorm{\gradweight_i}_2^2}{\vecnorm{\gradweight_i}_1^2} + 6\sum_{i=1}^\nworkers p_i\Exs\brackets{\vecnorm{\tg_i(\x^{(t,0)}) -\ntg_i^{(t)}}^2} + \frac{6}{\samplenum}\left( \beta^2 \|\nabla F(\x^{(t,0)}) \|^2 + \kappa^2\right) \nonumber \\
	& + 6\vecnorm{\tg(\x^{(t,0)})}^2 \label{eqn:sampling_T4}
	\end{align}
	where the last inequality comes from \Cref{lem:sampling}, stated below.
	\begin{lem}\label{lem:sampling}
		Suppose we are given $\z_1,\z_2,\ldots,\z_m,\x\in \mathbb{R}^d$
		and let $l_1,l_2,\ldots,l_\samplenum$ be i.i.d. sampled from a multinomial distribution $\mathcal{D}$ supported on $\{1,2,\ldots,m\}$ satisfying
		$\Prob(l=i) = p_i$ and $\sum_{i=1}^m p_i  =1 $. We have
		\begin{align}
		\Exs[ \frac{1}{\samplenum}\sum_{j=1}^\samplenum \z_{l_j}]
		=& \sum_{i=1}^\nworkers p_i \z_i, \\
		\Exs [ \| \frac{1}{\samplenum} \sum_{j=1}^\samplenum \z_{l_j} \|^2   ] 
		\le& 
		3 \sum_{i=1}^{m} p_i  \| \z_i -  \nabla F_{i}(\x)\|^2 + 3\vecnorm{\tg(\x)}^2
		+ \frac{3}{\samplenum}  \left( \beta^2 \|\nabla F(\x) \|^2 + \kappa^2\right).
		\end{align}
	\end{lem}
	\begin{proof}
		First, we have
		\begin{align}\label{eq:0}
		&\Exs [ \| \frac{1}{\samplenum} \sum_{j=1}^\samplenum \z_{l_j} \|^2   ] \nonumber\\
		=& \Exs \brackets{ \vecnorm{\parenth{\frac{1}{\samplenum} \sum_{j=1}^\samplenum \z_{l_j} - \frac{1}{\samplenum} \sum_{j=1}^\samplenum \tg_{l_j}(\x)} + \parenth{\frac{1}{\samplenum} \sum_{j=1}^\samplenum \tg_{l_j}(\x) -\tg(\x)} + \tg(\x)}^2 } \\
		\le & 3 \Exs [ \| \frac{1}{\samplenum} \sum_{j=1}^\samplenum \z_{l_j} - \frac{1}{\samplenum} \sum_{j=1}^\samplenum\nabla F_{l_j}(\x) \|^2   ] 
		+ 3\Exs [ \| \frac{1}{\samplenum} \sum_{j=1}^\samplenum\nabla F_{l_j}(\x) - \nabla F(\x) \|^2   ] + 3\vecnorm{\tg(\x)}^2. 
		\end{align}
		For the first term, by Cauchy-Schwarz inequality, we have 
		\begin{align}\label{eq:1}
		\Exs [ \| \frac{1}{\samplenum} \sum_{j=1}^\samplenum \z_{l_j} - \frac{1}{\samplenum} \sum_{j=1}^\samplenum \nabla F_{l_j}(\x) \|^2   ] 
		\le \frac{1}{\samplenum} \sum_{j=1}^\samplenum \Exs_{l_j\sim \mathcal{D}} [ \| \z_{l_j} - \nabla F_{l_j}(\x)\|^2]
		= \sum_{i=1}^{m} p_i  \| \z_i -  \nabla F_{i}(\x)\|^2.
		\end{align}
		The second term can be bounded as following
		\begin{align}\label{eq:2}
		\Exs [ \| \frac{1}{\samplenum} \sum_{j=1}^\samplenum \nabla F_{l_j}(\x) - \nabla F(\x) \|^2   ]
		&=  \frac{1}{\samplenum}  \Exs_{l_j\sim \mathcal{D}} [ \|  \nabla F_{l_j}(\x) - \nabla F(\x) \|^2   ] \\
		&= \frac{1}{\samplenum}\sum_{i=1}^\nworkers p_i \|\tg_i(\x) - \tg(\x)\|^2 \\
		&\leq \frac{1}{\samplenum}\brackets{(\bnda-1)\|\tg(\x)\|^2 + \bndb}.
		\end{align}
		where the first identity follows from $\Exs_{i\sim \mathcal{D}}[F_i(\x)] = \nabla F(\x)$ and the independence between $l_1,\ldots,l_\samplenum$, and the last inequality is a direct application of Assumption \ref{assump:dissimilarity}.
		
		Substituting \eqref{eq:1} and \eqref{eq:2} into \eqref{eq:0} completes the proof.
	\end{proof}
	
	Substituting \Cref{eqn:sampling_T3,eqn:sampling_T4} into \Cref{eqn:sampling_basis}, we have
	\begin{align}
	\frac{\Exs\brackets{\obj(\x^{(t+1,0)})} - \obj(\x^{(t,0)})}{\lr \cpeff}
	\leq& -\frac{1}{2}\parenth{1-6\cpeff\lr\lip}\vecnorm{\tg(\x^{(t,0)})}^2 \nonumber \\
	& + \parenth{\frac{1}{2} + 3\cpeff\lr\lip}\sum_{i=1}^\nworkers p_i \Exs\brackets{\vecnorm{\tg_i(\x^{(t,0)}) - \ntg_i^{(t)}}^2} \nonumber \\
	& + \frac{\cpeff\lr\lip\vbnd}{\samplenum}\sum_{i=1}^\nworkers p_i \frac{\vecnorm{\gradweight_i}_2^2}{\vecnorm{\gradweight_i}_1^2} + \frac{3\cpeff\lr\lip}{\samplenum}\left( \beta^2 \vecnorm{\nabla F(\x^{(t,0)})}^2 + \kappa^2\right) \\
	=& -\frac{1}{2}\parenth{1 - 6\cpeff\lr\lip - \frac{6\cpeff\lr\lip\bnda}{\samplenum}}\vecnorm{\tg(\x^{(t,0)})}^2 \nonumber \\
	& + \frac{\cpeff\lr\lip\vbnd}{\samplenum}\sum_{i=1}^\nworkers p_i \frac{\vecnorm{\gradweight_i}_2^2}{\vecnorm{\gradweight_i}_1^2} \nonumber \\
	& +  \parenth{\frac{1}{2} + 2\cpeff\lr\lip}\sum_{i=1}^\nworkers p_i \Exs\brackets{\vecnorm{\tg_i(\x^{(t,0)}) - \ntg_i^{(t)}}^2} + \frac{3\cpeff\lr\lip\bndb}{\samplenum}. 
	\end{align}
	When $\lr\lip \leq 1/(2\cpeff)$ and $6\cpeff\lr\lip + 6\cpeff\lr\lip\bnda/\samplenum \leq \frac{1}{2}$, it follows that
	\begin{align}
	\frac{\Exs\brackets{\obj(\x^{(t+1,0)})} - \obj(\x^{(t,0)})}{\lr \cpeff}
	\leq& -\frac{1}{4}\vecnorm{\tg(\x^{(t,0)})}^2 + \frac{\cpeff\lr\lip\vbnd}{\samplenum}\sum_{i=1}^\nworkers p_i \frac{\vecnorm{\gradweight_i}_2^2}{\vecnorm{\gradweight_i}_1^2} \nonumber \\
	& + \frac{3}{2}\sum_{i=1}^\nworkers p_i \Exs\brackets{\vecnorm{\tg_i(\x^{(t,0)}) - \ntg_i^{(t)}}^2} + \frac{3\cpeff\lr\lip\bndb}{\samplenum}. \label{eqn:sampling_step1}
	\end{align}
	Recall that the third term in \Cref{eqn:sampling_step1} can be bounded as follows (see \Cref{eqn:T3_step7}):
	\begin{align}
	\frac{1}{2}\sum_{i=1}^\nworkers p_i \Exs\brackets{\vecnorm{\tg_i(\x^{(t,0)}) - \ntg_i^{(t)}}^2}
	\leq& \frac{\lr^2\lip^2\vbnd}{1-D}\sum_{i=1}^\nworkers p_i \parenth{\vecnorm{\bm{a}_i}_2^2 - [a_{i,-1}]^2} \nonumber \\
	& + \frac{D\bnda}{2(1-D)}\vecnorm{\tg(\x^{(t,0)})}^2 + \frac{D\bndb}{2(1-D)} \label{eqn:sampling_final1}
	\end{align}
	where $D = 4\lr^2\lip^2 \max_i\{\vecnorm{\bm{a}_i}_1(\vecnorm{\bm{a}_i}_1-a_{i,-1}) \} < 1$. If $D \leq \frac{1}{12\beta^2+1}$, then it follows that $\frac{1}{1-D}\leq 1+\frac{1}{12\beta^2}\leq 2$ and $\frac{3D\beta^2}{1-D} \leq \frac{1}{4}$. These facts can help us further simplify inequality \Cref{eqn:sampling_final1}. One can obtain
	\begin{align}
	\frac{3}{2}\sum_{i=1}^\nworkers p_i \Exs\brackets{\vecnorm{\tg_i(\x^{(t,0)}) - \ntg_i^{(t)}}^2}
	\leq& 6\lr^2\lip^2\vbnd\sum_{i=1}^\nworkers p_i \parenth{\vecnorm{\bm{a}_i}_2^2 - [a_{i,-1}]^2} + \frac{1}{8}\vecnorm{\tg(\x^{(t,0)})}^2 \nonumber \\
	&+ 12\lr^2\lip^2\bndb \max_i\{\vecnorm{\bm{a}_i}_1(\vecnorm{\bm{a}_i}_1-a_{i,-1}) \\
	=& 6\lr^2\lip^2\vbnd B + \frac{1}{8}\vecnorm{\tg(\x^{(t,0)})}^2 + 12\lr^2\lip^2\bndb C \label{eqn:sampling_step2}
	\end{align}
	Substituting \Cref{eqn:sampling_step2} into \Cref{eqn:sampling_step1}, we have
	\begin{align}
	\frac{\Exs\brackets{\obj(\x^{(t+1,0)})} - \obj(\x^{(t,0)})}{\lr \cpeff}
	\leq& -\frac{1}{8}\vecnorm{\tg(\x^{(t,0)})}^2 + \frac{\cpeff\lr\lip\vbnd}{\samplenum}\sum_{i=1}^\nworkers p_i \frac{\vecnorm{\gradweight_i}_2^2}{\vecnorm{\gradweight_i}_1^2} + \frac{3\cpeff\lr\lip\bndb}{\samplenum} \nonumber \\
	& + 6\lr^2\lip^2\vbnd B + 12\lr^2\lip^2\bndb C \\
	\leq& -\frac{1}{8}\vecnorm{\tg(\x^{(t,0)})}^2 + \frac{\cpeff\lr\lip\vbnd}{\samplenum} + \frac{3\cpeff\lr\lip\bndb}{\samplenum} \nonumber \\
	& + 6\lr^2\lip^2\vbnd B + 12\lr^2\lip^2\bndb C
	\end{align}
	where the last inequality uses the fact that $\vecnorm{\bm{a}}_2\leq\vecnorm{\bm{a}}_1$, for any vector $\bm{a}$. Taking the total expectation and averaging all rounds, one can obtain
	\begin{align}
	\frac{\Exs\brackets{\obj(\x^{(T,0)})} - \obj(\x^{(0,0)})}{\lr \cpeff T}
	\leq& -\frac{1}{8T}\sum_{t=0}^{T-1}\Exs\brackets{\vecnorm{\tg(\x^{(t,0)})}^2} + \frac{\cpeff\lr\lip(\vbnd+3\bndb)}{\samplenum}  \nonumber \\
	& + 6\lr^2\lip^2\vbnd B + 12\lr^2\lip^2\bndb C.
	\end{align}
	After minor rearranging, the above inequality is equivalent to
	\begin{align}
	&\frac{1}{T}\sum_{t=0}^{T-1}\Exs\brackets{\vecnorm{\tg(\x^{(t,0)})}^2}\nonumber\\
	\leq& \frac{8\brackets{\obj(\x^{(0,0)}) - \obj_{\text{inf}}}}{\lr\cpeff T} + \frac{8\cpeff\lr\lip(\vbnd+3\bndb)}{\samplenum} + 48\lr^2\lip^2\vbnd B + 96\lr^2\lip^2\bndb C.
	\end{align}
	If we set the learning rate to be small enough, \ie, $\lr = \sqrt{\frac{\samplenum}{\widetilde{\cp} T}}$ where $\widetilde{\cp} = \sum_{t=0}^{T-1}\cpavg/T$, then we get
	\begin{align}
	\frac{1}{T}\sum_{t=0}^{T-1}\Exs\brackets{\vecnorm{\tg(\x^{(t,0)})}^2}
	\leq& \mathcal{O}\parenth{\frac{\widetilde{\cp}/\cpeff}{\sqrt{\samplenum \widetilde{\cp} T}}} + \mathcal{O}\parenth{\frac{\cpeff/\widetilde{\cp}}{\sqrt{\samplenum \widetilde{\cp} T}}} + \mathcal{O}\parenth{\frac{\samplenum (B+C)}{\widetilde{\cp} T}}
	\end{align}
	where $\mathcal{O}$ swallows all other constants.
\end{proof}

\newpage
\section{Pseudo-code of FedNova}\label{sec:pseudocode}
Here we provide a pseudo-code of \fednova (see \Cref{algo:nova}) as a general algorithmic framework. Then, as an example, we show the pseudo-code of a special case of \fednova, where the local solver is specified as momentum SGD with cross-client variance reduction \cite{liang2019variance,karimireddy2019scaffold} (see \Cref{algo:novavrl}). Note that when the server updates the global model, we set $\cpeff$ to be the same as \fedavg, \ie, $\cpeff = \sum_{i \in \mathcal{S}_t} p_i \|\gradweight_i^{(t)}\|_1$ where $\mathcal{S}_t$ denotes the randomly selected subset of clients. Alternatively, the server can also choose other values of $\cpeff$.

\begin{algorithm}
	\DontPrintSemicolon
	\SetKwInput{Input}{Input}
	\SetAlgoLined
	\LinesNumbered
	\Input{Client learning rate $\lr$; Client momentum factor $\rho$. }
	\For{$t \in \{0,1,\dots,T-1\}$}{
		Randomly sample a subset of clients $\mathcal{S}_t$\;
		\textbf{Communication:} Broadcast global model $\x^{(t,0)}$ to selected clients\;
		Clients perform local updates\;
		\textbf{Communication:} Receive $\|\gradweight_i^{(t)}\|_1$ and $\nsg_i^{(t)}$ from clients\;
		Update global model: $\x^{(t+1,0)} = \x^{(t,0)} - \frac{\sum_{i\in\mathcal{S}_t}p_i\|\gradweight_i^{(t)}\|_1}{\sum_{i\in\mathcal{S}_t}p_i} \sum_{i\in\mathcal{S}_t} \frac{\lr p_i \nsg_i^{(t)}}{\sum_{i\in\mathcal{S}_t}p_i}$
	}
	\caption{FedNova Framework}
	\label{algo:nova}
\end{algorithm}

\begin{algorithm}[!h]
	\DontPrintSemicolon
	\SetKwInput{Input}{Input}
	\SetAlgoLined
	\LinesNumbered
	\Input{Client learning rate $\lr$; Client momentum factor $\rho$. }
	\For{$t \in \{0,1,\dots,T-1\}$ {\it \bf at cleint} $i$ {\it \bf in parallel}}{
		Zero client optimizer buffers $\bm{u}_i^{(t,0)}=0$\;
		\textbf{Communication:} Receive $\x^{(t,0)} = \x^{(t-1,0)}  - (\sum_{i=1}^\nworkers p_i a_i) \lr \sum_{i=1}^{\nworkers}p_i \nsg_i^{(t-1)}$ from server\;
		\textbf{Communication:} Receive $\sum_{i=1}^{\nworkers} p_i\nsg_i^{(t-1)}$ from server \;
		Update gradient correction term: $\bm{c}_i^{(t)} = -\nsg_i^{(t-1)} + \sum_{i=1}^{\nworkers} p_i\nsg_i^{(t-1)}$\;
		\For{$k \in \{0,1,\dots,\cp_i-1\}$}{
			Compute: $\tilde{g}_i(\x^{(t,k)}) = \sg_i(\x^{(t,k)}) + \bm{c}_i^{(t)}$\;
			Update momentum buffer: $\bm{u}_i^{(t,k)} = \rho \bm{u}_i^{(t,k-1)} + \tilde{g}_i(\x^{(t,k)})$\;
			Update local model: $\x_i^{(t,k)} = \x_i^{(t,k-1)} - \lr \bm{u}_i^{(t,k)}$\;
		}
		Compute: $a_i = [\cp_i - \rho(1-\rho^{\cp_i})/(1-\rho)]/(1-\rho)$\;
		Compute normalized gradient: $\nsg_i^{(t)} = (\x^{(t,0)} - \x^{(t,\cp_i)})/(\lr a_i)$\;
		\textbf{Communication:} Send $p_i a_i$ and $p_i \nsg_i^{(t)}$ to the server\;
	}
	\caption{FedNova with Client-side Momentum SGD + Cross-client Variance Reduction}
	\label{algo:novavrl}
\end{algorithm}

\section{More Experiments Details}\label{sec:more_exps}
\textbf{Platform.} All experiments in this paper are conducted on a cluster of $16$ machines, each of which is equipped with one NVIDIA TitanX GPU. The machines communicate (\ie, transfer model parameters) with each other via Ethernet. We treat each machine as one client in the federated learning setting. The algorithms are implemented by \texttt{PyTorch}. We run each experiments for $3$ times with different random seeds.

\textbf{Hyper-parameter Choices.} On non-IID CIFAR10 dataset, we fix the mini-batch size per client as $32$. When clients use momentum SGD as the local solver, the momentum factor is $0.9$; when clients use proximal SGD, the proximal parameter $\mu$ is selected from $\{0.0005, 0.001, 0.005, 0.01\}$. It turns out that when $E_i=2$, $\mu=0.005$ is the best and when $E_i(t)\sim \mathcal{U}(2,5)$, $\mu=0.001$ is the best. The client learning rate $\lr$ is tuned from $\{0.005,0.01,0.02,0.05,0.08\}$ for \fedavg with each local solver separately. When using the same local solver, \fednova uses the same client learning rate as \fedavg. Specifically, if the local solver is momentum SGD, then we set $\lr=0.02$. In other cases, $\lr=0.05$ consistently performs the best. On the synthetic dataset, the mini-batch size per client is $20$ and the client learning rate is $0.02$.

\textbf{Training Curves on Non-IID CIFAR10.}
The training curves of \fedavg and \fednova are presented in \Cref{fig:cifar_curves}. Observe that \fednova (red curve) outperforms \fedavg (blue curve) by a large margin. \fednova only requires about half of the total rounds to achieve the same test accuracy as \fedavg. Besides, note that in \cite{wang2020federated}, the test accuracy of \fedavg is higher than ours. This is because the authors of \cite{wang2020federated} let clients to perform $20$ local epochs per round, which is $10$ times more than our setting. In \cite{wang2020federated}, after $100$ communication rounds, \fedavg equivalently runs $100\times 20 = 2000$ epochs.
\begin{figure}[!ht]
	\centering
	\begin{subfigure}{.33\textwidth}
		\centering
		\includegraphics[width=\textwidth]{fix_SGD.pdf}
	\end{subfigure}%
	\hspace{2em}
	\begin{subfigure}{.33\textwidth}
		\centering
		\includegraphics[width=\textwidth]{fix_SGDM.pdf}
	\end{subfigure}%
	
	\begin{subfigure}{.33\textwidth}
		\centering
		\includegraphics[width=\textwidth]{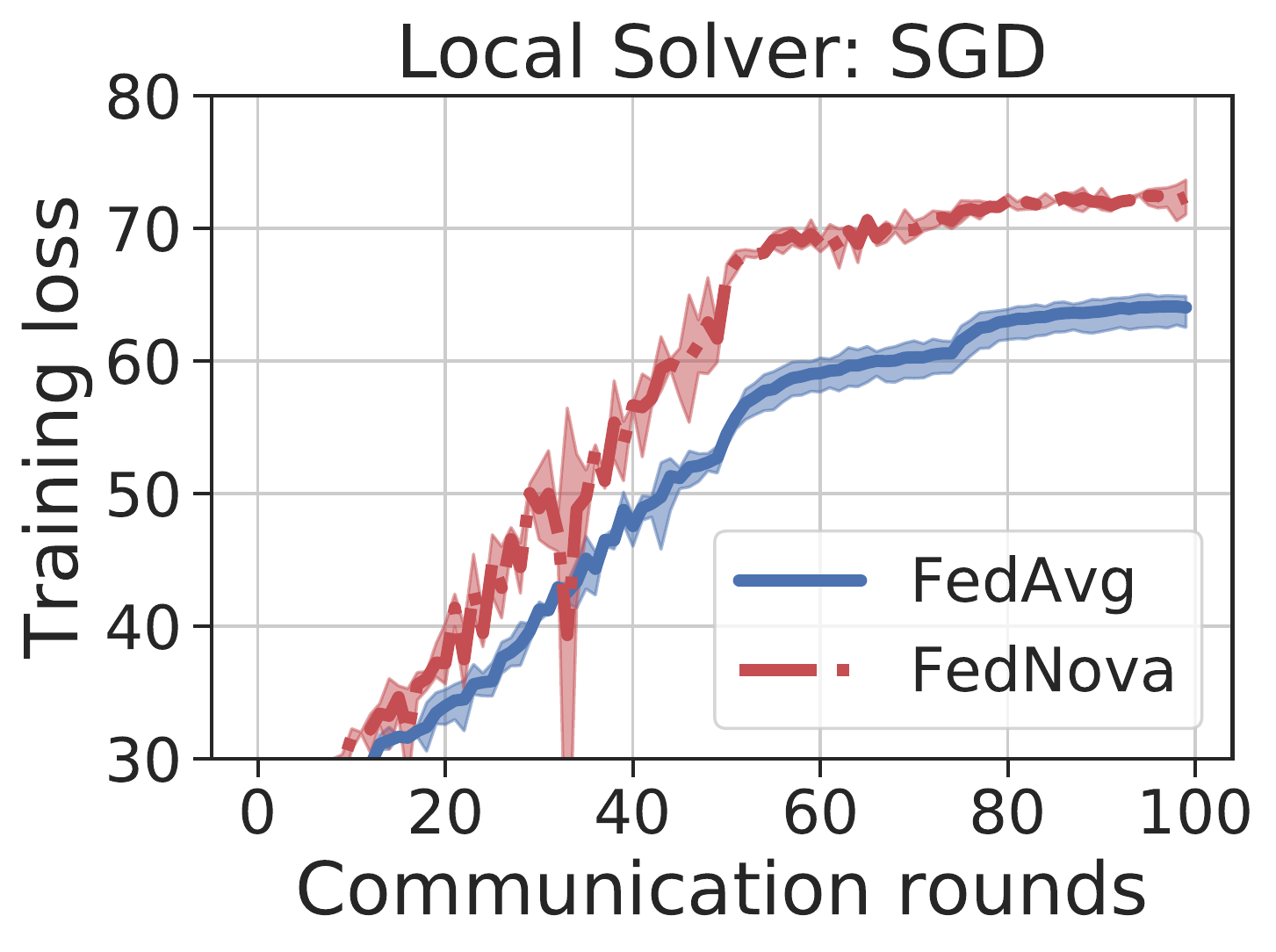}
	\end{subfigure}%
	\hspace{2em}
	\begin{subfigure}{.33\textwidth}
		\centering
		\includegraphics[width=\textwidth]{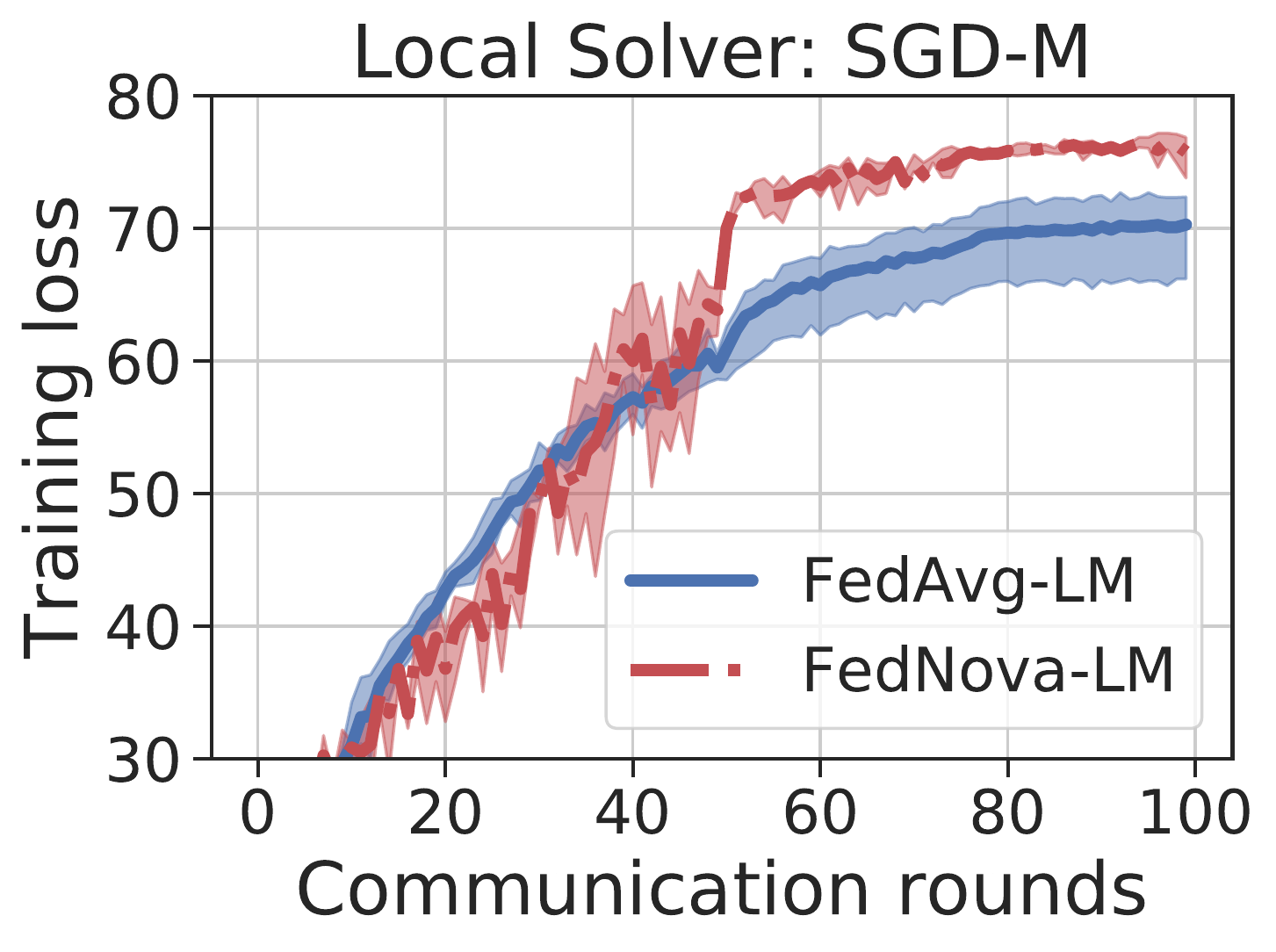}
	\end{subfigure}
	\caption{Training curves on non-IID partitioned CIFAR10 dataset. In these curves, the only difference between \fedavg and \fednova is the weights when aggregating normalized gradients. `LM' represents for local momentum. \textbf{\emph{First row}}: All clients perform $E_i=2$ local epochs; \textbf{\emph{Second row}}: All clients perform random and time-varying local epochs $E_i(t) \sim \mathcal{U}(2,5)$.}
	\label{fig:cifar_curves}
	\vspace{-1em}
\end{figure}

\begin{figure}[!ht]
	\centering
	\begin{subfigure}{.33\textwidth}
		\centering
		\includegraphics[width=\textwidth]{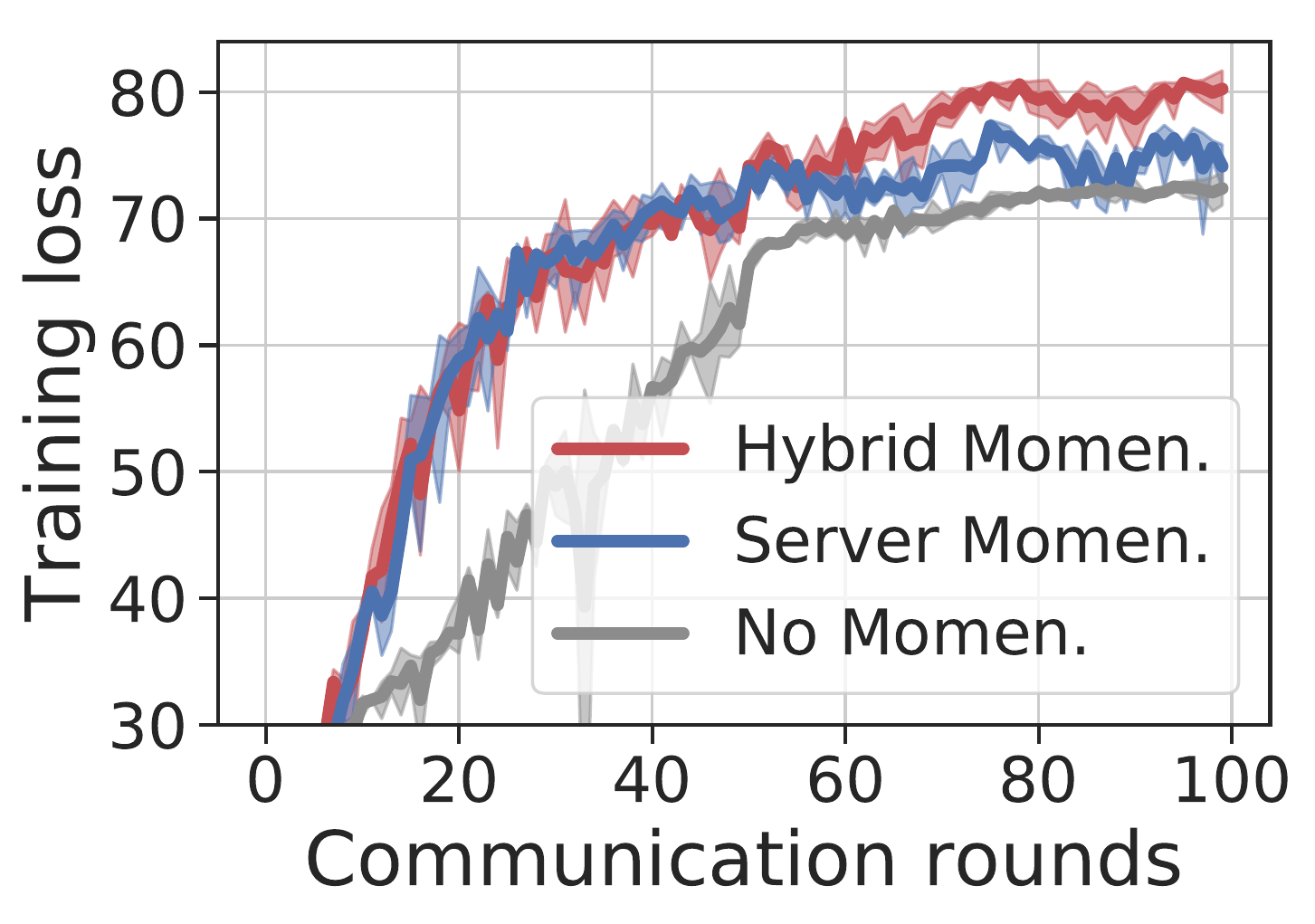}
	\end{subfigure}%
	\hspace{2em}
	\begin{subfigure}{.33\textwidth}
		\centering
		\includegraphics[width=\textwidth]{fedprox_ablation_v2.pdf}
	\end{subfigure}%
	\caption{\textbf{\emph{Left}}: Comparison of different momentum schemes in \fednova. `Hybrid momentum' corresponds to the combination of server momentum and client momentum. \textbf{\emph{Right}}: How \fednova-prox outperform vanilla \fedprox (blue curve). By setting $\cpeff=\sum_{i=1}^\nworkers p_i \cp_i$ instead of its default value, the accuracy of \fedprox can be improved by $5\%$ (see the green curve). By further correcting the aggregated weights, \fednova-prox (red curves) achieves around $10\%$ higher accuracy than \fedprox.}
	\label{fig:my_label}
\end{figure}

\end{document}